\documentclass[10pt]{article} 
\usepackage{amsmath,amssymb,amsthm,mathtools}
\usepackage[T1]{fontenc}

\usepackage{newtxtext}
\usepackage[section]{placeins}
\usepackage{float}
\usepackage{newtxmath}
\usepackage[margin=1in]{geometry}
\usepackage{booktabs}
\usepackage{algorithm}
\usepackage{multirow}
\usepackage{algpseudocode}
\usepackage{enumitem}
\usepackage{subcaption}
\usepackage{color,soul}
\usepackage{enumitem}
\usepackage{siunitx}
\usepackage{algorithm}
\usepackage{algpseudocode}
\usepackage{placeins}

\usepackage{amsmath,amsfonts,bm}

\def\eqref#1{equation~\ref{#1}}

\def\plaineqref#1{(\ref{#1})}

\def\1{\bm{1}}

\DeclareMathAlphabet{\mathsfit}{\encodingdefault}{\sfdefault}{m}{sl}
\SetMathAlphabet{\mathsfit}{bold}{\encodingdefault}{\sfdefault}{bx}{n}

\newcommand{\E}{\mathbb{E}}

\DeclareMathOperator*{\argmin}{arg\,min}

\usepackage[backref,colorlinks,citecolor=blue,bookmarks=true]{hyperref}
\usepackage{mathtools, amssymb, amsthm, dsfont}
\numberwithin{equation}{section}
\usepackage{enumitem}
\usepackage[nameinlink,capitalize]{cleveref}
\usepackage{algorithm}
\usepackage[table]{xcolor}

\usepackage{parskip}
\usepackage{mathtools, amssymb, amsthm, bbm}
\usepackage[nameinlink,capitalize]{cleveref}

\theoremstyle{plain}
\newtheorem{theorem}{Theorem}[section]
\newtheorem{lemma}{Lemma}[section]

\newtheorem{assumption}{Assumption}[section]

\theoremstyle{definition}

\newtheorem{example}{Example}[section]

\theoremstyle{remark}

\newtheorem{remark}[theorem]{Remark}

\newcommand{\norm}[1]{\|#1\|}

\usepackage{url}
\usepackage{etoc}
\usepackage[title]{appendix}

\definecolor{light-gray}{gray}{0.95}

\definecolor{HHgreen}{HTML}{34C759}

\newlist{thmassumptions}{enumerate}{1}
\setlist[thmassumptions,1]{label=(\alph*), ref=\thetheorem(\alph*)}

\crefname{thmassumptionsi}{Assumption}{assumptions}
\crefname{thmassumptionsi}{Assumption}{Assumptions}

\title{Adapt or Forget: Provable Tradeoffs Between Adam and SGD in Nonstationary Optimization}

\author{ Sharan Sahu\footnotemark[1] \\ \texttt{ss4329@cornell.edu} \\ Department of Statistics and Data Science \\ Cornell University \and Abir Sarkar\footnotemark[1] \\ \texttt{as4458@cornell.edu} \\ Department of Statistics and Data Science \\ Cornell University \and Cameron J. Hogan\footnotemark[1] \\ \texttt{cjh337@cornell.edu} \\ Department of Statistics and Data Science \\ Cornell University \and Martin T. Wells \\ \texttt{mtw1@cornell.edu} \\ Department of Statistics and Data Science \\ Cornell University}

\makeatletter
\gdef\equalcontrib{
    \renewcommand{\thefootnote}{\fnsymbol{footnote}} 
    \footnotetext[1]{\normalsize These authors contributed equally to this work.}
    \renewcommand{\thefootnote}{\arabic{footnote}}
}
\makeatother

\begin{document}

\maketitle
\equalcontrib
\begin{abstract}
We provide a theoretical analysis of Adam under non-stationary stochastic objectives, separating two regimes: Euclidean tracking under adaptive strong monotonicity of the Adam-preconditioned mean-gradient operator, and high-probability projected stationarity guarantees under general \(L\)-smooth objectives. In the tracking regime, we derive finite-time expected and high-probability bounds that decompose sharply into four components: initialization, objective drift, a first-moment tracking error governed by \(\beta_1\), and a preconditioner perturbation governed by \(\beta_2\). We characterize the burn-in time required for the transient terms to decay to the asymptotic tracking bound under constant and step-decay schedules. We also prove a high-probability bound on the average projected stationarity gap for Adam under distribution shift. Across both analyses, our bounds reveal a noise--drift tradeoff: in noise-dominated regimes, first-moment averaging and adaptive preconditioning can yield favorable upper guarantees, whereas in drift-dominated regimes, stale first-moment information and preconditioner perturbations can enlarge Adam's tracking guarantee, potentially allowing vanilla SGD to attain a smaller tracking error. Our explicit \((\beta_1,\beta_2,\epsilon)\)-dependent bounds identify mechanisms through which adaptive step-sizing can help or hurt under nonstationarity and provide theoretical explanations consistent with Adam's empirical instability and stabilization under distribution shift.
\end{abstract}

\section{Introduction}
Consider the optimization problem on a compact convex set $\Theta \subset \mathbb{R}^{d}$ posed by a smooth objective function $G_{t}: \mathbb{R}^{d} \rightarrow \mathbb{R}$ defined as:
\begin{equation}
    \tag{Opt}\label{eq:nonstationary-convex-opt}
    \boldsymbol{\theta}_{t}^{\star} \in \argmin_{\boldsymbol{\theta}\in \Theta} G_{t}(\boldsymbol{\theta}), \;\; G_{t}(\boldsymbol{\theta}) = \mathbb{E}_{X_{t} \sim \Pi_{t}}[g(\boldsymbol{\theta}, X_{t})].
\end{equation}
Here $g(\boldsymbol{\theta}, X_{t})$ is a noise perturbed measurement of $G_{t}(\boldsymbol{\theta})$ and $X_{t}$ is a random variable sampled from a time-varying distribution $\Pi_{t}$. This setting connects naturally to classical stochastic tracking in signal processing \cite{Kushner1997,sayed2003fundamentals}, concept drift in online learning \cite{10.1561/2400000013,NEURIPS2021_62e7f2e0}, and minimum divergence-type statistical objectives \cite{roy2026asymptotic}, where the goal is not convergence to a fixed solution but faithful tracking of an unknown, time-varying minimizer.

Adaptive gradient methods such as Adam \cite{DBLP:journals/corr/KingmaB14} have become the default optimizer for training large-scale deep learning models \cite{brown2020language, dosovitskiy2020image}, combining coordinate-wise adaptive step sizes derived from second-moment estimates with first-moment gradient averaging. In stationary settings, Adam is known to enjoy favorable convergence properties \cite{ hong2024convergence, jin2025comprehensive}, and its adaptivity is widely credited with reducing sensitivity to hyperparameter tuning and accelerating progress on ill-conditioned objectives. Recent work analyzing stochastic gradient descent (SGD) and its momentum-based variants (SGDM), including Polyak's Heavy-Ball \cite{POLYAK1963864} and Nesterov acceleration \cite{Nesterov1983AMF}, in non-stationary settings \cite{sahu2026provable} has shown that gradient averaging induces tracking lag under distribution shift. Since Adam incorporates momentum in both its first and second-moment estimates, it is natural to ask if similar degradation occurs, and whether the adaptive preconditioning compounds or alleviates it. These failure modes have indeed been observed empirically in continual learning and non-stationary optimization \cite{dohare2023maintaining, lyle2023understanding, ellis2024adam, DONG2026133279}, yet a precise theoretical account of when and why Adam degrades under distribution shift has remained unclear.

These observations motivate the following central question in the non-stationary tracking setting:
\begin{center}
\begin{minipage}{0.99\linewidth}
\centering\itshape
When does Adam's adaptive preconditioning help under distribution shift, and when can its moment estimates provably hurt relative to vanilla SGD, for \(L\)-smooth time-varying functions \(G_t(\boldsymbol{\theta})\)?
\end{minipage}
\end{center}

\textbf{Our contributions.} We answer this question under predictable distribution shift and make the following advances.

\begin{enumerate}
    \item \textbf{Finite-time tracking bounds under adaptive strong monotonicity.} We prove finite-time tracking bounds for Adam that separate four distinct error sources: initialization decay, drift-induced tracking lag, a first-moment tracking error governed by $\beta_1$, and a preconditioner perturbation governed by $\beta_2$. To obtain a tracking recursion, we impose strong monotonicity on the mean-gradient map after Adam-style preconditioning. Since the realized Adam preconditioner is computed from the fresh sample $X_{t+1}$ at the same step, we formulate this condition using a predictable proxy rather than the sample-dependent preconditioner itself. This yields a population-level contraction condition, while the difference between the realized preconditioner and its predictable proxy is controlled explicitly as an Adam-specific perturbation term.

    \item \textbf{A noise--drift tradeoff separating Adam from SGD.} 
    Our explicit \((\beta_1,\beta_2,\epsilon)\)-dependent bounds reveal an optimizer-specific noise--drift tradeoff that is absent from the corresponding vanilla SGD guarantee and takes a different form from that of SGDM. In noise-dominated regimes, first-moment averaging and adaptive preconditioning can yield favorable high-probability tracking guarantees, and our bounds identify parameter regimes in which Adam's upper bound is smaller than the corresponding SGD upper bound, consistent with our numerical simulations. In drift-dominated regimes, however, stale first-moment information together with preconditioner perturbations can enlarge Adam's asymptotic tracking bound, potentially allowing vanilla SGD to attain a smaller tracking error. The bounds also provide a theoretical mechanism consistent with the empirical observation of \cite{lyle2023understanding, lee2024slow, bai2025adaptive, gogianu2026revisiting} that increasing \(\epsilon\) stabilizes Adam under task changes: larger \(\epsilon\) dampens the effect of adaptive second-moment variability, while also weakening the effective adaptive step size and thereby potentially slowing adaptation to a changing objective.

    \item \textbf{Projected stationarity bounds under general Adam preconditioners.} Without the adaptive monotonicity condition, we prove a high-probability bound on the average projected stationarity gap for Adam under distribution shift. The same Adam-specific error structure persists: objective drift, first-moment bias-variance tradeoffs, and second-moment preconditioner perturbation all contribute explicitly, showing that the qualitative picture identified in the tracking regime extends to general Adam preconditioners.

    \item \textbf{Experiments corroborating the noise--drift tradeoff.} We validate the theory on strongly convex least squares and three non-convex online problems: teacher--student MLP regression, phase retrieval, and matrix factorization under controlled drift and time-varying noise. Across these settings, the results support the predicted tradeoff: SGD performs better when drift dominates because it adapts more quickly, whereas Adam is more effective when stochastic noise dominates due to momentum and adaptive scaling.
\end{enumerate}

\subsection{Related Work}
Adam \cite{DBLP:journals/corr/KingmaB14} has been the subject of extensive theoretical study in stationary settings. Indeed, in a variety of nonconvex settings, Adam provably converges under realistic 
conditions and adapts to the geometry of the problem in ways that SGD cannot \cite{li2023convergence, hong2024convergence, taniguchi2024adopt, jin2026adam}. The advantage of Adam over 
SGD has been attributed to several sources, including heavy-tailed gradient distributions in language modeling \cite{kunstner2024heavytailed}, Hessian heterogeneity across layers in Transformer architectures \cite{zhang2024transformers}, and the 
implicit sign-descent behavior of its updates \cite{kunstner2023noise}. However, Adam does not universally outperform simpler methods: it has 
been shown to generalize worse than SGD in certain settings \cite{wilson2017marginal}, fail to converge without modification in convex problems \cite{ReddiKK18}, and exhibit instability and loss of 
plasticity under distribution shift \cite{lyle2023understanding, dohare2023maintaining, ellis2024adam}. Most existing convergence theory for Adam assumes a fixed stationary objective.

Under non-stationarity, past gradients become stale and temporal averaging can be actively harmful. Building on the foundational dynamic regret framework of \cite{zinkevich2003online}, non-asymptotic 
tracking and regret guarantees have been established for SGD under time drift and dependent data \cite{8814889, JMLR:v24:21-1410, JMLR:v25:21-0748, shen2026sgddependentdataoptimal}. Most directly 
related to our work, \cite{sahu2026provable} provide finite-time tracking bounds for SGD and its momentum variants under strongly convex smooth objectives, showing that momentum amplifies drift-induced 
tracking lag. Nonconvex tracking guarantees under distribution shift remain scarce even for SGD, and no prior work addresses Adam in this setting. The failure of Adam under distribution shift has been documented empirically in continual learning and reinforcement learning \cite{lyle2023understanding, dohare2023maintaining, ellis2024adam}, yet no prior work provides a theoretical framework characterizing when and why Adam degrades, how its $(\beta_1, \beta_2, \epsilon)$ hyperparameters interact with the drift-noise balance, or how it compares to SGD in a unified tracking setting. We survey the broader literature on adaptive methods and non-stationary optimization more extensively in Appendix~\ref{app:A}.

\section{Preliminaries}
\label{sec:preliminaries}
We first summarize the notation used throughout the paper. Scalars, vectors, and matrices are denoted by lowercase, bold lowercase, and bold uppercase letters, respectively; calligraphic letters denote sets, operators, or $\sigma$-algebras. For $\boldsymbol{x} \in \mathbb{R}^d$, $\|\boldsymbol{x}\|$ denotes the $\ell_2$ norm,  $\langle \boldsymbol{x}, 
\boldsymbol{x}' \rangle$ the inner product, and $\boldsymbol{x}^{\odot 2} := (x_1^2, \ldots, x_d^2)^\top$ the coordinatewise square. For a positive semidefinite matrix $\boldsymbol{A} \in \mathbb{R}^{d \times d}$, we write $\|\boldsymbol{x}\|_{\boldsymbol{A}} := \sqrt{\boldsymbol{x}^\top \boldsymbol{A} \boldsymbol{x}}$. For $x \in \mathbb{R}$, we write $(x)_+ := \max\{x, 0\}$. We write $a_m = \mathcal{O}(b_m)$ if $a_m \le Cb_m$ for some $C>0$, $a_m = \Theta(b_m)$ if $a_m$ and $b_m$ are of the same order, and $a_m = o(b_m)$ if $a_m/b_m \to 0$ as $m \to \infty$. The notation $a_m \lesssim b_m$ (resp.\ $\gtrsim$, $\asymp$) indicates inequality up to constants independent of $m$ and problem parameters. For $f: \mathbb{R}^d \to \mathbb{R}$, we write $\nabla f$ for the gradient and $f^{\star} = \min_{x \in \mathbb{R}^{d}} f(x)$. We use $\mathbb{E}[\cdot]$ for expectation and $\mathcal{F}_t = \sigma(X_0, \ldots, X_t)$ for the natural filtration. We write $a \wedge b := \min \left\{a, b \right\}$ and $a \vee b := \max \left\{ a, b \right\}$. We define $\log_+(x):=\log(e\vee x)$ for $x\ge0$. We write $\mathcal{P}_{\Theta}$ to denote the projection operator onto the set $\Theta$.

We next introduce conditional $\Psi_\alpha$--Orlicz norms, which we will use throughout our analysis to control random quantities in a filtration-adapted (i.e., history-dependent) manner. Fix $\alpha\ge 1$. For a real-valued random variable $X$, recall the (unconditional) $\Psi_\alpha$--Orlicz norm
\[
\|X\|_{\Psi_\alpha}:=\inf\Big\{u>0:\ \mathbb{E}\exp\big((|X|/u)^\alpha\big)\le 2\Big\}.
\]
Given a $\sigma$-algebra $\mathcal F$, we write $\|X\|_{\Psi_\alpha\mid\mathcal F}\le K_{\mathcal F}$ for some $\mathcal F$-measurable $K_{\mathcal F}>0$ if $\mathbb{E}\!\left[\exp\!\left(\big(|X|/K_{\mathcal F}\big)^\alpha\right)\middle|\mathcal F\right]\le 2$. Equivalently, one may define for $u$
\[
\|X\|_{\Psi_\alpha\mid\mathcal F} :=\inf\Big\{u>0:\
\mathbb{E}\!\left[\exp\big((|X|/u)^\alpha\big)\mid\mathcal F\right]\le 2\ \Big\}
\]
provided $u$ is $\mathcal{F}$ measurable. Vector and matrix conditional Orlicz norms are defined analogously to their unconditional counterparts by taking suprema over one-dimensional projections (see \cite{sahu2026provable} for details).

\subsection{Problem setup}
Recall the optimization problem (\ref{eq:nonstationary-convex-opt}). Let $(\mathcal{F}_t)_{t \geq 1}$ be the natural filtration $\mathcal{F}_{t} = \sigma(X_{0}, \dots, X_{t})$. Throughout our analysis, we adopt the following assumption from \cite{sahu2026provable}, which models stochasticity in the non-stationary setting:

\vspace{0.5em}

\begin{assumption}[Stochastic predictability framework]
\label{assump:filtered-predictable-mds}
There exists a filtered probability space $(\Omega,\mathcal F,(\mathcal F_t)_{t\ge 0},\mathbb P)$ with $\mathcal F_0=\{\emptyset,\Omega\}$.
Let $(X_t)_{t\ge 0}$ be an $\mathbb F$--adapted process, i.e., $X_t$ is $\mathcal F_t$--measurable for all $t$. For each $t\ge 0$, let $\Pi_{t+1}$ denote the regular conditional law of $X_{t+1}$ given $\mathcal F_t$, i.e., $\Pi_{t+1}(A)=\mathbb P(X_{t+1}\in A\mid \mathcal F_t)$ a.s.\ for every measurable set $A$, and note $\Pi_{t+1}$ is $\mathcal F_t$--measurable.
Define the conditional risk
\[
G_{t+1}(\boldsymbol{\theta})
:= \mathbb E\!\left[g(\boldsymbol{\theta},X_{t+1})\mid \mathcal F_t\right]
= \mathbb E_{X\sim \Pi_{t+1}}\!\left[g(\boldsymbol{\theta},X)\right],
\]
and let $\boldsymbol{\theta}_{t+1}^\star\in\arg\min_{\boldsymbol{\theta}\in \Theta} G_{t+1}(\boldsymbol{\theta})$ denote a (measurable) minimizer contained within the interior of $\Theta$ where $\Theta \subset \mathbb{R}^{d}$ is compact and convex. Assume the following hold for all $t\ge 0$:
\begin{enumerate}
    \item \textbf{(Predictable minimizer)} $\boldsymbol{\theta}_{t+1}^\star$ is $\mathcal F_t$--measurable.\footnote[1]{Since $G_{t+1}(\boldsymbol{\theta})=\E[g(\boldsymbol{\theta}, X_{t+1})\mid \mathcal{F}_t]$ is $\mathcal{F}_t$-measurable for each fixed $\boldsymbol{\theta}$, standard measurable selection conditions ensure existence of an $\mathcal{F}_t$-measurable minimizer. If the minimizer is a.s.\ unique (e.g., under strong convexity), measurability follows by the measurable maximum theorem. See \cite{sahu2026provable} for details.}
    \item \textbf{(Algorithm adaptedness)} The iterate $\boldsymbol{\theta}_t$ is $\mathcal F_t$--measurable.
    \item \textbf{(Martingale difference noise)} Define the conditional mean gradient $\bar{\boldsymbol{g}}_{t+1}(\boldsymbol{\theta}) := \mathbb{E}[\nabla_{\boldsymbol{\theta}} g(\boldsymbol{\theta}, X_{t+1}) \mid \mathcal{F}_t]$, the conditional second moment $\boldsymbol{s}_{t+1}(\boldsymbol{\theta}) := \mathbb{E}[(\nabla_{\boldsymbol{\theta}} g(\boldsymbol{\theta}, X_{t+1}))^{\odot 2} \mid \mathcal{F}_t]$, and the associated noise terms $\boldsymbol{\xi}_{t+1} := \nabla_{\boldsymbol{\theta}} g(\boldsymbol{\theta}_t, X_{t+1}) - \bar{\boldsymbol{g}}_{t+1}(\boldsymbol{\theta}_t)$ and $\boldsymbol{\chi}_{t+1} := (\nabla_{\boldsymbol{\theta}} g(\boldsymbol{\theta}_t, X_{t+1}))^{\odot 2} - \boldsymbol{s}_{t+1}(\boldsymbol{\theta}_t)$. Then $\boldsymbol{\xi}_{t+1}$ and $\boldsymbol{\chi}_{t+1}$ are $\mathcal{F}_{t+1}$-measurable and satisfy $\mathbb{E}[\boldsymbol{\xi}_{t+1} \mid \mathcal{F}_t] = \mathbf{0}$ and $\mathbb{E}[\boldsymbol{\chi}_{t+1} \mid \mathcal{F}_t] = \mathbf{0}$ a.s. for all $t$.
\end{enumerate}
\end{assumption}

The time-varying distribution $\Pi_t$ in \cref{assump:filtered-predictable-mds} is a standard modeling assumption in the non-stationary stochastic optimization literature \cite{sahu2026provable,NEURIPS2021_62e7f2e0} and captures several practical settings, including policy optimization and reinforcement learning \cite{kakade2002approximately, schulman2015trust}, online recommendation and contextual bandits \cite{li2010contextual}, continual learning \cite{parisi2019continual}, and federated learning with non-stationary clients \cite{kairouz2021advances}.

The Adam update \cite{DBLP:journals/corr/KingmaB14}, with step size
\(\alpha>0\), momentum parameters \(\beta_1,\beta_2\in(0,1)\), and
stabilization parameter \(\epsilon>0\), is summarized in
Algorithm~\ref{alg:adam}. We write the stochastic gradient and its
elementwise square as

$$
    \nabla_{\boldsymbol{\theta}} g(\boldsymbol{\theta}_t,X_{t+1})
    =
    \bar{\boldsymbol g}_{t+1}(\boldsymbol{\theta}_t)
    +
    \boldsymbol{\xi}_{t+1},
    \;\;
    \bigl(
    \nabla_{\boldsymbol{\theta}} g(\boldsymbol{\theta}_t,X_{t+1})
    \bigr)^{\odot 2}
    =
    \boldsymbol{s}_{t+1}(\boldsymbol{\theta}_t)
    +
    \boldsymbol{\chi}_{t+1}.
$$

For any positive definite matrix \(\boldsymbol A\succ0\), define the
\(\boldsymbol A\)-metric projection onto \(\Theta\) by
$
    \mathcal P_\Theta^{\boldsymbol A}(\boldsymbol z)
    :=
    \operatorname*{arg\,min}_{\boldsymbol\theta\in\Theta}
    \frac12
    \|\boldsymbol\theta-\boldsymbol z\|_{\boldsymbol A}^2.
$
Thus, \(\mathcal P_\Theta^{\boldsymbol I_d}\) is the usual Euclidean
projection. Throughout a given run of Algorithm~\ref{alg:adam}, the
projection convention is fixed. For the tracking results, we take
\(\boldsymbol A_{t+1}=\boldsymbol I_d\), whereas for the
metric-projected stationarity results, we take
\(\boldsymbol A_{t+1}=\widetilde{\boldsymbol P}_{t+1}^{-1}\), where
\(\widetilde{\boldsymbol P}_{t+1}\) is the predictable preconditioner (see Appendix~\ref{app:C3}).

\begin{algorithm}[h]
\caption{Adam optimizer}
\label{alg:adam}
\begin{algorithmic}[1]
\Require $\boldsymbol{\theta}_0\in\Theta$, horizon $T\ge1$,
step size $\alpha>0$, $\beta_1,\beta_2\in(0,1)$, and $\epsilon>0$
\State Initialize $\boldsymbol m_0=\boldsymbol v_0=\boldsymbol 0$
\For{$t=0,1,\ldots,T-1$}
    \State $\boldsymbol g_{t+1}
    :=\nabla_{\boldsymbol\theta}g(\boldsymbol\theta_t,X_{t+1})$
    \State $\boldsymbol m_{t+1}
    :=\beta_1\boldsymbol m_t+(1-\beta_1)\boldsymbol g_{t+1}$
    \State $\boldsymbol v_{t+1}
    :=\beta_2\boldsymbol v_t+(1-\beta_2)
    \boldsymbol g_{t+1}^{\odot 2}$
    \State $\widehat{\boldsymbol m}_{t+1}
    :=\boldsymbol m_{t+1}/(1-\beta_1^{t+1})$, \quad
    $\widehat{\boldsymbol v}_{t+1}
    :=\boldsymbol v_{t+1}/(1-\beta_2^{t+1})$
    \State $\boldsymbol P_{t+1}
    :=\operatorname{Diag}\!\left(
    \left(\sqrt{\widehat{\boldsymbol v}_{t+1}}+\epsilon\right)^{-1}
    \right)$
    \State
\begingroup
\renewcommand{\theequation}{Adam}%
\refstepcounter{equation}\label{eq:adam-update}%
$\displaystyle
\boldsymbol{\theta}_{t+1}
:=
\mathcal P_{\Theta}^{\boldsymbol A_{t+1}}\!\left(
    \boldsymbol{\theta}_t
    - \alpha \boldsymbol P_{t+1}\widehat{\boldsymbol m}_{t+1}
\right)
$
\hfill \textnormal{(Adam)}
\endgroup
\EndFor
\end{algorithmic}
\end{algorithm}

To analyze the convergence of \plaineqref{eq:adam-update}, we assume the conditional mean gradient map $\boldsymbol{\theta} \mapsto \bar{\boldsymbol{g}}_{t+1}(\boldsymbol{\theta})$ is uniformly $L$-Lipschitz. Under mild regularity  conditions (e.g.\ sufficient integrability for the Dominated  Convergence Theorem), this implies $G_{t+1}(\boldsymbol{\theta})$ is $L$-smooth.

\vspace{0.5em}

\begin{assumption}[Uniform Lipschitz continuity]
\label{assumption:lipschitz}
There exists a constant $0 < L < \infty$ such that for all $t \geq 0$ and all $\boldsymbol{\theta}, \boldsymbol{\theta}' \in 
\Theta$,
$
    \|\bar{\boldsymbol{g}}_{t+1}(\boldsymbol{\theta}) - 
    \bar{\boldsymbol{g}}_{t+1}(\boldsymbol{\theta}')\| 
    \leq L\|\boldsymbol{\theta} - \boldsymbol{\theta}'\|.
$
\end{assumption}

To obtain high-probability guarantees, we will need to make a standard light-tail assumption on the gradient noise \cite{sahu2026provable, pmlr-v99-harvey19a, NEURIPS2021_62e7f2e0}: 

\vspace{0.5em}

\begin{assumption}[Conditional sub-Gaussian gradient noise along iterates]
\label{assump:cond-subgauss-noise}
There exists a constant $\sigma>0$ such that for all $t\ge 0$, $\big\|\boldsymbol{\xi}_{t+1}(\boldsymbol{\theta}_t)\big\|_{\Psi_2 \mid \mathcal F_t}\le \sigma \;\; \text{a.s.}$
\end{assumption}

\vspace{0.5em}

\begin{remark}
\label{rem:noise-assumptions}
Sub-exponential noise can be handled via Bernstein-type martingale inequalities, while bounded conditional \(q\)-th moments for \(q>2\) yield polynomial-tail guarantees. Predictable time-varying noise levels \(\sigma_t\) are also allowed by replacing uniform variance terms with the corresponding sums of \(\sigma_t^2\). We use the sub-Gaussian formulation for simplicity and consistency with the stochastic
tracking literature.
\end{remark}

\section{Theoretical Results}

\subsection{Tracking under adaptive strong monotonicity}
\label{subsec:tracking-under-adaptive-strong-monotonicity}
In this section, we obtain high-probability guarantees on the tracking error
\(\|\boldsymbol{\theta}_t-\boldsymbol{\theta}_t^\star\|\), deferring
expectation bounds to Appendix~\ref{app:C3}. Unlike \cite{sahu2026provable}, which
requires only that \(\bar{\boldsymbol g}_{t+1}\) itself be strongly monotone,
Adam's preconditioned update requires a compatibility condition between the
preconditioner and the curvature of the conditional risk. Since
\(\boldsymbol P_{t+1}\) depends on the current sample \(X_{t+1}\), it is
\(\mathcal F_{t+1}\)-measurable but not \(\mathcal F_t\)-measurable. Moreover,
\(\bar{\boldsymbol g}_{t+1}\) is the conditional mean-gradient map and is
\(\mathcal F_t\)-measurable, so imposing monotonicity on
\(\boldsymbol P_{t+1}\bar{\boldsymbol g}_{t+1}\) would make the structural
contraction condition depend on a single stochastic realization. We therefore
work with the predictable proxy \(\widetilde{\boldsymbol P}_{t+1}\), obtained
by replacing
\((\nabla_{\boldsymbol{\theta}}g(\boldsymbol{\theta}_t,X_{t+1}))^{\odot 2}\)
in the second-moment update by its conditional mean
\(\boldsymbol s_{t+1}(\boldsymbol\theta_t)\). This separates the predictable
preconditioned mean-gradient geometry from the random preconditioner
perturbation; the precise construction is given in Appendix~\ref{app:C1}.

\vspace{0.5em}

\begin{assumption}[Adaptive strong monotonicity]
\label{assumption:adaptive-strong-monotonicity}. There exists \(0<\mu<\infty\) such that,
almost surely, for all \(t\ge 0\) and all
\(\boldsymbol{\theta},\boldsymbol{\theta}'\in \Theta\),
$
    \left\langle \boldsymbol{\theta}-\boldsymbol{\theta}',\,
    \widetilde{\boldsymbol P}_{t+1}
    \left(
        \bar{\boldsymbol g}_{t+1}(\boldsymbol{\theta})
        -
        \bar{\boldsymbol g}_{t+1}(\boldsymbol{\theta}')
    \right)
    \right\rangle
    \ge
    \mu
    \|\boldsymbol{\theta}-\boldsymbol{\theta}'\|^2 .
$
\end{assumption}

\cref{assumption:adaptive-strong-monotonicity} requires the predictable preconditioner to be compatible with the curvature of the conditional risk. It holds in simple settings such as scalar preconditioning and diagonal quadratics, but also for substantially richer models. We refer the reader to Appendix~\ref{app:C-examples} for examples involving dense generalized linear models, graph-coupled objectives, and locally contracting nonconvex problems. Since \(\widetilde{\boldsymbol P}_{t+1}\) is positive definite, preconditioning preserves the zeros of the mean-gradient map, and strong monotonicity therefore guarantees uniqueness of the moving optimum. Importantly, the condition is imposed on the preconditioned operator and does not require \(G_{t+1}\) itself to be strongly convex.

We can now state the following theorem that establishes the high-probability tracking error for \plaineqref{eq:adam-update} in nonstationary stochastic environments. We defer the proofs for this section to Appendix \ref{app:C}.

\vspace{0.5em}

\begin{theorem}[High probability tracking error bound for \plaineqref{eq:adam-update}]
    \label{thm:high-prob-adam}
    Suppose \cref{assumption:lipschitz}, \cref{assump:cond-subgauss-noise}, and
    \ref{assumption:adaptive-strong-monotonicity} hold and
    \(\operatorname{diam}(\Theta)\le R\). Then, for all \(t\in[T]\),
    \(\delta\in(0,1)\), and
    \(\alpha\le\min\{\mu\epsilon^2/(4L^2),\,\mu^{-1}\}\), the following
    tracking error bound holds for \plaineqref{eq:adam-update} with probability
    at least \(1-\delta\):
    \begin{align*}
        \norm{\boldsymbol\theta_t-\boldsymbol\theta_t^\star}^2
        \lesssim\;&
        \rho_\alpha^t
        \norm{\boldsymbol\theta_0-\boldsymbol\theta_0^\star}^2
        +\frac{\mathfrak D_t}{\alpha\mu}
        + \frac{\alpha}{\mu\epsilon^4}
        \left[
            L^2R^2
            +\sigma^2
            \left(d+\log\frac{T}{\delta}\right)
        \right]^2
        \left[
            \beta_2 t\lambda_2^{t-1}
            +\frac{1-\beta_2}{\alpha\mu}
        \right] \\
        &+
        \frac{1}{\mu^2\epsilon^2}
        \left\{
            \beta_1^2L^2R^2
            +
            \frac{\sigma^2}{1+\beta_1}
            \left(d+\log\frac{T}{\delta}\right)
            \left[
                1-\beta_1
                +\alpha\mu\beta_1 t\lambda_1^{t-1}
            \right]
        \right\}
    \end{align*}
    where
    \(\rho_\alpha:=1-\alpha\mu/2\),
    \(\mathfrak D_t:=\sum_{\ell=0}^{t-1}
    \rho_\alpha^{t-\ell-1}\Delta_\ell^2\),
    \(\Delta_\ell
    :=\norm{\boldsymbol\theta_\ell^\star-\boldsymbol\theta_{\ell+1}^\star}\),
    and \(\lambda_i:=\rho_\alpha\vee\beta_i\) for \(i\in\{1,2\}\).
\end{theorem}

\paragraph{Comparison with SGD, SGDM, and implications for nonstationary optimization.}
To make the comparison with Adam self-contained, we recall the corresponding high-probability tracking guarantee for constant-stepsize SGD from \cite{sahu2026provable}:

\vspace{0.5em}

\begin{theorem}[High probability tracking error bound for SGD]
    \label{thm:high-prob-sgd}
    Assume $G_{t}$ is $\mu$-strongly convex and $L$-smooth. Under \cref{assump:cond-subgauss-noise}, for all $t \in [T]$, $\alpha \leq \min \left\{ \mu / L^2, 1/L \right\}$, and $\delta \in (0, 1)$, the following tracking error bound holds for SGD with probability at least $1-\delta$:
    \begin{equation*}
        \norm{\boldsymbol{\theta}_{t} - \boldsymbol{\theta}_{t}^{\star}}^2 
        \lesssim 
        \left(1 - \frac{\alpha \mu}{2} \right)^{t} 
        \norm{\boldsymbol{\theta}_{0} - \boldsymbol{\theta}_{0}^{\star}}^2
        + \frac{\mathfrak{D}_t}{\alpha \mu}
        + \frac{d\sigma^2\alpha}{\mu}
        + d\sigma^2\alpha^2\log\frac{T}{\delta}
        + \left(
            \frac{\sigma^2\alpha}{\mu}
            + \alpha^2\sigma^2\mathfrak D_t^{(2)}
          \right)\log\frac{T}{\delta},
    \end{equation*}
    where $\mathfrak{D}_t := \sum_{\ell=0}^{t-1} (1 - \alpha \mu / 2)^{t-\ell-1}\Delta_{\ell}^2$ and $\mathfrak{D}_t^{(2)} := \sum_{\ell=0}^{t-1} (1 - \alpha \mu / 2)^{2(t-\ell-1)}\Delta_{\ell}^2$.
\end{theorem}

Our Adam bound has the same basic decomposition as the SGD and SGDM guarantees of \cite{sahu2026provable}: an exponentially decaying initialization error, a discounted drift term, and stochastic terms determining the eventual tracking floor.

Adam separates the effects of first and second-moment memory. The parameter $\beta_1$ controls a smoothing--responsiveness tradeoff: its worst-case deterministic memory-bias contribution scales as $\beta_1^2$, while its stochastic contribution is reduced by the averaging factor $(1-\beta_1)/(1+\beta_1)$. Thus larger $\beta_1$ suppresses gradient noise without causing the bias to diverge, but increases the effective memory horizon, so stale gradients persist longer under nonstationarity. The remaining $\beta_1t\lambda_1^{t-1}$ term is a finite-time transient, with $\lambda_1=\rho_\alpha\vee\beta_1$, and vanishes geometrically. Consequently, larger $\beta_1$ is most favorable in relatively stable, noise-dominated regimes, whereas smaller $\beta_1$ permits faster adaptation when the objective changes rapidly.

The parameter $\beta_2$ induces an analogous transient--floor tradeoff through the adaptive preconditioner. Its persistent contribution scales with $1-\beta_2$, while the bias-correction transient scales as $\beta_2t\lambda_2^{t-1}$, where $\lambda_2=\rho_\alpha\vee\beta_2$. Hence increasing $\beta_2$ reduces the long-run perturbation induced by the stochastic second-moment estimate, but slows the decay of the corresponding transient. Large $\beta_2$ is therefore advantageous over long and comparatively stable horizons, while smaller $\beta_2$ can be preferable when rapid adaptation is more important.

Our bounds also reveal a tradeoff governed by $\epsilon$. Increasing $\epsilon$ permits a larger admissible stepsize through $\alpha\le \mu\epsilon^2/(4L^2)$ and reduces the first and second-moment adaptive penalties, which scale as $\mathcal{O}(\epsilon^{-2})$ and $\mathcal{O}(\epsilon^{-4})$, respectively. Thus larger $\epsilon$ can improve tracking when stochastic adaptive-preconditioning effects dominate, while its overall benefit still depends on how the preconditioned geometry, and hence $\mu$, changes with $\epsilon$. This provides a theoretical mechanism for the empirical observation of \cite{lyle2023understanding, lee2024slow, bai2025adaptive, gogianu2026revisiting} that increasing $\epsilon$ can stabilize Adam under nonstationarity.

After the initialization and bias-correction transients decay, the Adam tracking floor consists of four principal contributions: drift of order $\Delta^2/(\mu^2\alpha^2)$, first-moment memory bias of order $\beta_1^2L^2R^2/(\mu^2\epsilon^2)$, first-moment variance of order $\sigma^2(1-\beta_1)(d+\log(N/\delta))/[\mu^2\epsilon^2(1+\beta_1)]$, and second-moment preconditioner perturbation of order $(1-\beta_2)[L^2R^2+\sigma^2(d+\log(N/\delta))]^2/(\mu^2\epsilon^4)$. Thus, increasing $\beta_1$ suppresses stochastic gradient noise but increases the persistent bias from stale first-moment information, while increasing $\beta_2$ reduces the persistent error induced by fluctuations in the adaptive preconditioner. Both forms of memory can also slow adaptation: the corresponding transients decay at rates governed by $\lambda_{1,\alpha}=\rho_\alpha\vee\beta_1$ and $\lambda_{2,\alpha}=\rho_\alpha\vee\beta_2$, as reflected in the burn-in terms $\mathcal B_1$ and $\mathcal B_2$. Hence larger $\beta_1$ and $\beta_2$ improve temporal averaging and preconditioner stability, respectively, at the cost of slower response to nonstationarity.

\vspace{0.5em}

\begin{theorem}[Time to reach the tracking floor with high probability for \plaineqref{eq:adam-update}]
\label{thm:time-to-track-hp-adam}
Under the standing assumptions of
\cref{thm:high-prob-adam}, suppose
$0<\alpha \le \alpha_{\max} := \min\{\mu\epsilon^2/(4L^2),\,\mu^{-1}\}$,
$\Delta_t \le \Delta$ almost surely for all $t\ge 0$, and fix a
deterministic horizon bound $N\ge1$ and $\delta\in(0,1)$.
Define the tracking-floor
$$
\mathcal{E}_{\mathrm{A}}(N,\delta,\alpha)
\asymp 
\frac{\Delta^2}{\mu^2\alpha^2}
+
\frac{\beta_1^2L^2R^2}{\mu^2\epsilon^2}
+
\frac{C^2\sigma^2(1-\beta_1)}
{\mu^2\epsilon^2(1+\beta_1)}
\left(d+\log\frac{N}{\delta}\right)
+
\frac{C(1-\beta_2)}{\mu^2\epsilon^4}
\left[
    L^2R^2
    +
    \sigma^2\left(d+\log\frac{N}{\delta}\right)
\right]^2,
$$
where $C>0$ is a fixed sufficiently large universal constant.
For $\alpha\in(0,\alpha_{\max}]$, let
$\rho_\alpha:=1-\alpha\mu/2$ and
$\lambda_{i,\alpha}:=\rho_\alpha\vee\beta_i$ for $i\in\{1,2\}$.
For a target floor $E>0$, define
\begin{align*}
\mathcal{B}_1(\alpha,E)
&\asymp
\frac{1}{1-\lambda_{1,\alpha}}
\log_{+}\!\left(
    \frac{
        C^2\alpha\sigma^2
        \left(d+\log\frac{N}{\delta}\right)
    }{
        \mu\epsilon^2(1+\beta_1)
        (1-\lambda_{1,\alpha})E
    }
\right) \\
\mathcal{B}_2(\alpha,E)
&\asymp
\frac{1}{1-\lambda_{2,\alpha}}
\log_{+}\!\left(
    \frac{
        C\alpha
        \left[
            L^2R^2
            +
            \sigma^2\left(d+\log\frac{N}{\delta}\right)
        \right]^2
    }{
        \mu\epsilon^4
        (1-\lambda_{2,\alpha})E
    }
\right).
\end{align*}
Then we have the following:
\begin{enumerate}
    \item \textbf{(Constant learning rate).} If $\alpha_t\equiv\alpha$,
    with $\alpha$ chosen deterministically, then
    for any integer $1\le T\le N$, with probability $\ge 1-\delta$,
    for all $t\in[T]$,
    $$
    \begin{aligned}
        \|\boldsymbol{\theta}_t-\boldsymbol{\theta}_t^\star\|^2
        \;\lesssim\;&
        \rho_\alpha^t
        \|\boldsymbol{\theta}_0-\boldsymbol{\theta}_0^\star\|^2
        +
        \mathcal{E}_{\mathrm{A}}(N,\delta,\alpha)
        +
        \frac{\alpha\sigma^2}
        {\mu\epsilon^2(1+\beta_1)}
        \left(d+\log\frac{N}{\delta}\right)
        \sum_{\ell=0}^{t-1}
        \rho_\alpha^{t-\ell-1}\beta_1^{\ell+1}
        \\
        &+
        \frac{\alpha}{\mu\epsilon^4}
        \left[
            L^2R^2
            +\sigma^2\left(d+\log\frac{N}{\delta}\right)
        \right]^2
        \sum_{\ell=0}^{t-1}
        \rho_\alpha^{t-\ell-1}\beta_2^{\ell+1}.
    \end{aligned}
    $$
    Let $\alpha_{\mathrm{A}}^\star\in(0,\alpha_{\max}]$
    be a deterministic target stepsize and
    $
        \mathcal{E}_{\mathrm{A}}^\star
        :=
        \mathcal{E}_{\mathrm{A}}
        (N,\delta,\alpha_{\mathrm{A}}^\star).
    $
    Then, running Adam with $\alpha=\alpha_{\mathrm A}^\star$ gives
    $\|\boldsymbol{\theta}_t-\boldsymbol{\theta}_t^\star\|^2
    \lesssim
    \mathcal{E}_{\mathrm{A}}^\star$
    after at most
    $$
    \begin{aligned}
        t_{\mathrm A}^\star
        \;\lesssim\;&
        \frac{1}{\mu\alpha_{\mathrm{A}}^\star}
        \log_{+}\!\left(
            \frac{
                \|\boldsymbol{\theta}_0
                -\boldsymbol{\theta}_0^\star\|^2
            }{
                \mathcal{E}_{\mathrm{A}}^\star
            }
        \right)
        +
        \mathcal{B}_1
        \left(
            \alpha_{\mathrm{A}}^\star,
            \mathcal{E}_{\mathrm{A}}^\star
        \right)
        +
        \mathcal{B}_2
        \left(
            \alpha_{\mathrm{A}}^\star,
            \mathcal{E}_{\mathrm{A}}^\star
        \right)
    \end{aligned}
    $$
    iterations, provided $t_{\mathrm A}^\star\le T$.
    The bound holds simultaneously for
    $t_{\mathrm A}^\star\le t\le T$ with probability $\ge1-\delta$.

    \item \textbf{(Step-decay with Adam-state restart).}
    Suppose a deterministic upper bound
    $D\ge \|\boldsymbol{\theta}_0-\boldsymbol{\theta}_0^\star\|^2$
    is available and $\alpha_{\mathrm{A}}^\star<\alpha_{\max}$.
    Set $\alpha_0:=\alpha_{\max}$,
    $\alpha_k:=(\alpha_{k-1}+\alpha_{\mathrm{A}}^\star)/2$, and
    $K:=1+\lceil\log_2(\alpha_0/\alpha_{\mathrm{A}}^\star)\rceil$.
    For each $k\geq0$, let
    $
        E_k
        :=
        \mathcal{E}_{\mathrm{A}}(N,\delta,\alpha_k),
    $
    and set
    $$
    \begin{aligned}
        T_0
        &\asymp
        \left\lceil
        \max\left\{
            \frac{1}{\mu\alpha_0}
            \log_{+}\!\left(\frac{D}{E_0}\right),
            \,
            \mathcal{B}_1(\alpha_0,E_0),
            \,
            \mathcal{B}_2(\alpha_0,E_0)
        \right\}
        \right\rceil,
        \\
        T_k
        &\asymp
        \left\lceil
        \max\left\{
            \frac{1}{\mu\alpha_k},
            \,
            \mathcal{B}_1(\alpha_k,E_k),
            \,
            \mathcal{B}_2(\alpha_k,E_k)
        \right\}
        \right\rceil,
        \qquad k\geq1.
    \end{aligned}
    $$
    Suppose $T:=\sum_{k=0}^{K-1}T_k\le N$.
    Running Adam at constant stepsize $\alpha_k$ for $T_k$ steps per
    epoch with $(\boldsymbol m,\boldsymbol v)$ and the bias-correction restarted at the beginning of every epoch yields
    $
        \|\boldsymbol{\theta}_T-\boldsymbol{\theta}_T^\star\|^2
        \lesssim
        \mathcal{E}_{\mathrm{A}}^\star
    $
    with probability $\geq1-K\delta$
    Moreover,
    $$
    \begin{aligned}
        T
        \;\lesssim\;
        \frac{1}{\mu\alpha_0}
        \log_{+}\!\left(
            \frac{D}{\mathcal{E}_{\mathrm{A}}^\star}
        \right)
        +
        \frac{1}{\mu\alpha_{\mathrm{A}}^\star}
        +
        \sum_{k=0}^{K-1}
        \mathcal{B}_1(\alpha_k,E_k)
        +
        \sum_{k=0}^{K-1}
        \mathcal{B}_2(\alpha_k,E_k).
    \end{aligned}
    $$
\end{enumerate}
\end{theorem}

\section{Metric-projected stationarity under general Adam preconditioning}
\label{sec:metric-projected-stationarity}

The preceding analysis establishes Euclidean tracking guarantees under adaptive strong monotonicity. We now drop this condition and instead study stationarity through the \(\widetilde{\boldsymbol P}_{t+1}^{-1}\)-metric projected-gradient mapping \(\mathcal G_{\alpha,t}(\boldsymbol\theta_t)\), the natural constrained analogue of the preconditioned gradient (see Appendix~\ref{app:E1} for details). We derive high-probability bounds on its average squared norm under nonstationary objectives and defer the corresponding expectation bounds to Appendix~\ref{app:D4}. Throughout this section, we additionally assume that \(G_t(\boldsymbol\theta)\) is uniformly bounded below by \(G^\star\), as standard in nonconvex optimization \cite{ReddiKK18, defossez2022simple, li2023convergence}.

\vspace{0.5em}

\begin{assumption}[Uniform lower boundedness]
\label{assumption:lower-bounded}
There exists a constant $G^\star > -\infty$ such that for all 
$t \in [T]$ and all $\boldsymbol{\theta} \in \Theta$, $G_t(\boldsymbol{\theta}) \geq G^\star$.
\end{assumption}

We now state our main result for Adam with general adaptive preconditioners, which gives a high-probability bound on the average projected stationarity gap of the projected Adam update under non-stationary objectives. All proofs for this section are deferred to Appendix~\ref{app:D}.

\vspace{0.5em}

\begin{theorem}[High-probability projected-gradient rate]
\label{cor:nonconvex-adam-hp}
Suppose
\(\operatorname{diam}(\Theta)\le R\). Then, for all integers \(T\ge1\)
and all \(\delta\in(0,1)\), if \(\alpha\le\epsilon/(4L)\), the iterates
generated by the projected Adam update satisfy,  with
probability at least \(1-\delta\),
\begin{equation}
    \frac{1}{T}
    \sum_{t=0}^{T-1}
    \left\|
        \mathcal G_{\alpha,t}(\boldsymbol\theta_t)
    \right\|_{\widetilde{\boldsymbol P}_{t+1}^{-1}}^2
    \lesssim
    \frac{
        G_1(\boldsymbol\theta_0)-G^\star+\mathfrak D_T
    }{\alpha T}
    +
    \frac{\mathsf{Decay}_T(\delta)}{T}
    +
    \mathsf{Floor}_T(\delta),
\end{equation}
where
\begin{align*}
    \mathsf{Decay}_T(\delta)
    &\sim \mathcal{O} \left\{
    \frac{
        \beta_2
    }{
        \epsilon^3(1-\beta_2)
    }
    \left[
        L^2R^2
        +
        \sigma^2
        \left(
            d+\log\frac{T}{\delta}
        \right)
    \right]^2
    +
    \frac{
        \sigma^2
        \left(
            d+\log\frac{T}{\delta}
        \right)
        (1+\log T)
    }{
        \epsilon(1+\beta_1)
    } \right\}, \\
    \mathsf{Floor}_T(\delta)
    &\sim \mathcal{O} \left \{ 
    \frac{
        L^2R^2\beta_1^2
    }{
        \epsilon
    }
    +
    \frac{
        \sigma^2(1-\beta_1)
    }{
        \epsilon(1+\beta_1)
    }
    \left(
        d+\log\frac{T}{\delta}
    \right)
    +
    \frac{
        1-\beta_2
    }{
        \epsilon^3
    }
    \left[
        L^2R^2
        +
        \sigma^2
        \left(
            d+\log\frac{T}{\delta}
        \right)
    \right]^2 \right \}.
\end{align*}
\end{theorem}

The explicit rate in \cref{cor:nonconvex-adam-hp} separates objective
variation, finite-horizon transients, and persistent Adam errors. The term
\((G_1(\boldsymbol\theta_0)-G^\star+\mathfrak D_T)/(\alpha T)\) vanishes
whenever \(\mathfrak D_T=o(T)\), and at the usual \(1/T\) rate when
\(\mathfrak D_T=O(1)\). Under persistent nonstationarity with
\(\mathfrak D_T=\Theta(T)\), this contribution need not vanish.
The term \(\mathsf{Decay}_T(\delta)/T\) collects the remaining
bias-correction and finite-horizon effects and vanishes up to logarithmic
factors.

The persistent term \(\mathsf{Floor}_T(\delta)\) makes the roles of
\(\beta_1\) and \(\beta_2\) explicit. The first-moment parameter
\(\beta_1\) trades off a memory bias of order
\(\beta_1^2\) against a variance contribution proportional to
\((1-\beta_1)/(1+\beta_1)\). Thus larger \(\beta_1\) improves noise
averaging but increases the cost of stale gradients. Similarly,
\(\beta_2\) induces a transient--floor tradeoff: the finite-horizon
transient scales as \(\beta_2/(1-\beta_2)\), while the persistent
preconditioner perturbation scales as \(1-\beta_2\). Hence larger
\(\beta_2\) improves the long-run preconditioner floor at the cost of
slower finite-horizon adaptation. Together, these terms quantify the
bias--variance and stability--adaptivity tradeoffs governing Adam under
nonstationarity.


\section{Numerical Experiments}
\label{sec:numerical-experiments}
We compare SGD and Adam on four online problems: strongly convex least squares and three non-convex models, namely teacher--student MLP regression, phase retrieval, and matrix factorization. In all experiments, the target evolves via a normalized random walk,
\(
\boldsymbol{\theta}_{t+1}^\star
=
\boldsymbol{\theta}_t^\star
+
\Delta_t \frac{\boldsymbol{u}_t}{\|\boldsymbol{u}_t\|},
\,
\boldsymbol{u}_t \sim \mathcal{N}(\boldsymbol 0,\boldsymbol I_d),
\)
where \(\Delta_t\) controls drift and \(\sigma_t\) controls stochastic observation noise. For each problem, we show high-drift/low-noise (left) and low-drift/high-noise (right) regimes. SGD performs better in drift-dominated settings due to faster adaptation, while Adam is more effective in noise-dominated regimes via momentum and adaptive scaling. Additional experimental details are provided in Appendix~\ref{app:F}.

\begin{figure}[H]
    \centering
    \includegraphics[width=0.75\textwidth]{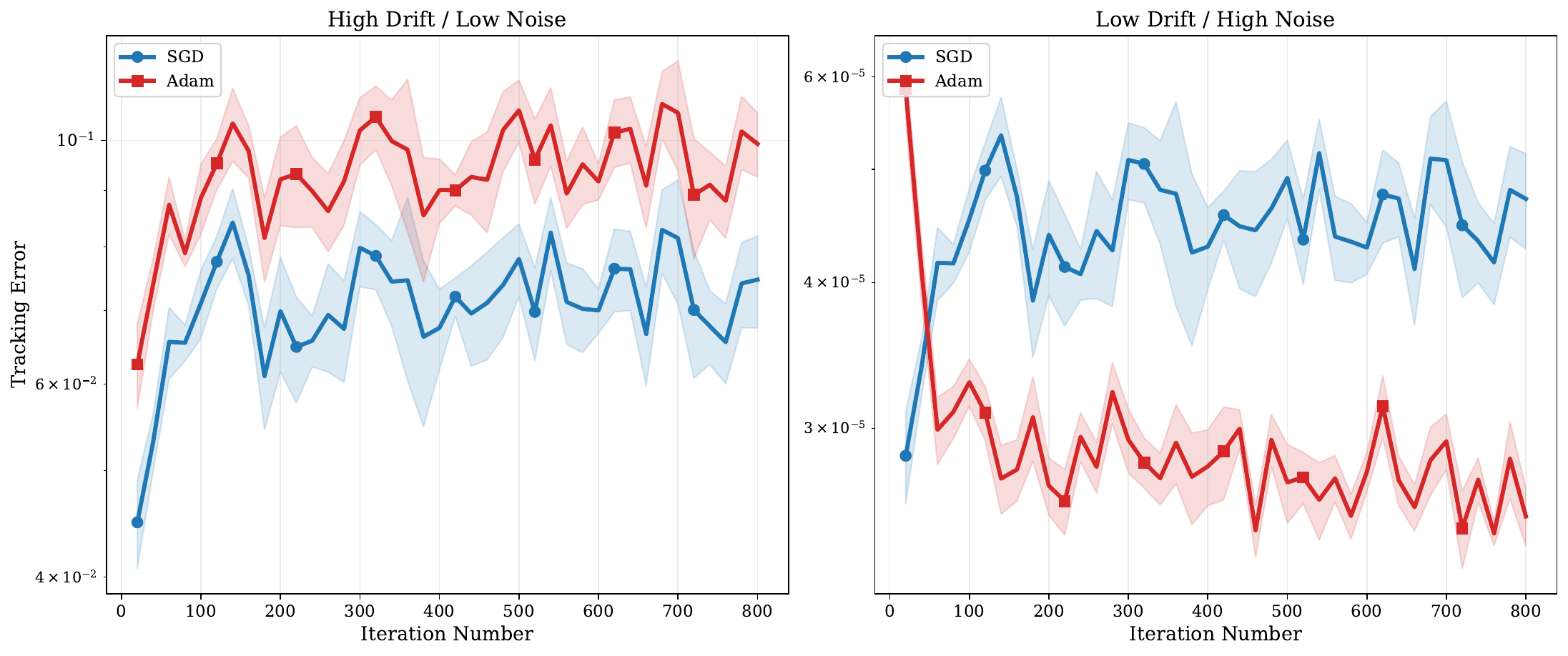}
    \caption{Strongly convex least squares with objective
\(F_t(\boldsymbol{\theta})=\frac{1}{2}\|\boldsymbol A(\boldsymbol{\theta}-\boldsymbol{\theta}_t^\star)\|^2\).
We report tracking error \(\|\boldsymbol{\theta}_t-\boldsymbol{\theta}_t^\star\|^2\).
\textbf{Left}: high-drift, low-noise regime with
\(\Delta_t=c_2\log(t+2)\) and \(\sigma_t=c_1\log(t+2)\).
\textbf{Right}: low-drift, high-noise regime with
\(\Delta_t=c_1\log(t+2)\) and \(\sigma_t=c_2\log(t+2)\),
for positive constants \(c_1<c_2\).}
\label{fig:diagram_ols}
    \label{fig:diagram_ols}
\end{figure}

\begin{figure}[H]
    \centering
    \includegraphics[width=0.75\textwidth]{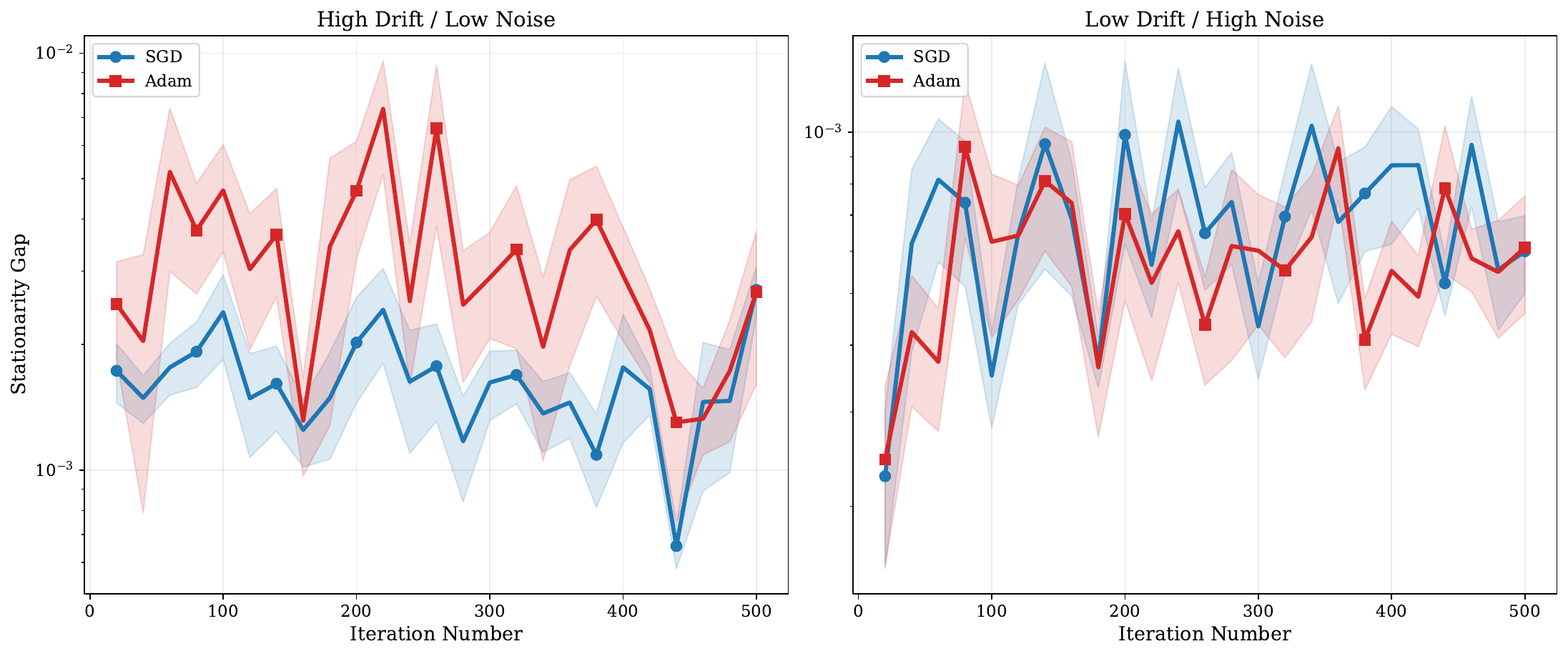}
    \caption{Teacher--student MLP regression. The objective is
\(
F_t(\boldsymbol{\theta}) = \frac{1}{2}
\mathbb{E}_{\boldsymbol{x}} \left[
\left( f_{\boldsymbol{\theta}}(\boldsymbol{x}) - f_{\boldsymbol{\theta}_t^\star}(\boldsymbol{x}) \right)^2
\right],
\)
and we report the stationarity gap.
\textbf{Left}: high-drift, low-noise regime with
\(\Delta_t=c_2\log(t+2)\) and \(\sigma_t=c_1\log(t+2)\).
\textbf{Right}: low-drift, high-noise regime with
\(\Delta_t=c_1\log(t+2)\) and \(\sigma_t=c_2\log(t+2)\),
for positive constants \(c_1<c_2\).}
    \label{fig:diagram_mlp}
\end{figure}

\begin{figure}[!hbtp]
    \centering
    \includegraphics[width=0.75\textwidth]{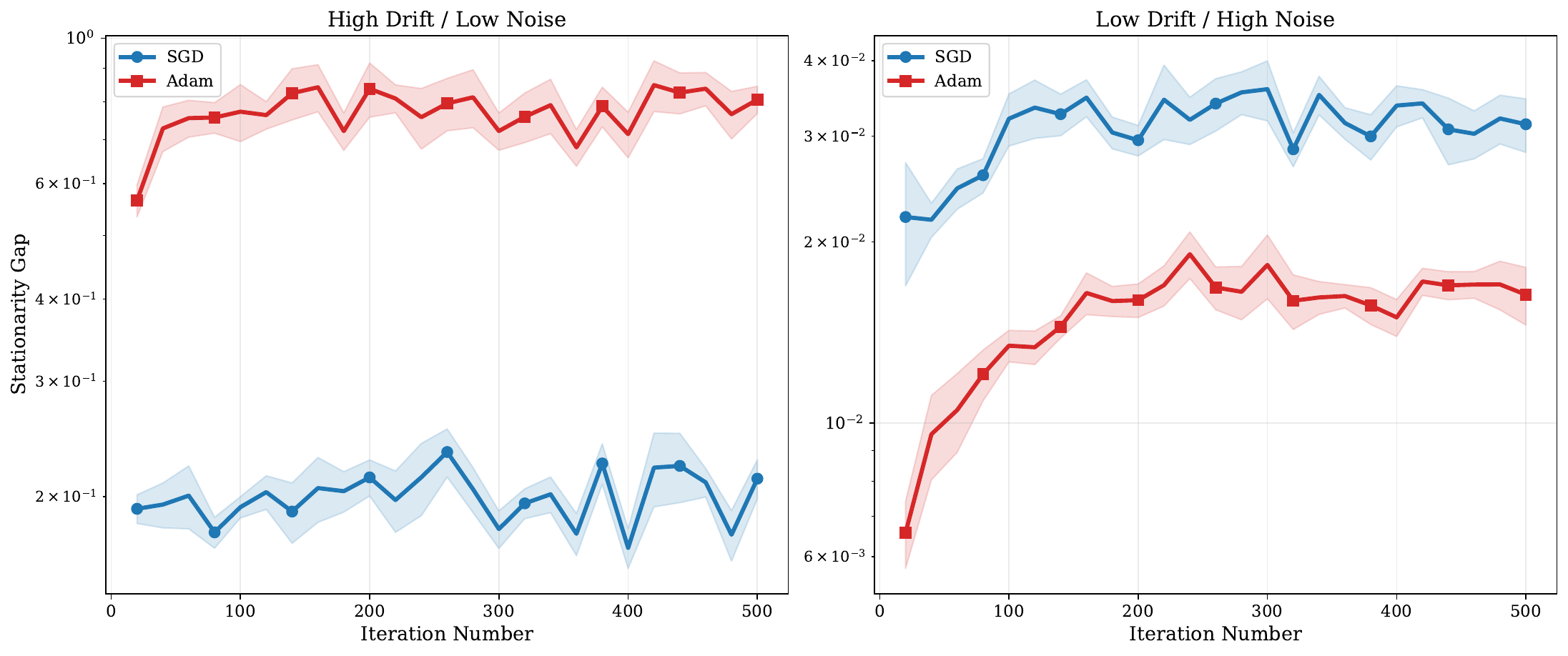}
    \caption{Phase retrieval. The objective is
\(
F_t(\boldsymbol{\theta})
=
\frac{1}{2}
\mathbb{E}_{\boldsymbol{x}}
\left[
\left(
(\boldsymbol{x}^{\top}\boldsymbol{\theta})^2
-
(\boldsymbol{x}^{\top}\boldsymbol{\theta}_t^\star)^2
\right)^2
\right],
\)
and we report the stationarity gap.
\textbf{Left}: high-drift, low-noise regime with
\(\Delta_t=c_{2}\log(t+2)\) and \(\sigma_t=c_{1}\log(t+2)\).
\textbf{Right}: low-drift, high-noise regime with
\(\Delta_t=c_1\log(t+2)\) and \(\sigma_t=c_2\log(t+2)\),
for positive constants \(c_1<c_2\).}
    \label{fig:diagram_phase}
\end{figure}

\begin{figure}[H]
    \centering
    \includegraphics[width=0.75\textwidth]{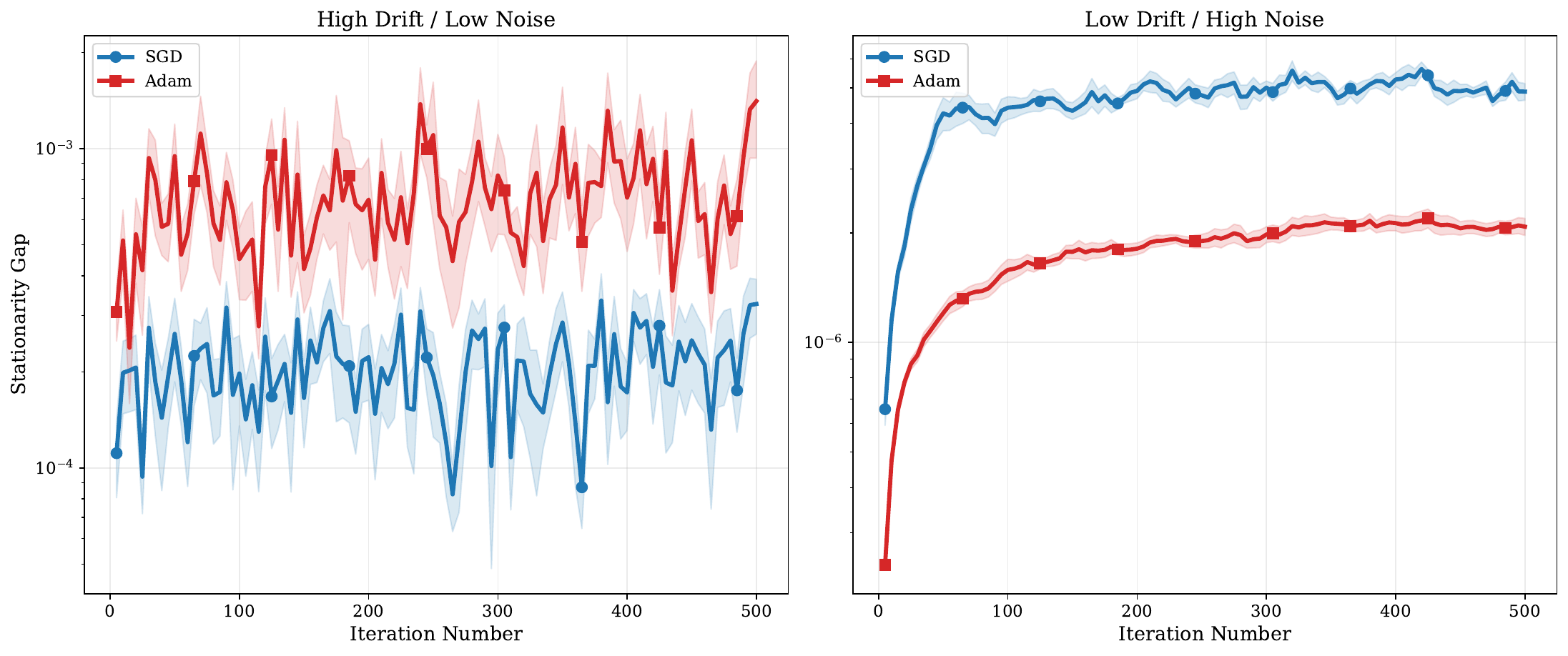}
\caption{Matrix factorization with objective
\(
F_t(\boldsymbol U,\boldsymbol V)
=
\frac{1}{2mn}
\left\|
\boldsymbol U\boldsymbol V^{\top}
-
\boldsymbol M_t^\star
\right\|_F^2,
\)
and we report the stationarity gap.
\textbf{Left}: high-drift, low-noise regime with
\(\Delta_t=c_2\) and \(\sigma_t=c_1\).
\textbf{Right}: low-drift, high-noise regime with
\(\Delta_t=c_3\) and \(\sigma_t=c_4\),
where \(0<c_1<c_2\) and \(0<c_3<c_4\).}
    \label{fig:diagram_matrix}
\end{figure}

\section{Conclusion}
\label{sec:conclusion}
We provide a finite-time theoretical analysis of Adam under non-stationary stochastic objectives, characterizing how its moment estimates interact with distribution shift. Our bounds decompose the tracking error into four interpretable components (initialization, objective drift, first-moment memory bias, and preconditioner perturbation) and reveal a fundamental noise--drift tradeoff: Adam's adaptive mechanisms can yield more favorable tracking guarantees when gradient noise dominates, whereas under substantial drift, stale first-moment information and preconditioner perturbations can enlarge Adam's tracking bound, potentially allowing vanilla SGD to attain a smaller tracking error. The explicit $(\beta_1, \beta_2, \epsilon)$-dependence of our bounds provides a theoretical explanation for the empirical instability of Adam under distribution shift, including the stabilizing effect of increasing $\epsilon$ under task changes, and our experimental results corroborate the predicted tradeoffs across synthetic nonstationary settings. 

For future work, it would be interesting to establish minimax lower bounds for Adam under nonstationarity, both to characterize the best achievable tracking rates and to determine whether the dependence of our upper bounds on drift, noise, and \((\beta_1,\beta_2,\epsilon)\) is minimax optimal. The lower bounds of \cite{sahu2026provable} for SGD and SGDM rely on strongly convex constructions, and extending such results to the general \(L\)-smooth setting considered here would likely require different techniques. Another promising direction is to develop adaptive optimizers that can switch between update rules or parameter regimes so as to minimize cumulative tracking error.

\clearpage %
\bibliographystyle{alpha}
\bibliography{refs}

\clearpage 
\clearpage
\appendix
\crefalias{section}{appendix}
\crefalias{subsection}{appendix}

\part*{Appendix} 
\addcontentsline{toc}{part}{Appendix} 

\begingroup
\etocsettocstyle{%
  \section*{Table of Contents}%
  \vspace{-0.25em}%
  \noindent\rule{\linewidth}{0.4pt}\par
  \vspace{0.75em}%
}{}
\localtableofcontents
\endgroup
\clearpage

\section{Related work}
\label{app:A}

\paragraph{Adam in stationary settings.}
\label{app:A1}
Since its introduction \cite{DBLP:journals/corr/KingmaB14}, Adam has been the subject of extensive theoretical study in stationary settings. Early work showed that Adam can fail to converge even in simple convex problems due to the exponential moving average of squared gradients, motivating the AMSGrad variant \cite{ReddiKK18}. Subsequent analyses established convergence guarantees under progressively weaker assumptions: \cite{defossez2022simple} provide a clean proof for smooth nonconvex objectives with bounded gradients, obtaining the tightest known dependence on $\beta_1$; \cite{li2023convergence} prove convergence to $\epsilon$-stationary points under a generalized smoothness condition without requiring globally bounded gradients; and \cite{hong2024convergence} extend this to affine variance noise models, showing that Adam is free to tune step sizes without knowledge of problem parameters. Most recently, \cite{taniguchi2024adopt} propose ADOPT, a modification of Adam that achieves the optimal $\mathcal{O}(1/\sqrt{T})$ nonconvex rate with any choice of $\beta_2$ without relying on bounded noise, by removing the current gradient from the second-moment estimate. The adaptivity of Adam relative to SGD has also been studied, with work identifying heavy-tailed gradient distributions and ill-conditioned objectives as settings where Adam's coordinate-wise preconditioning yields provable advantages 
\cite{kunstner2024heavytailed, zhang2024transformers, kunstner2023noise}. Complementary recent work by \cite{jin2026adam} gives a high-probability separation between Adam and SGD for stationary stochastic optimization under bounded-variance noise, showing that Adam's second-moment normalization can yield sharper confidence dependence than SGD through a stopping-time martingale analysis. All of these results assume a fixed stationary objective and none address the behavior of the preconditioner under distribution shift.

\paragraph{SGD and momentum in non-stationary settings.}
\label{app:A2}
A complementary line of work studies optimization under time-varying objectives and distribution shift. The foundational dynamic regret framework of \cite{zinkevich2003online} measures performance against a drifting comparator in online convex optimization, and has inspired a large body of follow-up work on efficient non-stationary online learning \cite{JMLR:v25:21-0748}. In stationary settings with constant step sizes, momentum methods are essentially equivalent at steady state to SGD with a rescaled learning rate and do not reduce the MSE floor \cite{JMLR:v17:16-157, JMLR:v22:19-466}. Under non-stationarity, this equivalence breaks down: past gradients become stale and may point in the wrong direction, making temporal averaging actively harmful. \cite{8814889} derive dynamic regret bounds for online SGD under time-varying distributions, while \cite{JMLR:v24:21-1410} establish non-asymptotic tracking guarantees for proximal stochastic gradient methods under time drift. \cite{shen2026sgddependentdataoptimal} establish optimal bounds for SGD under temporally dependent data. Most directly related to our work, \cite{sahu2026provable} provide finite-time tracking bounds for SGD and its momentum variants, including Polyak's Heavy-Ball and Nesterov acceleration, under strongly convex smooth objectives, showing explicitly that momentum amplifies drift-induced tracking lag and establishing minimax lower bounds confirming this penalty is unavoidable. Nonconvex tracking guarantees under distribution shift remain scarce even for SGD, and our work is the first to establish such guarantees for Adam. Our work extends this line of inquiry to Adam, where the adaptive preconditioner introduces an additional and qualitatively distinct source of non-stationary error.

\paragraph{Adam and adaptive methods under non-stationarity.}
\label{app:A3}
The failure of Adam under distribution shift has been documented empirically in several settings. \cite{lyle2023understanding} study plasticity loss in neural networks under non-stationary objectives and find that Adam's moment estimates degrade under task changes, recommending increased $\epsilon$ and more aggressive second-moment decay as fixes — a phenomenon our bounds now explain theoretically. \cite{dohare2023maintaining} show empirically that loss of plasticity in continual learning is worse for Adam than for SGD, and \cite{ellis2024adam} directly analyze how non-stationary gradient magnitudes in reinforcement learning cause Adam's updates to become excessively large, proposing a local-timestep variant to address this. Despite this empirical evidence, no prior work provides a theoretical framework that precisely characterizes when and why Adam degrades under distribution shift, how its $(\beta_1, \beta_2, \epsilon)$ hyperparameters interact with the drift-noise balance, or how it compares to SGD in a unified tracking setting. This paper fills that gap.

\section{Problem setup, standing assumptions, and notation}
\label{app:B}
We work on a filtered probability space
$(\Omega, \mathcal{F}, (\mathcal{F}_t)_{t \geq 0}, \mathbb{P})$ with
$\mathcal{F}_0 = \{\emptyset, \Omega\}$, and let $(X_t)_{t \geq 0}$ be
an $\mathbb{F}$-adapted process. We take the natural filtration
$\mathcal{F}_t := \sigma(X_0, \ldots, X_t)$. For each $t \geq 0$,
conditional mean gradient and conditional second moment by
\[
\bar{\boldsymbol{g}}_{t+1}(\boldsymbol{\theta})
:= \mathbb{E}\!\left[\nabla_{\boldsymbol{\theta}} g(\boldsymbol{\theta},
X_{t+1}) \mid \mathcal{F}_t\right],
\;\;
\boldsymbol{s}_{t+1}(\boldsymbol{\theta})
:= \mathbb{E}\!\left[(\nabla_{\boldsymbol{\theta}} g(\boldsymbol{\theta},
X_{t+1}))^{\odot 2} \mid \mathcal{F}_t\right].
\]
We write $\boldsymbol{\theta}_{t+1}^\star \in \arg\min_{\boldsymbol{\theta}
\in \Theta} G_{t+1}(\boldsymbol{\theta})$ for a (measurable)
conditional minimizer where $\Theta \subset \mathbb{R}^{d}$ is compact and convex, and define the tracking error, minimizer drift,
and shifted error as
\[
\boldsymbol{e}_t := \boldsymbol{\theta}_t - \boldsymbol{\theta}_t^\star,
\;\;
\boldsymbol{\Delta}_t := \boldsymbol{\theta}_t^\star -
\boldsymbol{\theta}_{t+1}^\star,
\;\;
\boldsymbol{d}_t := \boldsymbol{\theta}_t - \boldsymbol{\theta}_{t+1}^\star
= \boldsymbol{e}_t + \boldsymbol{\Delta}_t.
\]
The associated first and second moment noise terms are
\[
\boldsymbol{\xi}_{t+1} :=
\nabla_{\boldsymbol{\theta}} g(\boldsymbol{\theta}_t, X_{t+1}) -
\bar{\boldsymbol{g}}_{t+1}(\boldsymbol{\theta}_t),
\;\;
\boldsymbol{\chi}_{t+1} :=
(\nabla_{\boldsymbol{\theta}} g(\boldsymbol{\theta}_t,
X_{t+1}))^{\odot 2} - \boldsymbol{s}_{t+1}(\boldsymbol{\theta}_t),
\]
both of which are $\mathcal{F}_{t+1}$-measurable martingale differences
satisfying $\mathbb{E}[\boldsymbol{\xi}_{t+1} \mid \mathcal{F}_t] =
\mathbf{0}$ and $\mathbb{E}[\boldsymbol{\chi}_{t+1} \mid \mathcal{F}_t]
= \mathbf{0}$ a.s.

\vspace{0.5em}

\begin{assumption}[Stochastic predictability framework]
\label{assump:filtered-predictable-mds-restate}
There exists a filtered probability space $(\Omega, \mathcal{F},
(\mathcal{F}_t)_{t \geq 0}, \mathbb{P})$ with $\mathcal{F}_0 =
\{\emptyset, \Omega\}$. Let $(X_t)_{t \geq 0}$ be an
$\mathbb{F}$-adapted process, i.e., $X_t$ is
$\mathcal{F}_t$-measurable for all $t$. For each $t \geq 0$, let
$\Pi_{t+1}$ denote the regular conditional law of $X_{t+1}$ given
$\mathcal{F}_t$, i.e., $\Pi_{t+1}(A) = \mathbb{P}(X_{t+1} \in A \mid
\mathcal{F}_t)$ a.s.\ for every measurable set $A$, and assume
$\Pi_{t+1}$ is $\mathcal{F}_t$-measurable. Define the conditional risk
\[
G_{t+1}(\boldsymbol{\theta})
:= \mathbb{E}\!\left[g(\boldsymbol{\theta}, X_{t+1}) \mid
\mathcal{F}_t\right]
= \mathbb{E}_{X \sim \Pi_{t+1}}\!\left[g(\boldsymbol{\theta},
X)\right],
\]
and let $\boldsymbol{\theta}_{t+1}^\star\in\arg\min_{\boldsymbol{\theta}\in \Theta} G_{t+1}(\boldsymbol{\theta})$ denote a (measurable) minimizer contained within the interior of $\Theta$ where $\Theta \subset \mathbb{R}^{d}$ is compact and convex. Assume
the following hold for all $t \geq 0$:
\begin{enumerate}
    \item \textbf{(Predictable minimizer)}
    $\boldsymbol{\theta}_{t+1}^\star$ is
    $\mathcal{F}_t$-measurable.
    \item \textbf{(Algorithm adaptedness)} The iterate
    $\boldsymbol{\theta}_t$ is $\mathcal{F}_t$-measurable.
    \item \textbf{(Martingale difference noise)} The noise terms
    $\boldsymbol{\xi}_{t+1}$ and $\boldsymbol{\chi}_{t+1}$ defined
    above are $\mathcal{F}_{t+1}$-measurable and satisfy
    $\mathbb{E}[\boldsymbol{\xi}_{t+1} \mid \mathcal{F}_t] =
    \mathbf{0}$ and $\mathbb{E}[\boldsymbol{\chi}_{t+1} \mid
    \mathcal{F}_t] = \mathbf{0}$ a.s.\ for all $\boldsymbol{\theta}
    \in \mathbb{R}^d$.
\end{enumerate}
\end{assumption}

The Adam update \cite{DBLP:journals/corr/KingmaB14} with step size
\(\alpha>0\), parameters \(\beta_1,\beta_2\in(0,1)\), and \(\epsilon>0\),
initialized with \(\boldsymbol m_0=\boldsymbol v_0=\boldsymbol 0\), uses the
uncorrected moment recursions
\(\boldsymbol m_{t+1}:=\beta_1\boldsymbol m_t+(1-\beta_1)
\nabla_{\boldsymbol\theta}g(\boldsymbol\theta_t,X_{t+1})\) and
\(\boldsymbol v_{t+1}:=\beta_2\boldsymbol v_t+(1-\beta_2)
(\nabla_{\boldsymbol\theta}g(\boldsymbol\theta_t,X_{t+1}))^{\odot 2}\).
The bias-corrected Adam update can then be written as
\begin{equation}
    \tag{Adam}
    \label{eq:adam-update-restate}
    \begin{split}
        \widehat{\boldsymbol{m}}_{t+1}
        &:= \frac{\boldsymbol m_{t+1}}{1-\beta_1^{t+1}}, \\[4pt]
        \widehat{\boldsymbol{v}}_{t+1}
        &:= \frac{\boldsymbol v_{t+1}}{1-\beta_2^{t+1}}, \\[4pt]
        \boldsymbol{P}_{t+1}
        &:= \mathrm{Diag}\!\left(\left(
        \sqrt{\widehat{\boldsymbol{v}}_{t+1}} + \epsilon\right)^{-1}\right), \\[4pt]
        \boldsymbol{\theta}_{t+1}
        &:=
        \mathcal P_{\Theta}^{\boldsymbol A_{t+1}}\!\left(
        \boldsymbol{\theta}_t
        -\alpha\boldsymbol P_{t+1}\widehat{\boldsymbol m}_{t+1}
        \right).
    \end{split}
\end{equation} 
where $\nabla_{\boldsymbol{\theta}} g(\boldsymbol{\theta}_t, X_{t+1}) = \bar{\boldsymbol{g}}_{t+1}(\boldsymbol{\theta}_t) + \boldsymbol{\xi}_{t+1}$,  $(\nabla_{\boldsymbol{\theta}} g(\boldsymbol{\theta}_t, X_{t+1}))^{\odot 2} 
= \boldsymbol{s}_{t+1}(\boldsymbol{\theta}_t) + \boldsymbol{\chi}_{t+1}$, and $\mathcal P_\Theta^{\boldsymbol A}(\boldsymbol z)
    :=
    \operatorname*{arg\,min}_{\boldsymbol\theta\in\Theta}
    \frac12
    \|\boldsymbol\theta-\boldsymbol z\|_{\boldsymbol A}^2.$

The following three assumptions hold throughout all analyses:

\vspace{0.5em}

\begin{assumption}[Uniform $L$-Lipschitz continuity]
\label{assumption:lipschitz-restate}
There exists $L > 0$ such that for all $t \geq 0$ and all
$\boldsymbol{\theta}, \boldsymbol{\theta}' \in \Theta$,
\[
\|\bar{\boldsymbol{g}}_{t+1}(\boldsymbol{\theta}) -
\bar{\boldsymbol{g}}_{t+1}(\boldsymbol{\theta}')\|
\;\leq\;
L \|\boldsymbol{\theta} - \boldsymbol{\theta}'\|.
\]
\end{assumption}

\vspace{0.5em}

\begin{assumption}[Conditional sub-Gaussian gradient noise along iterates]
\label{assump:cond-subgauss-noise-restate}
There exists a constant $\sigma>0$ such that for all $t\ge 0$, $\big\|\boldsymbol{\xi}_{t+1}(\boldsymbol{\theta}_t)\big\|_{\Psi_2 \mid \mathcal F_t}\le \sigma \;\; \text{a.s.}$
\end{assumption}

\vspace{0.5em}

\paragraph{Bias-correction weights and constants.}
For \(i\in\{1,2\}\) and \(0\le k\le t-1\), define the bias-correction
weights
\[
    w_{i,t,k}
    :=
    \frac{(1-\beta_i)\beta_i^{t-1-k}}
    {1-\beta_i^t},
\]
which satisfy \(\sum_{k=0}^{t-1}w_{i,t,k}=1\). The following quantities appear throughout the analysis and are collected here for reference:
\begin{equation}
\label{eq:adam-constants}
\begin{aligned}
    \kappa_{1,t}
    &:=
    \frac{(1-\beta_1)(1+\beta_1^t)}
    {(1+\beta_1)(1-\beta_1^t)},
    \;\;
    \vartheta_{2,t}
    :=
    \beta_2^t+(1-\beta_2),
    \;\;
    q_+
    :=
    \frac{1}{\epsilon}.
\end{aligned}
\end{equation}
Here $\kappa_{1,t}=\sum_{k=0}^{t-1}w_{1,t,k}^2$ governs the variance
of the bias-corrected first-moment average,
$\vartheta_{2,t}$ captures the transient bias and stochastic fluctuation
of the bias-corrected second moment, and $q_+$ is the deterministic upper
spectral bound on the Adam preconditioners.

\section{Examples and interpretation of adaptive strong monotonicity}
\label{app:C-examples}

We provide several examples to illustrate the scope of
adaptive strong monotonicity. In
particular, adaptive strong monotonicity is not restricted to diagonal quadratics or objectives whose curvature is aligned with the adaptive
preconditioner. It can hold for dense, nonseparable objectives such as
generalized linear models and graph-coupled ranking problems, as well as for
globally nonconvex objectives when the projected iterates remain in a locally
contracting region. We also provide an explicit example showing that the
assumption is nonvacuous: ordinary strong convexity alone does not guarantee adaptive strong monotonicity.

Throughout this section, fix \(t\) and suppress the time index when there is
no ambiguity. When
\(\bar{\boldsymbol g}_{t+1}=\nabla F_{t+1}\),
\(\Theta\) is convex, and \(F_{t+1}\) is twice continuously
differentiable, a direct sufficient condition for adaptive strong
monotonicity is
\[
    \boldsymbol S_{t+1}(\boldsymbol\theta)
    :=
    \frac{
        \widetilde{\boldsymbol P}_{t+1}
        \nabla^2F_{t+1}(\boldsymbol\theta)
        +
        \nabla^2F_{t+1}(\boldsymbol\theta)
        \widetilde{\boldsymbol P}_{t+1}
    }{2}
    \succeq
    \mu\boldsymbol I
\]
uniformly over \(t\) and \(\boldsymbol\theta\in\Theta\). Indeed, if
\(\boldsymbol u:=\boldsymbol\theta-\boldsymbol\theta'\), then the
fundamental theorem of calculus gives
\(
\nabla F_{t+1}(\boldsymbol\theta)
-
\nabla F_{t+1}(\boldsymbol\theta')
=
\int_0^1
\nabla^2F_{t+1}(\boldsymbol\theta'+s\boldsymbol u)
\boldsymbol u\,ds
\).
Since
\(\boldsymbol u^\top\widetilde{\boldsymbol P}_{t+1}
\nabla^2F_{t+1}\boldsymbol u
=
\boldsymbol u^\top\boldsymbol S_{t+1}\boldsymbol u\),
integration yields
\[
    \left\langle
        \boldsymbol\theta-\boldsymbol\theta',
        \widetilde{\boldsymbol P}_{t+1}
        \left(
            \nabla F_{t+1}(\boldsymbol\theta)
            -
            \nabla F_{t+1}(\boldsymbol\theta')
        \right)
    \right\rangle
    \ge
    \mu
    \norm{\boldsymbol\theta-\boldsymbol\theta'}^2.
\]
Thus each example below amounts to establishing a uniform positive lower
bound on the symmetric preconditioned Hessian.

\vspace{0.5em}

\begin{example}[Parameter-separable objectives]
\label{ex:adaptive-sm-separable}
Consider
\(F(\boldsymbol\theta)=\sum_{j=1}^d f_j(\theta_j)\), where
\(f_j''(\theta_j)\ge\mu_0>0\), together with the diagonal adaptive
preconditioner
\(\widetilde{\boldsymbol P}
=\operatorname{Diag}(p_1,\ldots,p_d)\), with \(p_j>0\). Then, for any
\(\boldsymbol\theta,\boldsymbol\theta'\in\Theta\),
\[
    \begin{aligned}
    \left\langle
        \boldsymbol\theta-\boldsymbol\theta',
        \widetilde{\boldsymbol P}
        \left(
            \nabla F(\boldsymbol\theta)
            -
            \nabla F(\boldsymbol\theta')
        \right)
    \right\rangle
    &=
    \sum_{j=1}^d
    p_j(\theta_j-\theta_j')
    \left(
        f_j'(\theta_j)-f_j'(\theta_j')
    \right)
    \\
    &\ge
    \mu_0
    \left(
        \min_{j\in[d]}p_j
    \right)
    \norm{\boldsymbol\theta-\boldsymbol\theta'}^2.
    \end{aligned}
\]
Thus, at time \(t\), the preconditioned map is strongly monotone with
modulus
\(\mu_t=\mu_0\min_j p_{t+1,j}\). In particular, if
\(\inf_t\min_j p_{t+1,j}>0\), then
Assumption~\ref{assumption:adaptive-strong-monotonicity-restated} holds
uniformly with
\(\mu=\mu_0\inf_t\min_j p_{t+1,j}\).

Scalar and diagonal strongly convex quadratics are immediate special cases,
but the argument is not specific to quadratic losses. It also applies to
coordinatewise M-estimation, separable stochastic regression, and other
coordinate-separable stochastic programs. The same argument extends directly
to block-separable objectives with the corresponding block-diagonal
preconditioner.
\end{example}

\vspace{0.5em}

\begin{example}[Dense generalized linear models]
\label{ex:adaptive-sm-glm}
Consider the conditional regularized generalized linear model
\[
    F_{t+1}(\boldsymbol\theta)
    =
    \mathbb E_t\left[
        b(\boldsymbol X_{t+1}^{\top}\boldsymbol\theta)
        -
        Y_{t+1}\boldsymbol X_{t+1}^{\top}\boldsymbol\theta
    \right]
    +
    \frac{\lambda}{2}\norm{\boldsymbol\theta}^2.
\]
Its Hessian is
\(
\nabla^2F_{t+1}(\boldsymbol\theta)
=
\mathbb E_t[
b''(\boldsymbol X_{t+1}^{\top}\boldsymbol\theta)
\boldsymbol X_{t+1}\boldsymbol X_{t+1}^{\top}]
+\lambda\boldsymbol I
\).
Suppose, uniformly over \(\boldsymbol\theta\in\Theta\), that
\(0\le\underline b\le
b''(\boldsymbol X_{t+1}^{\top}\boldsymbol\theta)\le\overline b\) and
\[
    \nu\boldsymbol I
    \preceq
    \mathbb E_t[
        \boldsymbol X_{t+1}\boldsymbol X_{t+1}^{\top}
    ]
    \preceq
    M_X\boldsymbol I.
\]
Then
\(m\boldsymbol I
\preceq\nabla^2F_{t+1}(\boldsymbol\theta)
\preceq M\boldsymbol I\),
where \(m:=\lambda+\underline b\,\nu\) and
\(M:=\lambda+\overline b\,M_X\).

The conditional covariance may be fully dense and need not share
eigenvectors with \(\widetilde{\boldsymbol P}_{t+1}\). Let
\[
    \underline p_t
    :=
    \lambda_{\min}(
        \widetilde{\boldsymbol P}_{t+1}
    ),
    \;\;
    \overline p_t
    :=
    \lambda_{\max}(
        \widetilde{\boldsymbol P}_{t+1}
    ),
    \;\;
    \kappa_{P,t}
    :=
    \frac{\overline p_t}{\underline p_t}.
\]
Write
\(\widetilde{\boldsymbol P}_{t+1}
=p_{0,t}\boldsymbol I+\boldsymbol E_t\), where
\(p_{0,t}:=(\underline p_t+\overline p_t)/2\). Then
\(\norm{\boldsymbol E_t}
\le(\overline p_t-\underline p_t)/2\). If
\(\boldsymbol H:=\nabla^2F_{t+1}(\boldsymbol\theta)\), then for every unit
vector \(\boldsymbol v\),
\begin{align*}
    \boldsymbol v^\top
    \frac{
        \widetilde{\boldsymbol P}_{t+1}\boldsymbol H
        +
        \boldsymbol H\widetilde{\boldsymbol P}_{t+1}
    }{2}
    \boldsymbol v
    &=
    p_{0,t}
    \boldsymbol v^\top\boldsymbol H\boldsymbol v
    +
    \boldsymbol v^\top
    \boldsymbol E_t\boldsymbol H\boldsymbol v
    \\
    &\ge
    p_{0,t}m
    -
    \norm{\boldsymbol E_t}M
    \\
    &=
    \underline p_t
    \left[
        m
        -
        \frac{\kappa_{P,t}-1}{2}(M-m)
    \right].
\end{align*}
Therefore, the preconditioned mean-gradient map is strongly monotone whenever
\begin{equation}
    m
    >
    \frac{\kappa_{P,t}-1}{2}(M-m),
    \label{eq:app-adaptive-sm-glm-condition}
\end{equation}
with local modulus
\[
    \mu_t
    =
    \underline p_t
    \left[
        m
        -
        \frac{\kappa_{P,t}-1}{2}(M-m)
    \right].
\]
A uniform lower bound \(\inf_t\mu_t>0\) verifies
Assumption~\ref{assumption:adaptive-strong-monotonicity-restated}.

Substituting the definitions of \(m\) and \(M\),
\cref{eq:app-adaptive-sm-glm-condition} is equivalent to
\[
    \lambda+\underline b\,\nu
    >
    \frac{\kappa_{P,t}-1}{2}
    \left(
        \overline b\,M_X-\underline b\,\nu
    \right).
\]
For linear regression, \(b''\equiv1\), so the condition becomes
\(\lambda+\nu>
(\kappa_{P,t}-1)(M_X-\nu)/2\). Thus the feature covariance may be fully
dense and arbitrarily rotated relative to the adaptive coordinates.

For logistic regression,
\(b''(z)=e^z/(1+e^z)^2\le1/4\). If the projected parameter region and
covariates satisfy \(|\boldsymbol X^\top\boldsymbol\theta|\le B\), then
\(b''(\boldsymbol X^\top\boldsymbol\theta)
\ge e^B/(1+e^B)^2=:\underline b_B>0\).
Consequently, regularized logistic regression, and unregularized logistic
regression under sufficiently well-conditioned covariance and adaptive
scaling, also satisfy the condition. Linear and logistic regression are
among the examples used in our experiments.

This example shows directly that adaptive strong monotonicity does not
require the Hessian and adaptive preconditioner to commute; it only requires
their anisotropies to be sufficiently compatible.
\end{example}

\vspace{0.5em}

\begin{example}[Pairwise logistic regression and graph-ranking models]
\label{ex:adaptive-sm-ranking}
Consider
\[
    F(\boldsymbol\theta)
    =
    \frac{\lambda}{2}\norm{\boldsymbol\theta}^2
    +
    \sum_{(i,j)\in E}
    a_{ij}
    \log\left(
        1+\exp\left(
            -y_{ij}(\theta_i-\theta_j)
        \right)
    \right),
\]
where \(a_{ij}\ge0\), \(y_{ij}\in\{-1,1\}\), and
\(\widetilde{\boldsymbol P}
=\operatorname{Diag}(p_1,\ldots,p_d)\). This includes regularized
Bradley--Terry estimation and pairwise graph-ranking models.

For an edge \((i,j)\), let
\(h_{ij}:=\ell_{ij}''(\theta_i-\theta_j)\), where
\(\ell_{ij}(z)=a_{ij}\log(1+\exp(-y_{ij}z))\).
Since \(0\le h_{ij}\le a_{ij}/4\), the Hessian has weighted
graph-Laplacian structure,
\[
    H_{ii}
    =
    \lambda+\sum_{j:(i,j)\in E}h_{ij},
    \;\;
    H_{ij}=-h_{ij}
    \;\text{for }(i,j)\in E.
\]
For
\(\boldsymbol S
=(\widetilde{\boldsymbol P}\boldsymbol H+
\boldsymbol H\widetilde{\boldsymbol P})/2\),
we have
\(S_{ii}=p_i(\lambda+\sum_jh_{ij})\) and
\(S_{ij}=-(p_i+p_j)h_{ij}/2\). Gershgorin's theorem therefore gives
\begin{align*}
    \lambda_{\min}(\boldsymbol S)
    &\ge
    \min_i
    \left\{
        \lambda p_i
        +
        \frac12
        \sum_{j:(i,j)\in E}
        (p_i-p_j)h_{ij}
    \right\}
    \\
    &\ge
    \min_i
    \left\{
        \lambda p_i
        -
        \frac18
        \sum_{j:(i,j)\in E}
        a_{ij}|p_i-p_j|
    \right\}.
\end{align*}
Thus adaptive strong monotonicity with modulus \(\mu>0\) holds whenever
\begin{equation}
    \lambda p_i
    -
    \frac18
    \sum_{j:(i,j)\in E}
    a_{ij}|p_i-p_j|
    \ge
    \mu
    \;\;
    \text{for every }i\in[d].
    \label{eq:app-adaptive-sm-ranking-condition}
\end{equation}

This condition has a direct interpretation. The edge weights \(a_{ij}\)
measure the strength of the statistical coupling between coordinates, while
\(|p_i-p_j|\) measures the discrepancy between their adaptive scales.
Adaptive scaling is therefore benign when strongly coupled coordinates
receive similar preconditioning. If \(p_i=p_j\) along every edge, the
graph-induced penalty vanishes regardless of graph density. More generally,
dense coupling is allowed provided the adaptive scales vary sufficiently
slowly across strongly interacting coordinates.
\end{example}

\vspace{0.5em}

\begin{example}[A globally nonconvex objective]
\label{ex:adaptive-sm-nonconvex}
Adaptive strong monotonicity only needs to hold on the projected tracking
region \(\Theta\). Thus the objective may be globally nonconvex while the
preconditioned gradient remains contracting on the region visited by the
algorithm.

Consider
\(F(x,y)=x^4/4-x^2/2+y^2/2+cxy\) and
\(\widetilde{\boldsymbol P}=\operatorname{Diag}(p_1,p_2)\). The objective
is globally nonconvex because
\[
    \nabla^2F(0,0)
    =
    \begin{pmatrix}
        -1 & c\\
        c & 1
    \end{pmatrix}
\]
is indefinite. Restrict instead to
\(\Theta=[a,b]\times[-R,R]\), where \(a>1/\sqrt{3}\), and let
\(m:=3a^2-1>0\). On this region,
\[
    \nabla^2F(x,y)
    =
    \begin{pmatrix}
        3x^2-1 & c\\
        c & 1
    \end{pmatrix},
    \;\;
    3x^2-1\ge m.
\]
Hence
\[
    \frac{
        \widetilde{\boldsymbol P}\nabla^2F
        +
        \nabla^2F\widetilde{\boldsymbol P}
    }{2}
    \succeq
    \begin{pmatrix}
        p_1m & (p_1+p_2)c/2\\
        (p_1+p_2)c/2 & p_2
    \end{pmatrix}.
\]
The matrix on the right is positive definite precisely when
\begin{equation}
    4p_1p_2m>(p_1+p_2)^2c^2.
    \label{eq:app-adaptive-sm-nonconvex-condition}
\end{equation}
Equivalently,
\(|c|<2\sqrt{p_1p_2m}/(p_1+p_2)\). Under this condition, its smallest
eigenvalue is
\[
    \gamma
    =
    \frac12
    \left[
        p_1m+p_2
        -
        \sqrt{
            (p_1m-p_2)^2
            +(p_1+p_2)^2c^2
        }
    \right]
    >0,
\]
so adaptive strong monotonicity holds on \(\Theta\) with modulus
\(\mu=\gamma\).

This example is both nonseparable and globally nonconvex. Projection
restricts the iterates to a region of positive local curvature, while
\cref{eq:app-adaptive-sm-nonconvex-condition} controls the strength of the
cross-coordinate interaction relative to that curvature and the adaptive
scales.
\end{example}

\vspace{0.5em}

\begin{example}[Strong convexity does not imply adaptive strong monotonicity]
\label{ex:adaptive-sm-failure}
Finally, adaptive strong monotonicity imposes a compatibility
condition between objective curvature and adaptive scaling. Consider the
strongly convex quadratic
\(F(\boldsymbol\theta)
=\boldsymbol\theta^\top\boldsymbol H\boldsymbol\theta/2\), where
\[
    \boldsymbol H
    =
    \begin{pmatrix}
        a & c\\
        c & b
    \end{pmatrix},
    \;\;
    \widetilde{\boldsymbol P}
    =
    \begin{pmatrix}
        p_1 & 0\\
        0 & p_2
    \end{pmatrix}.
\]
Assume \(a,b,p_1,p_2>0\) and \(ab>c^2\), so
\(\boldsymbol H\succ0\). The symmetric preconditioned Hessian is
\[
    \boldsymbol S
    =
    \frac{
        \widetilde{\boldsymbol P}\boldsymbol H
        +
        \boldsymbol H\widetilde{\boldsymbol P}
    }{2}
    =
    \begin{pmatrix}
        p_1a & (p_1+p_2)c/2\\
        (p_1+p_2)c/2 & p_2b
    \end{pmatrix}.
\]
Although \(\boldsymbol H\succ0\), \(\boldsymbol S\succ0\) only when
\begin{equation}
    4p_1p_2ab>(p_1+p_2)^2c^2.
    \label{eq:app-adaptive-sm-failure-condition}
\end{equation}
Define
\(\rho_H:=|c|/\sqrt{ab}\) and
\(\rho_P:=2\sqrt{p_1p_2}/(p_1+p_2)\). Then
\cref{eq:app-adaptive-sm-failure-condition} is exactly
\(\rho_H<\rho_P\): stronger cross-coordinate coupling requires more balanced
adaptive scaling.

For a concrete failure case, take
\[
    \boldsymbol H
    =
    \begin{pmatrix}
        1&r\\
        r&1
    \end{pmatrix},
    \;\;
    \widetilde{\boldsymbol P}
    =
    \begin{pmatrix}
        9&0\\
        0&1
    \end{pmatrix}.
\]
The Hessian eigenvalues are \(1-r\) and \(1+r\), so the objective is
strongly convex whenever \(|r|<1\). However,
\[
    \boldsymbol S
    =
    \begin{pmatrix}
        9&5r\\
        5r&1
    \end{pmatrix},
    \;\;
    \det(\boldsymbol S)=9-25r^2.
\]
Thus adaptive strong monotonicity requires \(|r|<3/5\). Taking
\(r=4/5\), for example, gives a strongly convex objective with Hessian
eigenvalues \(1/5\) and \(9/5\), while
\(\det(\boldsymbol S)=-7<0\).
\end{example}

Taken together, these examples show that
Assumption~\ref{assumption:adaptive-strong-monotonicity-restated} is substantive but not restricted to diagonal or quadratic objectives. It applies to separable problems with uniformly positive adaptive scales, dense generalized linear models, graph-coupled estimation problems, and locally contracting regions of globally nonconvex objectives. At the same time, Example \ref{ex:adaptive-sm-failure} shows that sufficiently strong cross-coordinate curvature can destroy contraction when it is incompatible
with the adaptive scaling.

\section{Proofs for tracking under adaptive strong monotonicity}
\label{app:C}
For this section, we will assume the preconditioned conditional mean gradient map $\boldsymbol{\theta} \mapsto \widetilde{\boldsymbol P}_{t+1}\bar{\boldsymbol{g}}_{t+1}(\boldsymbol{\theta})$ is $\mu$-strongly monotone.

\vspace{0.5em}

\begin{assumption}[Adaptive strong monotonicity]
\label{assumption:adaptive-strong-monotonicity-restated}
There exists \(0<\mu<\infty\) such that,
almost surely, for all \(t\ge 0\) and all
\(\boldsymbol{\theta},\boldsymbol{\theta}'\in \Theta\),
$
    \left\langle \boldsymbol{\theta}-\boldsymbol{\theta}',\,
    \widetilde{\boldsymbol P}_{t+1}
    \left(
        \bar{\boldsymbol g}_{t+1}(\boldsymbol{\theta})
        -
        \bar{\boldsymbol g}_{t+1}(\boldsymbol{\theta}')
    \right)
    \right\rangle
    \ge
    \mu
    \|\boldsymbol{\theta}-\boldsymbol{\theta}'\|^2 .
$
\end{assumption}

\subsection{Proofs for \cref{thm:high-prob-adam}}
\label{app:C1}

We first introduce the predictable preconditioner proxy used throughout the Adam analysis. The realized Adam preconditioner \(\boldsymbol P_{t+1}\) is computed from the fresh sample \(X_{t+1}\), and is therefore \(\mathcal F_{t+1}\)-measurable rather than \(\mathcal F_t\)-measurable. Since our tracking argument imposes contraction on a predictable preconditioned mean-gradient map, we separate the predictable geometry from the random second-moment fluctuation. 

To do so, we construct a proxy by replacing the uncorrected second
moment with its conditional expectation. Define
\(\widetilde{\boldsymbol v}_{t+1}
:=\mathbb E[\boldsymbol v_{t+1}\mid\mathcal F_t]\) and
\(\widetilde{\boldsymbol P}_{t+1}
:=
[\operatorname{Diag}(\sqrt{\widetilde{\boldsymbol v}_{t+1}})
+\epsilon\boldsymbol I]^{-1}\).
Since
\(\boldsymbol v_{t+1}
=
\beta_2\boldsymbol v_t
+
(1-\beta_2)
(\nabla_{\boldsymbol\theta}
g(\boldsymbol\theta_t,X_{t+1}))^{\odot2}\),
taking conditional expectation given \(\mathcal F_t\) yields
\(\widetilde{\boldsymbol v}_{t+1}
=
\beta_2\boldsymbol v_t
+
(1-\beta_2)\boldsymbol s_{t+1}(\boldsymbol\theta_t)\).
Thus both \(\widetilde{\boldsymbol v}_{t+1}\) and
\(\widetilde{\boldsymbol P}_{t+1}\) are \(\mathcal F_t\)-measurable.

We next rewrite the Adam update around this predictable proxy. Define
the first-moment tracking error by
\(\boldsymbol r_{t+1}
:=
\widehat{\boldsymbol m}_{t+1}
-
\bar{\boldsymbol g}_{t+1}(\boldsymbol\theta_t)\)
and the preconditioner perturbation by
\(\boldsymbol\eta_{t+1}
:=
(\boldsymbol P_{t+1}
-
\widetilde{\boldsymbol P}_{t+1})
\widehat{\boldsymbol m}_{t+1}\).
Then
\[
    \boldsymbol P_{t+1}\widehat{\boldsymbol m}_{t+1}
    =
    \widetilde{\boldsymbol P}_{t+1}
    \bar{\boldsymbol g}_{t+1}(\boldsymbol\theta_t)
    +
    \widetilde{\boldsymbol P}_{t+1}\boldsymbol r_{t+1}
    +
    \boldsymbol\eta_{t+1}.
\]
Consequently, for the pre-projection point
\(\boldsymbol z_{t+1}
:=
\boldsymbol\theta_t
-
\alpha\boldsymbol P_{t+1}\widehat{\boldsymbol m}_{t+1}\),
we obtain the predictable decomposition
\begin{equation}
    \boldsymbol z_{t+1}
    =
    \boldsymbol\theta_t
    -
    \alpha\widetilde{\boldsymbol P}_{t+1}
    \bar{\boldsymbol g}_{t+1}(\boldsymbol\theta_t)
    -
    \alpha\widetilde{\boldsymbol P}_{t+1}
    \boldsymbol r_{t+1}
    -
    \alpha\boldsymbol\eta_{t+1},
    \label{eq:app-adam-predictable-update}
\end{equation}
with
\(\boldsymbol\theta_{t+1}
=
\mathcal P_\Theta(\boldsymbol z_{t+1})\).

We will analyze the exact Adam update through this decomposition: the first term is the predictable preconditioned mean-gradient step, \(\boldsymbol r_{t+1}\) measures the error from using Adam's first-moment estimate instead of the current conditional mean gradient, and \(\boldsymbol\eta_{t+1}\) captures the price of replacing the random preconditioner by its predictable proxy.

We first establish uniform high-probability control of both the stochastic gradient and the centered squared-gradient fluctuation. These bounds will then be used to control the bias-corrected second moment and the corresponding preconditioner perturbation.

\vspace{0.5em}

\begin{lemma}[Uniform high-probability control of the stochastic gradient]
\label{lem:app-adam-gradient-hp}
Suppose Assumption~\ref{assump:cond-subgauss-noise-restate} holds and
\(\operatorname{diam}(\Theta)\le R\). Then there exists a universal constant
\(C>0\) such that, for every \(T\ge1\) and \(\delta\in(0,1)\), with
probability at least \(1-\delta\), simultaneously for every
\(t\in\{0,\ldots,T-1\}\), the following statements hold:
\begin{enumerate}
    \item The stochastic gradient satisfies
    \begin{equation}
        \left\|
            \nabla_{\boldsymbol\theta}
            g(\boldsymbol\theta_t,X_{t+1})
        \right\|
        \le
        LR
        +
        C\sigma
        \sqrt{
            d+\log\frac{4T}{\delta}
        }.
        \label{eq:app-adam-gradient-hp}
    \end{equation}

    \item The centered coordinatewise squared-gradient fluctuation satisfies
    \begin{equation}
        \left\|
            \left(
                \nabla_{\boldsymbol\theta}
                g(\boldsymbol\theta_t,X_{t+1})
            \right)^{\odot 2}
            -
            \boldsymbol s_{t+1}(\boldsymbol\theta_t)
        \right\|_\infty
        \le
        C\left(
            LR\sigma
            \sqrt{\log\frac{4dT}{\delta}}
            +
            \sigma^2
            \log\frac{4dT}{\delta}
        \right).
        \label{eq:app-adam-square-gradient-hp}
    \end{equation}

    \item The squared gradient coordinates satisfy
    \begin{equation}
        \left\|
            \left(
                \nabla_{\boldsymbol\theta}
                g(\boldsymbol\theta_t,X_{t+1})
            \right)^{\odot 2}
        \right\|_\infty
        \le
        C\left(
            L^2R^2
            +
            \sigma^2
            \log\frac{4dT}{\delta}
        \right).
        \label{eq:app-adam-coordinate-square-hp}
    \end{equation}
\end{enumerate}
\end{lemma}

\begin{proof}[Proof of \cref{lem:app-adam-gradient-hp}]
Since \(\boldsymbol\theta_{t+1}^\star\) lies in the interior of \(\Theta\), the first-order optimality condition gives \(\bar{\boldsymbol g}_{t+1}(\boldsymbol\theta_{t+1}^\star)=\boldsymbol 0\). Therefore, by the \(L\)-Lipschitz continuity of \(\bar{\boldsymbol g}_{t+1}\) and the fact that \(\operatorname{diam}(\Theta)\le R\),
\begin{align}
    \left\|
        \bar{\boldsymbol g}_{t+1}(\boldsymbol\theta_t)
    \right\|
    &=
    \left\|
        \bar{\boldsymbol g}_{t+1}(\boldsymbol\theta_t)
        -
        \bar{\boldsymbol g}_{t+1}(\boldsymbol\theta_{t+1}^\star)
    \right\|
    \le
    L
    \left\|
        \boldsymbol\theta_t
        -
        \boldsymbol\theta_{t+1}^\star
    \right\|
    \le
    LR.
    \label{eq:app-adam-mean-gradient-bound}
\end{align}

By definition, $\nabla_{\boldsymbol\theta}
    g(\boldsymbol\theta_t,X_{t+1})
    =
    \bar{\boldsymbol g}_{t+1}(\boldsymbol\theta_t)
    +
    \boldsymbol\xi_{t+1}(\boldsymbol\theta_t)$.  Under Assumption~\ref{assump:cond-subgauss-noise-restate}, for every \(\boldsymbol u\in\mathbb S^{d-1}\) $\left\|
        \left\langle
            \boldsymbol u,
            \boldsymbol\xi_{t+1}(\boldsymbol\theta_t)
        \right\rangle
    \right\|_{\Psi_2\mid\mathcal F_t}
    \le
    \sigma$ a.s. 
Thus, by the one-step conditional form of \cref{thm:martingale-subgaussian-hoeffding} (see also
\cite[Proposition~2.6.1]{Vershynin_2018}), there exists a universal constant
\(c>0\) such that, conditionally on \(\mathcal F_t\),
\[
    \mathbb P\left(
        \left.
        \left|
            \left\langle
                \boldsymbol u,
                \boldsymbol\xi_{t+1}(\boldsymbol\theta_t)
            \right\rangle
        \right|
        \ge x
        \,\right|\,
        \mathcal F_t
    \right)
    \le
    2
    \exp\left(
        -\frac{c x^2}{\sigma^2}
    \right).
\]

Let $\mathcal{N}\subseteq\mathbb{S}^{d-1}$ be a $1/2$-net of $\mathbb{S}^{d-1}$. Standard volumetric arguments then imply that $|\mathcal{N}|\le 5^d$ \cite[Corollary 4.2.11]{Vershynin_2018}. By \cite[Exercise 4.35]{Vershynin_2018}, $\|\boldsymbol x\|
    \le
    2
    \max_{\boldsymbol u\in\mathcal N}
    \left|
        \langle
            \boldsymbol u,\boldsymbol x
        \rangle
    \right|$ for every \(\boldsymbol x\in\mathbb R^d\), a union bound over \(\mathcal N\) gives
\[
    \mathbb P\left(
        \left.
        \left\|
            \boldsymbol\xi_{t+1}(\boldsymbol\theta_t)
        \right\|
        \ge
        C\sigma\sqrt{d+x}
        \,\right|\,
        \mathcal F_t
    \right)
    \le
    2e^{-x}
\]
for every \(x>0\). Taking expectations, choosing \(x=\log(4T/\delta)\), and applying a union bound over \(t=0,\ldots,T-1\) yields, with probability at least \(1-\delta/2\),
\[
    \max_{0\le t\le T-1}
    \left\|
        \boldsymbol\xi_{t+1}(\boldsymbol\theta_t)
    \right\|
    \le
    C\sigma
    \sqrt{
        d+\log\frac{4T}{\delta}
    }.
\]
Together with \cref{eq:app-adam-mean-gradient-bound}, this gives
\[
    \left\|
        \nabla_{\boldsymbol\theta}
        g(\boldsymbol\theta_t,X_{t+1})
    \right\|
    \le
    LR
    +
    C\sigma
    \sqrt{
        d+\log\frac{4T}{\delta}
    },
\]
simultaneously for all \(t\in\{0,\ldots,T-1\}\), which proves \cref{eq:app-adam-gradient-hp}. 

We next control the centered squared-gradient fluctuation. Fix \(t\in\{0,\ldots,T-1\}\) and \(i\in[d]\). Using $\nabla_{\theta_i}
    g(\boldsymbol\theta_t,X_{t+1})
    =
    \bar g_{t+1,i}(\boldsymbol\theta_t)
    +
    \xi_{t+1,i}(\boldsymbol\theta_t)$ 
and \(\mathbb E[\xi_{t+1,i}(\boldsymbol\theta_t)\mid\mathcal F_t]=0\), we have $s_{t+1,i}(\boldsymbol\theta_t)
    =
    \bar g_{t+1,i}(\boldsymbol\theta_t)^2
    +
    \mathbb E\!\left[
        \xi_{t+1,i}(\boldsymbol\theta_t)^2
        \mid
        \mathcal F_t
    \right]$.  
Consequently,
\begin{align}
    &
    \left(
        \nabla_{\theta_i}
        g(\boldsymbol\theta_t,X_{t+1})
    \right)^2
    -
    s_{t+1,i}(\boldsymbol\theta_t)
    =
    2
    \bar g_{t+1,i}(\boldsymbol\theta_t)
    \xi_{t+1,i}(\boldsymbol\theta_t)
    +
    \xi_{t+1,i}(\boldsymbol\theta_t)^2
    -
    \mathbb E\!\left[
        \xi_{t+1,i}(\boldsymbol\theta_t)^2
        \mid
        \mathcal F_t
    \right].
    \label{eq:app-adam-square-decomp}
\end{align}

Assumption~\ref{assump:cond-subgauss-noise-restate} implies $\left\|
        \xi_{t+1,i}(\boldsymbol\theta_t)
    \right\|_{\Psi_2\mid\mathcal F_t}
    \le
    \sigma$. 
Hence by \cite[Exercise 2.44]{Vershynin_2018}, we have
\[
    \left\|
        \xi_{t+1,i}(\boldsymbol\theta_t)^2
        -
        \mathbb E\!\left[
            \xi_{t+1,i}(\boldsymbol\theta_t)^2
            \mid
            \mathcal F_t
        \right]
    \right\|_{\Psi_1\mid\mathcal F_t}
    \le
    C\sigma^2.
\]
Moreover, by \cref{eq:app-adam-mean-gradient-bound}, $\left|
        \bar g_{t+1,i}(\boldsymbol\theta_t)
    \right|
    \le LR$, 
so \(2\bar g_{t+1,i}(\boldsymbol\theta_t)\xi_{t+1,i}(\boldsymbol\theta_t)\) is conditionally sub-Gaussian with scale at most \(CLR\sigma\). Therefore, conditional Bernstein's inequality (\cref{thm:martingale-subexponential-bernstein}) applied to \cref{eq:app-adam-square-decomp} gives, for every \(x\ge1\),
\begin{align}
    &\mathbb P\Bigg(
        \left.
        \left|
            \left(
                \nabla_{\theta_i}
                g(\boldsymbol\theta_t,X_{t+1})
            \right)^2
            -
            s_{t+1,i}(\boldsymbol\theta_t)
        \right|
        >
        C\left(
            LR\sigma\sqrt{x}
            +
            \sigma^2x
        \right)
        \,\right|\,
        \mathcal F_t
    \Bigg)
    \le
    2e^{-x}.
    \label{eq:app-adam-square-bernstein}
\end{align}
Taking expectations, choosing \(x=\log(4dT/\delta)\), and applying a union bound over \(t=0,\ldots,T-1\) and \(i\in[d]\) yields, with probability at least \(1-\delta/2\),
\[
    \max_{\substack{0\le t\le T-1\\i\in[d]}}
    \left|
        \left(
            \nabla_{\theta_i}
            g(\boldsymbol\theta_t,X_{t+1})
        \right)^2
        -
        s_{t+1,i}(\boldsymbol\theta_t)
    \right|
    \le
    C\left(
        LR\sigma
        \sqrt{\log\frac{4dT}{\delta}}
        +
        \sigma^2
        \log\frac{4dT}{\delta}
    \right),
\]
which proves \cref{eq:app-adam-square-gradient-hp}.

It remains to obtain a direct bound on the squared gradient coordinates. Conditional sub-Gaussianity implies
$\mathbb E\!\left[
        \xi_{t+1,i}(\boldsymbol\theta_t)^2
        \mid
        \mathcal F_t
    \right]
    \le
    C\sigma^2$, 
and hence $s_{t+1,i}(\boldsymbol\theta_t)
    \le
    L^2R^2
    +
    C\sigma^2$. 
Combining this with \cref{eq:app-adam-square-gradient-hp} gives
\begin{align*}
    \left|
        \nabla_{\theta_i}
        g(\boldsymbol\theta_t,X_{t+1})
    \right|^2
    &\le
    L^2R^2
    +
    C\sigma^2
    +
    C LR\sigma
    \sqrt{\log\frac{4dT}{\delta}}
    +
    C\sigma^2
    \log\frac{4dT}{\delta}
    \\
    &\le
    C\left(
        L^2R^2
        +
        \sigma^2
        \log\frac{4dT}{\delta}
    \right),
\end{align*}
where the last inequality follows from Young's inequality. Taking the maximum over \(i\in[d]\) proves \cref{eq:app-adam-coordinate-square-hp}. Intersecting the two preceding events completes the proof.
\end{proof}

Having established uniform concentration of the stochastic gradient and its coordinatewise square, we can now control the discrepancy between the realized Adam preconditioner and its predictable proxy. The argument separates naturally into two pieces: the transient error introduced by bias correction of the second moment and the fresh second-moment fluctuation represented by \(\boldsymbol\chi_{t+1}\).

\vspace{0.5em}

\begin{lemma}[High-probability control of the preconditioner perturbation]
\label{lem:app-adam-eta-hp}
Suppose Assumption~\ref{assump:cond-subgauss-noise-restate} holds. Then there exists a universal constant \(C>0\) such that, for every \(T\ge1\) and \(\delta\in(0,1)\), with probability at least \(1-\delta\), simultaneously for every \(t\in\{0,\ldots,T-1\}\),
\begin{equation}
    \norm{
        \boldsymbol P_{t+1}
        -
        \widetilde{\boldsymbol P}_{t+1}
    }_{\operatorname{op}}^2
    \le
    C\epsilon^{-4}
    \left(
        L^2R^2
        +
        \sigma^2
        \log\frac{4dT}{\delta}
    \right)
    \vartheta_{2,t+1}.
    \label{eq:app-adam-preconditioner-hp-simple}
\end{equation}
Moreover,
\begin{equation}
    \norm{\boldsymbol\eta_{t+1}}^2
    \le
    C\epsilon^{-4}
    \left(
        L^2R^2
        +
        \sigma^2
        \left(
            d+\log\frac{4T}{\delta}
        \right)
    \right)^2
    \vartheta_{2,t+1}.
    \label{eq:app-adam-eta-hp}
\end{equation}
\end{lemma}

\begin{proof}[Proof of \cref{lem:app-adam-eta-hp}]
Since both \(\boldsymbol P_{t+1}\) and
\(\widetilde{\boldsymbol P}_{t+1}\) are diagonal, it suffices to compare
their entries coordinatewise. For any \(a,b\ge0\),
\(\left|(\sqrt a+\epsilon)^{-1}-(\sqrt b+\epsilon)^{-1}\right|
\le \epsilon^{-2}|\sqrt a-\sqrt b|
\le \epsilon^{-2}\sqrt{|a-b|}\).
Applying this inequality coordinatewise yields
\(\norm{\boldsymbol P_{t+1}-\widetilde{\boldsymbol P}_{t+1}}_{\operatorname{op}}^2
\le
\epsilon^{-4}
\norm{\widehat{\boldsymbol v}_{t+1}-\widetilde{\boldsymbol v}_{t+1}}_\infty\).

It therefore remains to control the discrepancy between the
bias-corrected second moment \(\widehat{\boldsymbol v}_{t+1}\) and its
predictable proxy \(\widetilde{\boldsymbol v}_{t+1}\). Since
\(\widehat{\boldsymbol v}_{t+1}
=\boldsymbol v_{t+1}/(1-\beta_2^{t+1})\), we may write
\(\widehat{\boldsymbol v}_{t+1}-\widetilde{\boldsymbol v}_{t+1}
=
\frac{\beta_2^{t+1}}{1-\beta_2^{t+1}}\boldsymbol v_{t+1}
+
\boldsymbol v_{t+1}-\widetilde{\boldsymbol v}_{t+1}\).

The first term is the transient bias-correction error. Since
\(\boldsymbol v_0=\boldsymbol 0\), unrolling the second-moment recursion gives
\(\boldsymbol v_{t+1}
=
(1-\beta_2)\sum_{j=0}^{t}\beta_2^{t-j}
(\nabla_{\boldsymbol\theta}
g(\boldsymbol\theta_j,X_{j+1}))^{\odot2}\).
By \cref{lem:app-adam-gradient-hp}
(\cref{eq:app-adam-coordinate-square-hp}), with probability at least
\(1-\delta\), simultaneously for every \(j\in\{0,\ldots,T-1\}\),
\begin{align}
    \norm{\boldsymbol v_{t+1}}_\infty
    &\le
    C(1-\beta_2)
    \sum_{j=0}^{t}
    \beta_2^{t-j}
    \left(
        L^2R^2
        +
        \sigma^2\log\frac{4dT}{\delta}
    \right)
    =
    C(1-\beta_2^{t+1})
    \left(
        L^2R^2
        +
        \sigma^2\log\frac{4dT}{\delta}
    \right).
    \label{eq:app-adam-v-hp}
\end{align}

For the second term, the definition of
\(\widetilde{\boldsymbol v}_{t+1}\) gives
\(\boldsymbol v_{t+1}-\widetilde{\boldsymbol v}_{t+1}
=
(1-\beta_2)
[(\nabla_{\boldsymbol\theta}
g(\boldsymbol\theta_t,X_{t+1}))^{\odot2}
-
\boldsymbol s_{t+1}(\boldsymbol\theta_t)]
=
(1-\beta_2)\boldsymbol\chi_{t+1}\).
Hence, by \cref{lem:app-adam-gradient-hp}
(\cref{eq:app-adam-square-gradient-hp}),
\begin{equation}
    \norm{
        \boldsymbol v_{t+1}
        -
        \widetilde{\boldsymbol v}_{t+1}
    }_\infty
    \le
    C(1-\beta_2)
    \left(
        LR\sigma\sqrt{\log\frac{4dT}{\delta}}
        +
        \sigma^2\log\frac{4dT}{\delta}
    \right).
    \label{eq:app-adam-v-vtilde-hp-bound}
\end{equation}

Combining the preceding two bounds with the decomposition of
\(\widehat{\boldsymbol v}_{t+1}-\widetilde{\boldsymbol v}_{t+1}\) yields
\begin{align}
    \norm{
        \widehat{\boldsymbol v}_{t+1}
        -
        \widetilde{\boldsymbol v}_{t+1}
    }_\infty
    \le
    C\Bigg[
        &\beta_2^{t+1}
        \left(
            L^2R^2
            +
            \sigma^2\log\frac{4dT}{\delta}
        \right)
        +
        (1-\beta_2)
        \left(
            LR\sigma\sqrt{\log\frac{4dT}{\delta}}
            +
            \sigma^2\log\frac{4dT}{\delta}
        \right)
    \Bigg].
    \label{eq:app-adam-vhat-vtilde-hp-bound}
\end{align}
Young's inequality gives
\(LR\sigma\sqrt{\log(4dT/\delta)}
\le
\frac12L^2R^2
+
\frac12\sigma^2\log(4dT/\delta)\).
Recalling that
\(\vartheta_{2,t+1}:=\beta_2^{t+1}+(1-\beta_2)\), we therefore obtain
\[
    \norm{
        \boldsymbol P_{t+1}
        -
        \widetilde{\boldsymbol P}_{t+1}
    }_{\operatorname{op}}^2
    \le
    C\epsilon^{-4}
    \left(
        L^2R^2
        +
        \sigma^2\log\frac{4dT}{\delta}
    \right)
    \vartheta_{2,t+1},
\]
which proves \cref{eq:app-adam-preconditioner-hp-simple}.

It remains to control \(\widehat{\boldsymbol m}_{t+1}\). Unrolling the
bias-corrected first moment gives
\(\widehat{\boldsymbol m}_{t+1}
=
\sum_{j=0}^{t}
w_{1,t+1,j}
\nabla_{\boldsymbol\theta}
g(\boldsymbol\theta_j,X_{j+1})\),
where \(w_{1,t+1,j}\ge0\) and
\(\sum_{j=0}^{t}w_{1,t+1,j}=1\).
By convexity of the squared norm,
\(\norm{\widehat{\boldsymbol m}_{t+1}}^2
\le
\sum_{j=0}^{t}
w_{1,t+1,j}
\norm{\nabla_{\boldsymbol\theta}
g(\boldsymbol\theta_j,X_{j+1})}^2\).
Therefore, by \cref{lem:app-adam-gradient-hp}
(\cref{eq:app-adam-gradient-hp}),
\(\norm{\widehat{\boldsymbol m}_{t+1}}^2
\le
C\left(
L^2R^2
+
\sigma^2(d+\log(4T/\delta))
\right)\).

Finally, by definition,
\(\boldsymbol\eta_{t+1}
=
(\boldsymbol P_{t+1}-\widetilde{\boldsymbol P}_{t+1})
\widehat{\boldsymbol m}_{t+1}\), and hence
\(\norm{\boldsymbol\eta_{t+1}}^2
\le
\norm{\boldsymbol P_{t+1}-\widetilde{\boldsymbol P}_{t+1}}_{\operatorname{op}}^2
\norm{\widehat{\boldsymbol m}_{t+1}}^2\).
Combining the preceding two estimates gives
\[
    \norm{\boldsymbol\eta_{t+1}}^2
    \le
    C\epsilon^{-4}
    \left(
        L^2R^2
        +
        \sigma^2
        \left(
            d+\log\frac{4T}{\delta}
        \right)
    \right)^2
    \vartheta_{2,t+1},
\]
which proves \cref{eq:app-adam-eta-hp}.
\end{proof}

We now use the predictable decomposition in \cref{eq:app-adam-predictable-update} to derive the one-step tracking recursion. This proof follows similarly to \cite{sahu2026provable, NEURIPS2021_62e7f2e0}.

\vspace{0.5em}

\begin{lemma}[One-step recursive relation]
\label{lem:app-adam-recursive-relation}
Assume \cref{assumption:adaptive-strong-monotonicity-restated} and suppose $\alpha
    \le
    \min\left\{
        \frac{\mu\epsilon^{2}}{4L^2},
        \frac{1}{\mu}
    \right\}$. 
Then, for every \(t\ge 0\),
\begin{align}
    \norm{\boldsymbol{\theta}_{t+1}-\boldsymbol{\theta}_{t+1}^\star}^2
    \le\;&
    \left(1-\frac12\alpha \mu\right)
    \norm{\boldsymbol{\theta}_t-\boldsymbol{\theta}_t^\star}^2
    +\frac{5}{\alpha \mu}
    \norm{\boldsymbol{\theta}_t^\star-\boldsymbol{\theta}_{t+1}^\star}^2
    +
    \frac{10\alpha \epsilon^{-2}}{\mu}
    \norm{\boldsymbol r_{t+1}}^2
    +
    \frac{10\alpha}{\mu}
    \norm{\boldsymbol\eta_{t+1}}^2 .
    \label{eq:app-adam-descent-recursion}
\end{align}
\end{lemma}

\begin{proof}[Proof of \cref{lem:app-adam-recursive-relation}]
Define the pre-projection point
\(\boldsymbol z_{t+1}
:=\boldsymbol{\theta}_t-\alpha\boldsymbol P_{t+1}
\widehat{\boldsymbol m}_{t+1}\).
By the predictable decomposition in
\cref{eq:app-adam-predictable-update},
\(\boldsymbol z_{t+1}
=
\boldsymbol{\theta}_t
-\alpha\widetilde{\boldsymbol P}_{t+1}
\bar{\boldsymbol g}_{t+1}(\boldsymbol{\theta}_t)
-\alpha\widetilde{\boldsymbol P}_{t+1}\boldsymbol r_{t+1}
-\alpha\boldsymbol\eta_{t+1}\).
Since
\(\boldsymbol{\theta}_{t+1}
=\mathcal{P}_\Theta(\boldsymbol z_{t+1})\) and
\(\boldsymbol{\theta}_{t+1}^\star\in\Theta\), nonexpansiveness of the
Euclidean projection (\cref{lem:metric-projection-facts}) gives
\(\norm{\boldsymbol{\theta}_{t+1}
-\boldsymbol{\theta}_{t+1}^\star}^2
\le
\norm{\boldsymbol z_{t+1}
-\boldsymbol{\theta}_{t+1}^\star}^2\).

Recall
\(\boldsymbol d_t
:=\boldsymbol{\theta}_t-\boldsymbol{\theta}_{t+1}^\star\), and define
\(\boldsymbol u_t
:=
\boldsymbol d_t
-\alpha\widetilde{\boldsymbol P}_{t+1}
\bar{\boldsymbol g}_{t+1}(\boldsymbol{\theta}_t)\).
Then
\(\boldsymbol z_{t+1}-\boldsymbol{\theta}_{t+1}^\star
=
\boldsymbol u_t
-\alpha\widetilde{\boldsymbol P}_{t+1}\boldsymbol r_{t+1}
-\alpha\boldsymbol\eta_{t+1}\). Since
\(\bar{\boldsymbol g}_{t+1}(\boldsymbol{\theta}_{t+1}^\star)
=\boldsymbol 0\), adaptive strong monotonicity
(\cref{assumption:adaptive-strong-monotonicity-restated}) implies
$\left\langle
        \boldsymbol d_t,\,
        \widetilde{\boldsymbol P}_{t+1}
        \bar{\boldsymbol g}_{t+1}(\boldsymbol{\theta}_t)
    \right\rangle
    \ge
    \mu\norm{\boldsymbol d_t}^2.
$

Moreover, by \(L\)-Lipschitz continuity of
\(\bar{\boldsymbol g}_{t+1}\),
\(\bar{\boldsymbol g}_{t+1}(\boldsymbol{\theta}_{t+1}^\star)=0\), and
\(\widetilde{\boldsymbol P}_{t+1}\preceq q_+\boldsymbol I\), we have
\(\norm{\widetilde{\boldsymbol P}_{t+1}
\bar{\boldsymbol g}_{t+1}(\boldsymbol{\theta}_t)}^2
\le
q_+^2L^2\norm{\boldsymbol d_t}^2\).
Consequently,
\begin{align}
    \norm{\boldsymbol u_t}^2
    &=
    \norm{\boldsymbol d_t}^2
    -2\alpha
    \left\langle
        \boldsymbol d_t,\,
        \widetilde{\boldsymbol P}_{t+1}
        \bar{\boldsymbol g}_{t+1}(\boldsymbol{\theta}_t)
    \right\rangle
    +
    \alpha^2
    \norm{
        \widetilde{\boldsymbol P}_{t+1}
        \bar{\boldsymbol g}_{t+1}(\boldsymbol{\theta}_t)
    }^2
    \notag\\
    &\le
    \left(
        1-2\alpha\mu+\alpha^2q_+^2L^2
    \right)
    \norm{\boldsymbol d_t}^2
    \le
    (1-\alpha\mu)\norm{\boldsymbol d_t}^2,
    \label{eq:clean-contract-theta}
\end{align}
where the last inequality follows from
\(\alpha\le\mu/(4q_+^2L^2)\).

We now add back the two Adam-specific error terms. Let
\(\tau:=\alpha\mu/4\). By Young's inequality and the projection
inequality above,
\begin{align}
    \norm{
        \boldsymbol{\theta}_{t+1}
        -
        \boldsymbol{\theta}_{t+1}^\star
    }^2
    &\le
    (1+\tau)\norm{\boldsymbol u_t}^2
    +(1+\tau^{-1})\alpha^2
    \norm{
        \widetilde{\boldsymbol P}_{t+1}\boldsymbol r_{t+1}
        +
        \boldsymbol\eta_{t+1}
    }^2
    \notag\\
    &\le
    (1+\tau)\norm{\boldsymbol u_t}^2
    +
    2(1+\tau^{-1})\alpha^2
    \left(
        q_+^2\norm{\boldsymbol r_{t+1}}^2
        +
        \norm{\boldsymbol\eta_{t+1}}^2
    \right).
    \label{eq:young-error-theta}
\end{align}
Since \(\alpha\mu\le1\), we have \(\tau\le1/4\), and hence
\((1+\tau)(1-\alpha\mu)
\le1-\frac34\alpha\mu\).
Also,
\(2(1+\tau^{-1})\alpha^2
=
2\alpha^2+8\alpha/\mu
\le10\alpha/\mu\).
Combining these bounds with
\cref{eq:clean-contract-theta,eq:young-error-theta} yields
\begin{equation}
    \norm{
        \boldsymbol{\theta}_{t+1}
        -
        \boldsymbol{\theta}_{t+1}^\star
    }^2
    \le
    \left(1-\frac34\alpha\mu\right)
    \norm{\boldsymbol d_t}^2
    +
    \frac{10\alpha q_+^2}{\mu}
    \norm{\boldsymbol r_{t+1}}^2
    +
    \frac{10\alpha}{\mu}
    \norm{\boldsymbol\eta_{t+1}}^2.
    \label{eq:descent-before-drift-theta}
\end{equation}

It remains to express the shifted error \(\boldsymbol d_t\) in terms of
the current tracking error and the minimizer drift. Since
\(\boldsymbol d_t
=
\boldsymbol{\theta}_t-\boldsymbol{\theta}_t^\star
+
\boldsymbol{\theta}_t^\star-\boldsymbol{\theta}_{t+1}^\star\),
Young's inequality with parameter \(\alpha\mu/4\) gives
\begin{align}
    \norm{\boldsymbol d_t}^2
    &\le
    \left(1+\frac{\alpha\mu}{4}\right)
    \norm{
        \boldsymbol{\theta}_t-\boldsymbol{\theta}_t^\star
    }^2
    +
    \left(1+\frac{4}{\alpha\mu}\right)
    \norm{
        \boldsymbol{\theta}_t^\star
        -
        \boldsymbol{\theta}_{t+1}^\star
    }^2
    \notag\\
    &\le
    \left(1+\frac{\alpha\mu}{4}\right)
    \norm{
        \boldsymbol{\theta}_t-\boldsymbol{\theta}_t^\star
    }^2
    +
    \frac{5}{\alpha\mu}
    \norm{
        \boldsymbol{\theta}_t^\star
        -
        \boldsymbol{\theta}_{t+1}^\star
    }^2,
    \label{eq:shifted-error-young-theta}
\end{align}
where the last step again uses \(\alpha\mu\le1\). Finally,
\(\left(1-\frac34\alpha\mu\right)
\left(1+\frac14\alpha\mu\right)
\le1-\frac12\alpha\mu\).
Substituting \cref{eq:shifted-error-young-theta} into
\cref{eq:descent-before-drift-theta} and using
\(q_+^2=\epsilon^{-2}\) therefore gives
\[
    \norm{
        \boldsymbol{\theta}_{t+1}
        -
        \boldsymbol{\theta}_{t+1}^\star
    }^2
    \le
    \left(1-\frac12\alpha\mu\right)
    \norm{
        \boldsymbol{\theta}_t
        -
        \boldsymbol{\theta}_t^\star
    }^2
    +
    \frac{5}{\alpha\mu}
    \norm{
        \boldsymbol{\theta}_t^\star
        -
        \boldsymbol{\theta}_{t+1}^\star
    }^2
    +
    \frac{10\alpha\epsilon^{-2}}{\mu}
    \norm{\boldsymbol r_{t+1}}^2
    +
    \frac{10\alpha}{\mu}
    \norm{\boldsymbol\eta_{t+1}}^2,
\]
which is exactly \cref{eq:app-adam-descent-recursion}.
\end{proof}

Having reduced the tracking recursion to the two Adam-specific error terms
\(\boldsymbol r_{t+1}\) and \(\boldsymbol\eta_{t+1}\), it remains to control
the first-moment tracking error. The high-probability control of
\(\boldsymbol\eta_{t+1}\) is provided by
\cref{lem:app-adam-eta-hp}. We now turn to
\(\boldsymbol r_{t+1}\), which has two sources: a bias from averaging stale
conditional mean gradients and a stochastic noise term obtained by averaging
martingale differences.

\vspace{0.5em}

\begin{lemma}[Bias--noise decomposition of \(\boldsymbol r_{t+1}\)]
\label{lem:app-adam-r-split}
For every \(t\ge0\),
$
    \boldsymbol r_{t+1}
    =
    \boldsymbol B_{t+1}^{(1)}
    +
    \boldsymbol N_{t+1}^{(1)},
$
where
\begin{align}
    \boldsymbol B_{t+1}^{(1)}
    &:=
    \sum_{k=0}^{t}
    w_{1,t+1,k}
    \left(
        \bar{\boldsymbol g}_{k+1}(\boldsymbol\theta_k)
        -
        \bar{\boldsymbol g}_{t+1}(\boldsymbol\theta_t)
    \right),
    \label{eq:app-adam-B1-def}\\
    \boldsymbol N_{t+1}^{(1)}
    &:=
    \sum_{k=0}^{t}
    w_{1,t+1,k}
    \boldsymbol\xi_{k+1}.
    \label{eq:app-adam-N1-def}
\end{align}
Moreover, if \(\operatorname{diam}(\Theta)\le R\), then $\norm{\boldsymbol B_{t+1}^{(1)}} \le
    2LR\,
    \frac{\beta_1(1-\beta_1^t)}
    {1-\beta_1^{t+1}}
    \le
    2LR\beta_1.$
    
\end{lemma}

\begin{proof}[Proof of \cref{lem:app-adam-r-split}]
By unrolling the bias-corrected first moment, we have
\[
    \widehat{\boldsymbol m}_{t+1}
    =
    \sum_{k=0}^{t}
    w_{1,t+1,k}
    \nabla_{\boldsymbol\theta}
    g(\boldsymbol\theta_k,X_{k+1})
    =
    \sum_{k=0}^{t}
    w_{1,t+1,k}
    \left(
        \bar{\boldsymbol g}_{k+1}(\boldsymbol\theta_k)
        +
        \boldsymbol\xi_{k+1}
    \right).
\]
Subtracting
\(\bar{\boldsymbol g}_{t+1}(\boldsymbol\theta_t)\) and using
\(\sum_{k=0}^{t}w_{1,t+1,k}=1\) gives
\[
    \boldsymbol r_{t+1}
    =
    \sum_{k=0}^{t}
    w_{1,t+1,k}
    \left(
        \bar{\boldsymbol g}_{k+1}(\boldsymbol\theta_k)
        -
        \bar{\boldsymbol g}_{t+1}(\boldsymbol\theta_t)
    \right)
    +
    \sum_{k=0}^{t}
    w_{1,t+1,k}
    \boldsymbol\xi_{k+1},
\]
which gives the desired decomposition. It remains to bound the bias
term. Since the conditional minimizers lie in the interior of
\(\Theta\),
\(\bar{\boldsymbol g}_{k+1}(\boldsymbol\theta_{k+1}^\star)
=\boldsymbol0\).
Thus, by \(L\)-Lipschitz continuity and
\(\operatorname{diam}(\Theta)\le R\), for every \(k\ge0\),
\[
    \norm{\bar{\boldsymbol g}_{k+1}(\boldsymbol\theta_k)}
    =
    \norm{
        \bar{\boldsymbol g}_{k+1}(\boldsymbol\theta_k)
        -
        \bar{\boldsymbol g}_{k+1}
        (\boldsymbol\theta_{k+1}^\star)
    }
    \le
    L\norm{
        \boldsymbol\theta_k-\boldsymbol\theta_{k+1}^\star
    }
    \le
    LR.
\]
Consequently, for \(0\le k\le t-1\),
\begin{align}
    \norm{
        \bar{\boldsymbol g}_{k+1}(\boldsymbol\theta_k)
        -
        \bar{\boldsymbol g}_{t+1}(\boldsymbol\theta_t)
    }
    &\le
    \norm{\bar{\boldsymbol g}_{k+1}(\boldsymbol\theta_k)}
    +
    \norm{\bar{\boldsymbol g}_{t+1}(\boldsymbol\theta_t)}
    \notag\\
    &\le
    2LR.
    \label{eq:app-adam-stale-gradient-difference}
\end{align}

Since the \(k=t\) summand in
\cref{eq:app-adam-B1-def} is zero and the Adam weights are
nonnegative, averaging \cref{eq:app-adam-stale-gradient-difference}
gives
\[
    \norm{\boldsymbol B_{t+1}^{(1)}}
    \le
    2LR
    \sum_{k=0}^{t-1}w_{1,t+1,k}.
\]
For \(t\ge1\), the geometric-series identity yields
\begin{align*}
    \sum_{k=0}^{t-1}w_{1,t+1,k}
    &=
    \frac{1-\beta_1}{1-\beta_1^{t+1}}
    \sum_{k=0}^{t-1}\beta_1^{t-k}
    =
    \frac{\beta_1(1-\beta_1^t)}
    {1-\beta_1^{t+1}}
    \le
    \beta_1,
\end{align*}
where the last inequality uses
\(1-\beta_1^t\le1-\beta_1^{t+1}\).
Therefore,
\[
    \norm{\boldsymbol B_{t+1}^{(1)}}
    \le
    2LR\,
    \frac{\beta_1(1-\beta_1^t)}
    {1-\beta_1^{t+1}}
    \le
    2LR\beta_1.
\]
When \(t=0\), \(\boldsymbol B_1^{(1)}=\boldsymbol0\), and the
same bound holds since \(1-\beta_1^0=0\).
This proves the claim.
\end{proof}

The bias component of \(\boldsymbol r_{t+1}\) is therefore controlled
pathwise by the total weight assigned to past conditional mean
gradients, which is at most \(\beta_1\). The stochastic component
is a weighted martingale sum. Under conditional sub-Gaussian
gradient noise, this weighted sum remains sub-Gaussian with effective
variance determined by the sum of squared Adam weights
\(\kappa_{1,t+1}\).

\vspace{0.5em}

\begin{lemma}[High-probability control of \(\boldsymbol r_{t+1}\)]
\label{lem:app-adam-r-hp}
Suppose Assumption~\ref{assump:cond-subgauss-noise-restate} holds and
\(\operatorname{diam}(\Theta)\le R\). Then there exists a universal
constant \(C>0\) such that, for every \(T\ge1\) and
\(\delta\in(0,1)\), with probability at least \(1-\delta\),
simultaneously for every \(t\in\{0,\ldots,T-1\}\),
\begin{equation}
    \norm{\boldsymbol r_{t+1}}
    \le
    2LR\,
    \frac{\beta_1(1-\beta_1^t)}
    {1-\beta_1^{t+1}}
    +
    C\sigma
    \sqrt{
        \kappa_{1,t+1}
        \left(
            d+\log\frac{2T}{\delta}
        \right)
    }.
    \label{eq:app-adam-r-hp-explicit}
\end{equation}
\end{lemma}

\begin{proof}[Proof of \cref{lem:app-adam-r-hp}]
Recall from \cref{lem:app-adam-r-split} that
\(
    \boldsymbol r_{t+1}
    =
    \boldsymbol B_{t+1}^{(1)}
    +
    \boldsymbol N_{t+1}^{(1)}
\),
where
\(
    \boldsymbol N_{t+1}^{(1)}
    =
    \sum_{k=0}^{t}
    w_{1,t+1,k}\boldsymbol\xi_{k+1}
\).
The bias component satisfies
\[
    \norm{\boldsymbol B_{t+1}^{(1)}}
    \le
    2LR\,
    \frac{\beta_1(1-\beta_1^t)}
    {1-\beta_1^{t+1}},
\]
so it remains to control \(\boldsymbol N_{t+1}^{(1)}\).

Fix \(t\in\{0,\ldots,T-1\}\) and
\(\boldsymbol u\in\mathbb S^{d-1}\). Define
\(Z_{k+1}:=
w_{1,t+1,k}
\langle\boldsymbol u,\boldsymbol\xi_{k+1}\rangle\),
\(k=0,\ldots,t\).
Since \(w_{1,t+1,k}\) is deterministic and
\(\mathbb E[\boldsymbol\xi_{k+1}\mid\mathcal F_k]=\boldsymbol0\),
the sequence \((Z_{k+1})_{k=0}^{t}\) is a martingale difference
sequence. Moreover,
Assumption~\ref{assump:cond-subgauss-noise-restate} gives
\(\norm{Z_{k+1}}_{\Psi_2\mid\mathcal F_k}
\le\sigma|w_{1,t+1,k}|\) a.s.
Therefore, by \cref{thm:martingale-subgaussian-hoeffding},
\[
    \mathbb P\left(
        \left|
            \left\langle
                \boldsymbol u,
                \boldsymbol N_{t+1}^{(1)}
            \right\rangle
        \right|
        \ge s
    \right)
    \le
    2\exp\left(
        -\frac{c s^2}
        {\sigma^2\sum_{k=0}^{t}w_{1,t+1,k}^2}
    \right)
    =
    2\exp\left(
        -\frac{c s^2}
        {\sigma^2\kappa_{1,t+1}}
    \right).
\]
Thus, for every \(\boldsymbol u\in\mathbb S^{d-1}\),
\(\langle\boldsymbol u,\boldsymbol N_{t+1}^{(1)}\rangle\) is
sub-Gaussian with scale at most
\(C\sigma\sqrt{\kappa_{1,t+1}}\).
Let \(\mathcal N\) be a \(1/2\)-net of \(\mathbb S^{d-1}\)
satisfying \(\lvert\mathcal N\rvert\le5^d\)
\cite[Corollary~4.2.11]{Vershynin_2018}.
By \cite[Exercise~4.35]{Vershynin_2018},
\(\norm{\boldsymbol x}
\le2\max_{\boldsymbol u\in\mathcal N}
|\langle\boldsymbol u,\boldsymbol x\rangle|\)
for every \(\boldsymbol x\in\mathbb R^d\).
A union bound over \(\mathcal N\) therefore gives, for every \(x>0\),
\[
    \mathbb P\left(
        \norm{\boldsymbol N_{t+1}^{(1)}}
        \ge
        C\sigma
        \sqrt{
            \kappa_{1,t+1}(d+x)
        }
    \right)
    \le
    2e^{-x}.
\]
Taking \(x=\log(2T/\delta)\) and applying a union bound over
\(t=0,\ldots,T-1\) yields, with probability at least \(1-\delta\),
simultaneously for every \(t\in\{0,\ldots,T-1\}\),
\[
    \norm{\boldsymbol N_{t+1}^{(1)}}
    \le
    C\sigma
    \sqrt{
        \kappa_{1,t+1}
        \left(
            d+\log\frac{2T}{\delta}
        \right)
    }.
\]
Combining this with the bound on $ \norm{\boldsymbol B_{t+1}^{(1)}}$
proves \cref{eq:app-adam-r-hp-explicit}.
\end{proof}

We can now combine the one-step recursion in
\cref{lem:app-adam-recursive-relation} with the high-probability control of
\(\boldsymbol r_{t+1}\) from \cref{lem:app-adam-r-hp} and
\(\boldsymbol\eta_{t+1}\) from \cref{lem:app-adam-eta-hp}. To simplify the
resulting discounted convolutions, define
\[
    \rho_\alpha
    :=
    1-\frac12\alpha\mu,
    \;\;
    \lambda_i
    :=
    \rho_\alpha\vee\beta_i,
    \; i\in\{1,2\}.
\]
Thus, \(\lambda_i\) is the slower of the tracking contraction and the
corresponding Adam moment decay.

\vspace{0.5em}

\begin{theorem}[High probability tracking error bound for Adam]
    \label{thm:app-high-prob-adam-restated}
    Suppose Assumptions~\ref{assump:cond-subgauss-noise-restate} and
    \ref{assumption:adaptive-strong-monotonicity-restated} hold and
    \(\operatorname{diam}(\Theta)\le R\). Then, for all \(t\in[T]\),
    \(\delta\in(0,1)\), and
    \(\alpha\le\min\{\mu\epsilon^2/(4L^2),\,\mu^{-1}\}\), the following
    tracking error bound holds for \plaineqref{eq:adam-update} with probability
    at least \(1-\delta\):
    \begin{align*}
        \norm{\boldsymbol\theta_t-\boldsymbol\theta_t^\star}^2
        \le\;&
        \rho_\alpha^t
        \norm{\boldsymbol\theta_0-\boldsymbol\theta_0^\star}^2
        +\frac{5\mathfrak D_t}{\alpha\mu}
        +\frac{40}{\mu^2\epsilon^2}
    \left\{
        4\beta_1^2L^2R^2
        +\frac{C^2\sigma^2}{1+\beta_1}
        \left(d+\log\frac{8T}{\delta}\right)
        \left[
            1-\beta_1
            +\alpha\mu\beta_1t\lambda_1^{t-1}
        \right]
    \right\}
        \\
        &+
        \frac{10C\alpha}{\mu\epsilon^4}
        \left[
            L^2R^2
            +\sigma^2
            \left(d+\log\frac{8T}{\delta}\right)
        \right]^2
        \left[
            \beta_2t\lambda_2^{t-1}
            +\frac{2(1-\beta_2)}{\alpha\mu}
        \right],
    \end{align*}
    where
    \(\rho_\alpha:=1-\alpha\mu/2\),
    \(\mathfrak D_t:=\sum_{\ell=0}^{t-1}
    \rho_\alpha^{t-\ell-1}\Delta_\ell^2\),
    \(\Delta_\ell
    :=\norm{\boldsymbol\theta_\ell^\star-\boldsymbol\theta_{\ell+1}^\star}\),
    and \(\lambda_i:=\rho_\alpha\vee\beta_i\) for \(i\in\{1,2\}\).
\end{theorem}

\begin{proof}[Proof of \cref{thm:app-high-prob-adam-restated}]
Fix $s\in[T]$. Iterating
\cref{lem:app-adam-recursive-relation} from $0$ to $s-1$ gives
\begin{align}
    \|\boldsymbol\theta_s-\boldsymbol\theta_s^\star\|^2
    \le\;&
    \rho_\alpha^s
    \|\boldsymbol\theta_0-\boldsymbol\theta_0^\star\|^2
    +
    \frac{5\mathfrak D_s}{\alpha\mu}
    +
    \frac{10\alpha q_+^2}{\mu}
    \sum_{\ell=0}^{s-1}
    \rho_\alpha^{s-\ell-1}
    \|\boldsymbol r_{\ell+1}\|^2
    +
    \frac{10\alpha}{\mu}
    \sum_{\ell=0}^{s-1}
    \rho_\alpha^{s-\ell-1}
    \|\boldsymbol\eta_{\ell+1}\|^2.
    \label{eq:app-adam-hp-proof-recursion}
\end{align}

Apply \cref{lem:app-adam-r-hp,lem:app-adam-eta-hp} with failure
probability $\delta/2$ each. By a union bound, with probability at least
$1-\delta$, simultaneously for every $\ell=0,\ldots,T-1$,
\[
\begin{aligned}
    &\|\boldsymbol r_{\ell+1}\|
    \le
    2LR
    \frac{\beta_1(1-\beta_1^\ell)}
    {1-\beta_1^{\ell+1}}
    +
    C\sigma
    \sqrt{
        \kappa_{1,\ell+1}
        \left(
            d+\log\frac{8T}{\delta}
        \right)
    },\\
    &\|\boldsymbol\eta_{\ell+1}\|^2
    \le
    C\epsilon^{-4}
    \left[
        L^2R^2
        +
        \sigma^2
        \left(
            d+\log\frac{8T}{\delta}
        \right)
    \right]^2
    \vartheta_{2,\ell+1}.
\end{aligned}
\]
Hence, by $(a+b)^2\le2a^2+2b^2$,
\[
\|\boldsymbol r_{\ell+1}\|^2
\le
8L^2R^2
\left(
    \frac{\beta_1(1-\beta_1^\ell)}
    {1-\beta_1^{\ell+1}}
\right)^2
+
2C^2\sigma^2\kappa_{1,\ell+1}
\left(
    d+\log\frac{8T}{\delta}
\right).
\]

Since $1-\rho_\alpha=\alpha\mu/2$,
$\sum_{\ell=0}^{s-1}\rho_\alpha^{s-\ell-1}
\le2/(\alpha\mu)$.
Moreover,
$\beta_1(1-\beta_1^\ell)/(1-\beta_1^{\ell+1})\le\beta_1$, so
\[
\sum_{\ell=0}^{s-1}
\rho_\alpha^{s-\ell-1}
\left(
    \frac{\beta_1(1-\beta_1^\ell)}
    {1-\beta_1^{\ell+1}}
\right)^2
\le
\frac{2\beta_1^2}{\alpha\mu}.
\]

For the variance term, recall that
\[
\kappa_{1,\ell+1}
=
\frac{1-\beta_1}{1+\beta_1}
+
\frac{
    2(1-\beta_1)\beta_1^{\ell+1}
}{
    (1+\beta_1)(1-\beta_1^{\ell+1})
}
\le
\frac{1-\beta_1}{1+\beta_1}
+
\frac{2\beta_1^{\ell+1}}{1+\beta_1},
\]
where we used $1-\beta_1^{\ell+1}\ge1-\beta_1$.
Let $\lambda_1:=\rho_\alpha\vee\beta_1$. Then
$\rho_\alpha^{s-\ell-1}\beta_1^{\ell+1}
\le\beta_1\lambda_1^{s-1}$, and therefore
\[
\begin{aligned}
    \sum_{\ell=0}^{s-1}
    \rho_\alpha^{s-\ell-1}
    \kappa_{1,\ell+1}
    &\le
    \frac{1-\beta_1}{1+\beta_1}
    \sum_{\ell=0}^{s-1}\rho_\alpha^{s-\ell-1}
    +
    \frac{2}{1+\beta_1}
    \sum_{\ell=0}^{s-1}
    \rho_\alpha^{s-\ell-1}\beta_1^{\ell+1}\\
    &\le
    \frac{2}{1+\beta_1}
    \left[
        \frac{1-\beta_1}{\alpha\mu}
        +
        \beta_1s\lambda_1^{s-1}
    \right].
\end{aligned}
\]
Consequently,
\[
\begin{aligned}
    \frac{10\alpha q_+^2}{\mu}
    \sum_{\ell=0}^{s-1}
    \rho_\alpha^{s-\ell-1}
    \|\boldsymbol r_{\ell+1}\|^2
    \le
    \frac{40q_+^2}{\mu^2}
    \left\{
        4\beta_1^2L^2R^2
        +
        \frac{C^2\sigma^2}{1+\beta_1}
        \left(
            d+\log\frac{8T}{\delta}
        \right)
        \left[
            1-\beta_1
            +
            \alpha\mu\beta_1s\lambda_1^{s-1}
        \right]
    \right\}.
\end{aligned}
\]

Next, since
$\vartheta_{2,\ell+1}=(1-\beta_2)+\beta_2^{\ell+1}$, let
$\lambda_2:=\rho_\alpha\vee\beta_2$. As above,
$\sum_{\ell=0}^{s-1}
\rho_\alpha^{s-\ell-1}\beta_2^{\ell+1}
\le\beta_2s\lambda_2^{s-1}$, and hence
\[
\sum_{\ell=0}^{s-1}
\rho_\alpha^{s-\ell-1}
\vartheta_{2,\ell+1}
\le
\frac{2(1-\beta_2)}{\alpha\mu}
+
\beta_2s\lambda_2^{s-1}.
\]
It follows that
\[
\begin{aligned}
    \frac{10\alpha}{\mu}
    \sum_{\ell=0}^{s-1}
    \rho_\alpha^{s-\ell-1}
    \|\boldsymbol\eta_{\ell+1}\|^2
    \le
    \frac{10C\alpha}{\mu\epsilon^4}
    \left[
        L^2R^2
        +
        \sigma^2
        \left(
            d+\log\frac{8T}{\delta}
        \right)
    \right]^2
    \left[
        \frac{2(1-\beta_2)}{\alpha\mu}
        +
        \beta_2s\lambda_2^{s-1}
    \right].
\end{aligned}
\]

Substituting these estimates into
\cref{eq:app-adam-hp-proof-recursion} and using
$q_+=\epsilon^{-1}$ gives
\begin{align*}
    \|\boldsymbol\theta_s-\boldsymbol\theta_s^\star\|^2
    \le\;&
    \rho_\alpha^s
    \|\boldsymbol\theta_0-\boldsymbol\theta_0^\star\|^2
    +
    \frac{5\mathfrak D_s}{\alpha\mu}
    +
    \frac{160\beta_1^2L^2R^2}
    {\mu^2\epsilon^2}
    +
    \frac{40C^2\sigma^2}
    {\mu^2\epsilon^2(1+\beta_1)}
    \left(
        d+\log\frac{8T}{\delta}
    \right)
    \left[
        1-\beta_1
        +
        \alpha\mu\beta_1s\lambda_1^{s-1}
    \right]
    \\
    &+
    \frac{10C\alpha}{\mu\epsilon^4}
    \left[
        L^2R^2
        +
        \sigma^2
        \left(
            d+\log\frac{8T}{\delta}
        \right)
    \right]^2
    \left[
        \frac{2(1-\beta_2)}{\alpha\mu}
        +
        \beta_2s\lambda_2^{s-1}
    \right].
\end{align*}
Since the concentration event is uniform over
$\ell=0,\ldots,T-1$, this bound holds simultaneously for every
$s\in[T]$, proving the claim.
\end{proof}

\subsection{Proof of \cref{thm:time-to-track-hp-adam}}
\label{app:C2}

Using \cref{thm:app-high-prob-adam-restated}, we can obtain the following result which gives us an algorithmic guarantee for \plaineqref{eq:adam-update-restate}:

\vspace{0.5em}

\begin{theorem}[Time to reach the tracking floor with high probability for \plaineqref{eq:adam-update-restate}]
\label{thm:time-to-track-hp-adam-restate}
Under the standing assumptions of
\cref{thm:app-high-prob-adam-restated}, suppose
$0<\alpha \le \alpha_{\max} := \min\{\mu\epsilon^2/(4L^2),\,\mu^{-1}\}$,
$\Delta_t \le \Delta$ almost surely for all $t\ge 0$, and fix a
deterministic horizon bound $N\ge1$ and $\delta\in(0,1)$.
Define the tracking-floor
$$
\mathcal{E}_{\mathrm{A}}(N,\delta,\alpha)
\asymp 
\frac{\Delta^2}{\mu^2\alpha^2}
+
\frac{\beta_1^2L^2R^2}{\mu^2\epsilon^2}
+
\frac{C^2\sigma^2(1-\beta_1)}
{\mu^2\epsilon^2(1+\beta_1)}
\left(d+\log\frac{N}{\delta}\right)
+
\frac{C(1-\beta_2)}{\mu^2\epsilon^4}
\left[
    L^2R^2
    +
    \sigma^2\left(d+\log\frac{N}{\delta}\right)
\right]^2,
$$
where $C>0$ is a fixed sufficiently large universal constant.
For $\alpha\in(0,\alpha_{\max}]$, let
$\rho_\alpha:=1-\alpha\mu/2$ and
$\lambda_{i,\alpha}:=\rho_\alpha\vee\beta_i$ for $i\in\{1,2\}$.
For a target floor $E>0$, define
\begin{align*}
\mathcal{B}_1(\alpha,E)
&\asymp
\frac{1}{1-\lambda_{1,\alpha}}
\log_{+}\!\left(
    \frac{
        C^2\alpha\sigma^2
        \left(d+\log\frac{N}{\delta}\right)
    }{
        \mu\epsilon^2(1+\beta_1)
        (1-\lambda_{1,\alpha})E
    }
\right) \\
\mathcal{B}_2(\alpha,E)
&\asymp
\frac{1}{1-\lambda_{2,\alpha}}
\log_{+}\!\left(
    \frac{
        C\alpha
        \left[
            L^2R^2
            +
            \sigma^2\left(d+\log\frac{N}{\delta}\right)
        \right]^2
    }{
        \mu\epsilon^4
        (1-\lambda_{2,\alpha})E
    }
\right).
\end{align*}
Then we have the following:
\begin{enumerate}
    \item \textbf{(Constant learning rate).} If $\alpha_t\equiv\alpha$,
    with $\alpha$ chosen deterministically, then
    for any integer $1\le T\le N$, with probability $\ge 1-\delta$,
    for all $t\in[T]$,
    $$
    \begin{aligned}
        \|\boldsymbol{\theta}_t-\boldsymbol{\theta}_t^\star\|^2
        \;\lesssim\;&
        \rho_\alpha^t
        \|\boldsymbol{\theta}_0-\boldsymbol{\theta}_0^\star\|^2
        +
        \mathcal{E}_{\mathrm{A}}(N,\delta,\alpha)
        +
        \frac{\alpha\sigma^2}
        {\mu\epsilon^2(1+\beta_1)}
        \left(d+\log\frac{N}{\delta}\right)
        \sum_{\ell=0}^{t-1}
        \rho_\alpha^{t-\ell-1}\beta_1^{\ell+1}
        \\
        &+
        \frac{\alpha}{\mu\epsilon^4}
        \left[
            L^2R^2
            +\sigma^2\left(d+\log\frac{N}{\delta}\right)
        \right]^2
        \sum_{\ell=0}^{t-1}
        \rho_\alpha^{t-\ell-1}\beta_2^{\ell+1}.
    \end{aligned}
    $$
    Let $\alpha_{\mathrm{A}}^\star\in(0,\alpha_{\max}]$
    be a deterministic target stepsize and
    $
        \mathcal{E}_{\mathrm{A}}^\star
        :=
        \mathcal{E}_{\mathrm{A}}
        (N,\delta,\alpha_{\mathrm{A}}^\star).
    $
    Then, running Adam with $\alpha=\alpha_{\mathrm A}^\star$ gives
    $\|\boldsymbol{\theta}_t-\boldsymbol{\theta}_t^\star\|^2
    \lesssim
    \mathcal{E}_{\mathrm{A}}^\star$
    after at most
    $$
    \begin{aligned}
        t_{\mathrm A}^\star
        \;\lesssim\;&
        \frac{1}{\mu\alpha_{\mathrm{A}}^\star}
        \log_{+}\!\left(
            \frac{
                \|\boldsymbol{\theta}_0
                -\boldsymbol{\theta}_0^\star\|^2
            }{
                \mathcal{E}_{\mathrm{A}}^\star
            }
        \right)
        +
        \mathcal{B}_1
        \left(
            \alpha_{\mathrm{A}}^\star,
            \mathcal{E}_{\mathrm{A}}^\star
        \right)
        +
        \mathcal{B}_2
        \left(
            \alpha_{\mathrm{A}}^\star,
            \mathcal{E}_{\mathrm{A}}^\star
        \right)
    \end{aligned}
    $$
    iterations, provided $t_{\mathrm A}^\star\le T$.
    The bound holds simultaneously for
    $t_{\mathrm A}^\star\le t\le T$ with probability $\ge1-\delta$.

    \item \textbf{(Step-decay with Adam-state restart).}
    Suppose a deterministic upper bound
    $D\ge \|\boldsymbol{\theta}_0-\boldsymbol{\theta}_0^\star\|^2$
    is available and $\alpha_{\mathrm{A}}^\star<\alpha_{\max}$.
    Set $\alpha_0:=\alpha_{\max}$,
    $\alpha_k:=(\alpha_{k-1}+\alpha_{\mathrm{A}}^\star)/2$, and
    $K:=1+\lceil\log_2(\alpha_0/\alpha_{\mathrm{A}}^\star)\rceil$.
    For each $k\geq0$, let
    $
        E_k
        :=
        \mathcal{E}_{\mathrm{A}}(N,\delta,\alpha_k),
    $
    and set
    $$
    \begin{aligned}
        T_0
        &\asymp
        \left\lceil
        \max\left\{
            \frac{1}{\mu\alpha_0}
            \log_{+}\!\left(\frac{D}{E_0}\right),
            \,
            \mathcal{B}_1(\alpha_0,E_0),
            \,
            \mathcal{B}_2(\alpha_0,E_0)
        \right\}
        \right\rceil,
        \\
        T_k
        &\asymp
        \left\lceil
        \max\left\{
            \frac{1}{\mu\alpha_k},
            \,
            \mathcal{B}_1(\alpha_k,E_k),
            \,
            \mathcal{B}_2(\alpha_k,E_k)
        \right\}
        \right\rceil,
        \;\; k\geq1.
    \end{aligned}
    $$
    Suppose $T:=\sum_{k=0}^{K-1}T_k\le N$.
    Running Adam at constant stepsize $\alpha_k$ for $T_k$ steps per
    epoch with $(\boldsymbol m,\boldsymbol v)$ and the bias-correction restarted at the beginning of every epoch yields
    $
        \|\boldsymbol{\theta}_T-\boldsymbol{\theta}_T^\star\|^2
        \lesssim
        \mathcal{E}_{\mathrm{A}}^\star
    $
    with probability $\geq1-K\delta$
    Moreover,
    $$
    \begin{aligned}
        T
        \;\lesssim\;
        \frac{1}{\mu\alpha_0}
        \log_{+}\!\left(
            \frac{D}{\mathcal{E}_{\mathrm{A}}^\star}
        \right)
        +
        \frac{1}{\mu\alpha_{\mathrm{A}}^\star}
        +
        \sum_{k=0}^{K-1}
        \mathcal{B}_1(\alpha_k,E_k)
        +
        \sum_{k=0}^{K-1}
        \mathcal{B}_2(\alpha_k,E_k).
    \end{aligned}
    $$
\end{enumerate}
\end{theorem}

\begin{proof}[Proof of \cref{thm:time-to-track-hp-adam-restate}]
Fix a deterministic stepsize
$0<\alpha\le\alpha_{\max}$ and an integer $1\le T\le N$.
By \cref{lem:app-adam-r-split},
$\|\boldsymbol B_{t+1}^{(1)}\|
\le 2LR\beta_1(1-\beta_1^t)/(1-\beta_1^{t+1})
\le 2LR\beta_1$.
Applying
\cref{lem:app-adam-r-hp,lem:app-adam-eta-hp}
with horizon $N$ and failure probability $\delta/2$ each, and taking a
union bound, gives an event of probability at least $1-\delta$ on which,
simultaneously for every $\ell=0,\ldots,N-1$,
\[
\begin{aligned}
    &\|\boldsymbol r_{\ell+1}\|
    \le
    2LR
    \frac{\beta_1(1-\beta_1^\ell)}
    {1-\beta_1^{\ell+1}}
    +
    C\sigma
    \sqrt{
        \kappa_{1,\ell+1}
        \left(
            d+\log\frac{8T}{\delta}
        \right)
    },\\
    &\|\boldsymbol\eta_{\ell+1}\|^2
    \le
    C\epsilon^{-4}
    \left[
        L^2R^2
        +
        \sigma^2
        \left(
            d+\log\frac{8T}{\delta}
        \right)
    \right]^2
    \vartheta_{2,\ell+1}.
\end{aligned}
\]
Combining these bounds with
\cref{lem:app-adam-recursive-relation} and using $q_+=\epsilon^{-1}$,
we obtain, on the same event, for every $t\in[T]$,
\begin{align}
    \|\boldsymbol{\theta}_{t}-\boldsymbol{\theta}_{t}^{\star}\|^2
    \le\;&
    \rho_{\alpha}^t
    \|\boldsymbol{\theta}_{0}-\boldsymbol{\theta}_{0}^{\star}\|^2
    +
    \frac{5}{\alpha\mu}
    \sum_{\ell=0}^{t-1}
    \rho_{\alpha}^{t-\ell-1}\Delta_{\ell}^2
    +
    \frac{10C\alpha}{\mu\epsilon^4}
    \left[
        L^2R^2
        +
        \sigma^2
        \left(
            d+\log\frac{N}{\delta}
        \right)
    \right]^2
    \sum_{\ell=0}^{t-1}
    \rho_{\alpha}^{t-\ell-1}\vartheta_{2,\ell+1}
    \notag\\
    &+
    \frac{10\alpha}{\mu\epsilon^2}
    \sum_{\ell=0}^{t-1}
    \rho_{\alpha}^{t-\ell-1}
    \left(
        2LR\beta_1
        +
        C\sigma
        \sqrt{
            \kappa_{1,\ell+1}
            \left(
                d+\log\frac{N}{\delta}
            \right)
        }
    \right)^2 .
    \label{eq:adam-time-proof-start}
\end{align}

Since $1-\rho_\alpha=\alpha\mu/2$,
$\sum_{\ell=0}^{t-1}\rho_\alpha^{t-\ell-1}
\le 2/(\alpha\mu)$.
Hence the drift term in \cref{eq:adam-time-proof-start} is at most
$10\Delta^2/(\alpha^2\mu^2)$.
Moreover,
$\vartheta_{2,\ell+1}=(1-\beta_2)+\beta_2^{\ell+1}$, so
\[
\sum_{\ell=0}^{t-1}
\rho_\alpha^{t-\ell-1}\vartheta_{2,\ell+1}
\le
\frac{2(1-\beta_2)}{\alpha\mu}
+
\sum_{\ell=0}^{t-1}
\rho_\alpha^{t-\ell-1}\beta_2^{\ell+1}.
\]
Similarly, since
\[
\kappa_{1,\ell+1}
=
\frac{1-\beta_1}{1+\beta_1}
+
\frac{
    2(1-\beta_1)\beta_1^{\ell+1}
}{
    (1+\beta_1)(1-\beta_1^{\ell+1})
}
\le
\frac{1-\beta_1}{1+\beta_1}
+
\frac{2\beta_1^{\ell+1}}{1+\beta_1},
\]
we have
\[
\sum_{\ell=0}^{t-1}
\rho_\alpha^{t-\ell-1}\kappa_{1,\ell+1}
\le
\frac{2(1-\beta_1)}
{\alpha\mu(1+\beta_1)}
+
\frac{2}{1+\beta_1}
\sum_{\ell=0}^{t-1}
\rho_\alpha^{t-\ell-1}\beta_1^{\ell+1}.
\]
Using $(a+b)^2\le2a^2+2b^2$ in the last term of
\cref{eq:adam-time-proof-start} therefore yields
\begin{align}
    \|\boldsymbol{\theta}_t-\boldsymbol{\theta}_t^\star\|^2
    \lesssim\;&
    \rho_\alpha^t
    \|\boldsymbol{\theta}_0-\boldsymbol{\theta}_0^\star\|^2
    +
    \mathcal E_{\mathrm A}(N,\delta,\alpha)
    +
    \frac{\alpha\sigma^2}
    {\mu\epsilon^2(1+\beta_1)}
    \left(
        d+\log\frac{N}{\delta}
    \right)
    \sum_{\ell=0}^{t-1}
    \rho_\alpha^{t-\ell-1}\beta_1^{\ell+1}
    \notag\\
    &+
    \frac{\alpha}{\mu\epsilon^4}
    \left[
        L^2R^2
        +
        \sigma^2
        \left(
            d+\log\frac{N}{\delta}
        \right)
    \right]^2
    \sum_{\ell=0}^{t-1}
    \rho_\alpha^{t-\ell-1}\beta_2^{\ell+1}.
    \label{eq:adam-time-finite}
\end{align}
This proves the finite-time bound in part~(i).

We next control the three transient terms in
\cref{eq:adam-time-finite}.
Set $\alpha=\alpha_{\mathrm A}^\star$.
Since $1-x\le e^{-x}$,
$\rho_{\alpha_{\mathrm A}^\star}^t
\le e^{-\mu\alpha_{\mathrm A}^\star t/2}$.
Thus the initialization term is at most
$\mathcal E_{\mathrm A}^\star$ whenever
$t\ge
2(\mu\alpha_{\mathrm A}^\star)^{-1}
\log_+(
\|\boldsymbol\theta_0-\boldsymbol\theta_0^\star\|^2/
\mathcal E_{\mathrm A}^\star)$. For $i\in\{1,2\}$, let
$\lambda_{i,\alpha}
=\rho_\alpha\vee\beta_i$.
Then
\[
\begin{aligned}
    \sum_{\ell=0}^{t-1}
    \rho_\alpha^{t-\ell-1}\beta_i^{\ell+1}
    &\le
    t\lambda_{i,\alpha}^t\le
    \frac{2}{1-\lambda_{i,\alpha}}
    \exp\!\left(
        -\frac{1-\lambda_{i,\alpha}}{2}t
    \right),
\end{aligned}
\]
where the second inequality follows from
$\sup_{x\ge0}x\lambda^{x/2}
=2/[e\log(1/\lambda)]$ and
$\log(1/\lambda)\ge1-\lambda$.
By the definitions of
$\mathcal B_1$ and $\mathcal B_2$, the first and second-moment
transients in \cref{eq:adam-time-finite} are each at most
$\mathcal E_{\mathrm A}^\star$ for
$t\ge
\mathcal B_1(
\alpha_{\mathrm A}^\star,\mathcal E_{\mathrm A}^\star)$
and
$t\ge
\mathcal B_2(
\alpha_{\mathrm A}^\star,\mathcal E_{\mathrm A}^\star)$,
respectively.
Consequently, taking
\[
t_{\mathrm A}^\star \lesssim \frac{1}{\mu\alpha_{\mathrm A}^\star}
\log_+\!\left(
    \frac{
        \|\boldsymbol\theta_0-\boldsymbol\theta_0^\star\|^2
    }{
        \mathcal E_{\mathrm A}^\star
    }
\right)
+
\mathcal B_1(
    \alpha_{\mathrm A}^\star,\mathcal E_{\mathrm A}^\star)
+
\mathcal B_2(
    \alpha_{\mathrm A}^\star,\mathcal E_{\mathrm A}^\star)
\]
gives
$\|\boldsymbol\theta_t-\boldsymbol\theta_t^\star\|^2
\lesssim\mathcal E_{\mathrm A}^\star$
simultaneously for every
$t_{\mathrm A}^\star\le t\le T$, provided
$t_{\mathrm A}^\star\le T$.
This proves part~(i).

For part~(ii), let
$t_k:=\sum_{j=0}^{k-1}T_j$,
$\boldsymbol X_k:=\boldsymbol\theta_{t_k}$, and
$\boldsymbol X_k^\star:=\boldsymbol\theta_{t_k}^\star$.
The epoch boundaries and stepsizes are deterministic.
Since $(\boldsymbol m,\boldsymbol v)$ and the bias-correction are
restarted at every epoch, the zero-initialized moment expansions used
above apply with respect to the shifted filtration
$(\mathcal F_{t_k+j})_{j\ge0}$.
Applying part~(i) on each epoch and taking a union bound over the $K$
epochs gives an event of probability at least $1-K\delta$ on which
\begin{align}
    \|\boldsymbol X_{k+1}-\boldsymbol X_{k+1}^\star\|^2
    \lesssim\;&
    e^{-\mu\alpha_kT_k/2}
    \|\boldsymbol X_k-\boldsymbol X_k^\star\|^2
    +
    E_k
    +
    \frac{\alpha_k\sigma^2}
    {\mu\epsilon^2(1+\beta_1)}
    \left(
        d+\log\frac{N}{\delta}
    \right)
    \sum_{\ell=0}^{T_k-1}
    \rho_{\alpha_k}^{T_k-\ell-1}\beta_1^{\ell+1}
    \notag\\
    &+
    \frac{\alpha_k}{\mu\epsilon^4}
    \left[
        L^2R^2
        +
        \sigma^2
        \left(
            d+\log\frac{N}{\delta}
        \right)
    \right]^2
    \sum_{\ell=0}^{T_k-1}
    \rho_{\alpha_k}^{T_k-\ell-1}\beta_2^{\ell+1}
    \label{eq:adam-epoch-recursion}
\end{align}
for every $k=0,\ldots,K-1$.

By the convolution bound above and the definitions of
$\mathcal B_1(\alpha_k,E_k)$ and
$\mathcal B_2(\alpha_k,E_k)$, the choice of $T_k$ makes both moment
transients in \cref{eq:adam-epoch-recursion} of order at most $E_k$.
The choice of $T_0$ also gives
$e^{-\mu\alpha_0T_0/2}D\lesssim E_0$, whereas for every $k\ge1$,
$T_k\gtrsim(\mu\alpha_k)^{-1}$ gives
$e^{-\mu\alpha_kT_k/2}\le c$ for a fixed $c<1$.
Hence
$\|\boldsymbol X_1-\boldsymbol X_1^\star\|^2\lesssim E_0$.

The recursion
$\alpha_k=(\alpha_{k-1}+\alpha_{\mathrm A}^\star)/2$
implies $\alpha_k\le\alpha_{k-1}$.
Since the only $\alpha$-dependent term in
$\mathcal E_{\mathrm A}(N,\delta,\alpha)$ is proportional to
$\alpha^{-2}$, we have $E_{k-1}\lesssim E_k$.
Inductively, if
$\|\boldsymbol X_k-\boldsymbol X_k^\star\|^2\lesssim E_{k-1}$,
then \cref{eq:adam-epoch-recursion} gives
$\|\boldsymbol X_{k+1}-\boldsymbol X_{k+1}^\star\|^2
\lesssim E_k$.
Therefore,
$\|\boldsymbol X_k-\boldsymbol X_k^\star\|^2
\lesssim E_{k-1}$ for every $k=1,\ldots,K$.

The stepsize recursion has the explicit solution
$\alpha_k
=\alpha_{\mathrm A}^\star
+2^{-k}(\alpha_0-\alpha_{\mathrm A}^\star)$.
Thus
$\alpha_{\mathrm A}^\star\le\alpha_{K-1}
\le2\alpha_{\mathrm A}^\star$ by the definition of $K$.
Since
$\mathcal E_{\mathrm A}(N,\delta,\alpha)$ is nonincreasing in $\alpha$,
$E_{K-1}\lesssim\mathcal E_{\mathrm A}^\star$, and consequently
\[
\|\boldsymbol\theta_T-\boldsymbol\theta_T^\star\|^2
=
\|\boldsymbol X_K-\boldsymbol X_K^\star\|^2
\lesssim
E_{K-1}
\lesssim
\mathcal E_{\mathrm A}^\star
\]
with probability at least $1-K\delta$.

It remains to bound $T$.
The definitions of the epoch lengths give
\[
\begin{aligned}
    T
    \lesssim\;&
    \frac{1}{\mu\alpha_0}
    \log_+\!\left(\frac{D}{E_0}\right)
    +
    \sum_{k=1}^{K-1}\frac{1}{\mu\alpha_k}
    +
    \sum_{k=0}^{K-1}\mathcal B_1(\alpha_k,E_k)
    +
    \sum_{k=0}^{K-1}\mathcal B_2(\alpha_k,E_k).
    \label{eq:adam-total-time-pre}
\end{aligned}
\]
Let $r:=\alpha_0/\alpha_{\mathrm A}^\star\ge1$.
From the form of the tracking floor,
$E_0\lesssim\mathcal E_{\mathrm A}^\star\lesssim r^2E_0$, and hence
\[
\frac{1}{\mu\alpha_0}
\log_+\!\left(\frac{D}{E_0}\right)
\lesssim
\frac{1}{\mu\alpha_0}
\log_+\!\left(
    \frac{D}{\mathcal E_{\mathrm A}^\star}
\right)
+
\frac{1}{\mu\alpha_{\mathrm A}^\star},
\]
where we used $(\log r)/r\le1/e$. Finally,
$\sum_{k=1}^{K-1}\alpha_k^{-1}
\lesssim(\alpha_{\mathrm A}^\star)^{-1}$.
Indeed, this is immediate when $r\le2$. If $r>2$, writing
$n:=K-1=\lceil\log_2r\rceil$ and using
$\alpha_k\ge2^{-k}(\alpha_0-\alpha_{\mathrm A}^\star)$ gives
\[
\sum_{k=1}^{K-1}\frac{1}{\alpha_k}
\le
\frac{2^{n+1}-2}
{\alpha_0-\alpha_{\mathrm A}^\star}
\le
\frac{8}{\alpha_{\mathrm A}^\star}.
\]
Substituting these estimates into
\cref{eq:adam-total-time-pre} yields
\[
\begin{aligned}
    T
    \lesssim\;&
    \frac{1}{\mu\alpha_0}
    \log_+\!\left(
        \frac{D}{\mathcal E_{\mathrm A}^\star}
    \right)
    +
    \frac{1}{\mu\alpha_{\mathrm A}^\star}
    +
    \sum_{k=0}^{K-1}
    \mathcal B_1(\alpha_k,E_k)
    +
    \sum_{k=0}^{K-1}
    \mathcal B_2(\alpha_k,E_k),
\end{aligned}
\]
which is the claimed bound.
\end{proof}

\subsection{Proofs for expected tracking error}
\label{app:C3}

The high-probability analysis above controls the stochastic first-moment
and preconditioner errors uniformly over time. For completeness, we also
give their corresponding second-moment bounds. In expectation, the analysis
is simpler: orthogonality of the martingale differences controls the
first-moment noise, while conditional sub-Gaussianity provides the required
fourth-moment control for the preconditioner perturbation.

\vspace{0.5em}

\begin{lemma}[Second-moment control of the first-moment tracking error]
\label{lem:app-adam-r-second-moment}
Under the standing setup, suppose
Assumption~\ref{assump:cond-subgauss-noise-restate} holds and
\(\operatorname{diam}(\Theta)\le R\). Then, for every \(t\ge0\),
\begin{equation}
    \mathbb E\norm{\boldsymbol r_{t+1}}^2
    \le
    8L^2R^2
    \left(
        \frac{\beta_1(1-\beta_1^t)}
        {1-\beta_1^{t+1}}
    \right)^2
    +
    2C d\sigma^2\kappa_{1,t+1},
    \label{eq:app-adam-r-second-moment}
\end{equation}
where \(C>0\) is a universal constant.
\end{lemma}

\begin{proof}[Proof of \cref{lem:app-adam-r-second-moment}]
By \cref{lem:app-adam-r-split},
\[
    \boldsymbol r_{t+1}
    =
    \boldsymbol B_{t+1}^{(1)}
    +
    \boldsymbol N_{t+1}^{(1)},
    \;\;
    \norm{\boldsymbol B_{t+1}^{(1)}}
    \le
    2LR\,
    \frac{\beta_1(1-\beta_1^t)}
    {1-\beta_1^{t+1}}.
\]
Thus, using \((a+b)^2\le2a^2+2b^2\),
\[
    \mathbb E\norm{\boldsymbol r_{t+1}}^2
    \le
    8L^2R^2
    \left(
        \frac{\beta_1(1-\beta_1^t)}
        {1-\beta_1^{t+1}}
    \right)^2
    +
    2\mathbb E\norm{\boldsymbol N_{t+1}^{(1)}}^2.
\]
Since
\(\boldsymbol N_{t+1}^{(1)}
=\sum_{k=0}^{t}w_{1,t+1,k}\boldsymbol\xi_{k+1}\),
the martingale-difference property gives, for \(k<\ell\),
$
    \mathbb E
    \left\langle
        \boldsymbol\xi_{k+1},
        \boldsymbol\xi_{\ell+1}
    \right\rangle
    =
    \mathbb E
    \left\langle
        \boldsymbol\xi_{k+1},
        \mathbb E[
            \boldsymbol\xi_{\ell+1}\mid\mathcal F_\ell
        ]
    \right\rangle
    =
    0
$
where \(\boldsymbol\xi_{k+1}\) is
\(\mathcal F_\ell\)-measurable and the noise vectors are
square-integrable by conditional sub-Gaussianity.
Hence
\[
    \mathbb E\norm{\boldsymbol N_{t+1}^{(1)}}^2
    =
    \sum_{k=0}^{t}
    w_{1,t+1,k}^2
    \mathbb E\norm{\boldsymbol\xi_{k+1}}^2.
\]
Assumption~\ref{assump:cond-subgauss-noise-restate} implies
\(\mathbb E[\norm{\boldsymbol\xi_{k+1}}^2\mid\mathcal F_k]
\le C d\sigma^2\) a.s., and therefore
\[
    \mathbb E\norm{\boldsymbol N_{t+1}^{(1)}}^2
    \le
    C d\sigma^2
    \sum_{k=0}^{t}w_{1,t+1,k}^2
    =
    C d\sigma^2\kappa_{1,t+1}.
\]
Substituting this bound above proves
\cref{eq:app-adam-r-second-moment}.
\end{proof}

The next lemma gives the corresponding second-moment control of the
preconditioner perturbation. Without assuming bounded sample gradients,
conditional sub-Gaussianity provides sufficient fourth-moment control.

\vspace{0.5em}

\begin{lemma}[Second-moment control of the preconditioner perturbation]
\label{lem:app-adam-eta-second-moment}
Under the standing setup, suppose
Assumption~\ref{assump:cond-subgauss-noise-restate} holds and
\(\operatorname{diam}(\Theta)\le R\). Then, for every \(t\ge0\),
\begin{equation}
    \mathbb E\norm{\boldsymbol\eta_{t+1}}^2
    \le
    C\epsilon^{-4}
    \left(
        L^2R^2+d\sigma^2
    \right)^2
    \vartheta_{2,t+1},
    \label{eq:app-adam-eta-second-moment}
\end{equation}
where \(C>0\) is a universal constant.
\end{lemma}

\begin{proof}[Proof of \cref{lem:app-adam-eta-second-moment}]
By the deterministic coordinatewise comparison in the proof of
\cref{lem:app-adam-eta-hp},
\[
    \norm{
        \boldsymbol P_{t+1}
        -
        \widetilde{\boldsymbol P}_{t+1}
    }_{\operatorname{op}}^2
    \le
    \epsilon^{-4}
    \norm{
        \widehat{\boldsymbol v}_{t+1}
        -
        \widetilde{\boldsymbol v}_{t+1}
    }_\infty.
\]
Since
\(\boldsymbol\eta_{t+1}
=(\boldsymbol P_{t+1}
-\widetilde{\boldsymbol P}_{t+1})
\widehat{\boldsymbol m}_{t+1}\),
Cauchy--Schwarz therefore gives
\begin{align}
    \mathbb E\norm{\boldsymbol\eta_{t+1}}^2
    &\le
    \epsilon^{-4}
    \mathbb E\left[
        \norm{\widehat{\boldsymbol m}_{t+1}}^2
        \norm{
            \widehat{\boldsymbol v}_{t+1}
            -
            \widetilde{\boldsymbol v}_{t+1}
        }_\infty
    \right]
    \notag\\
    &\le
    \epsilon^{-4}
    \left(
        \mathbb E
        \norm{\widehat{\boldsymbol m}_{t+1}}^4
    \right)^{1/2}
    \left(
        \mathbb E
        \norm{
            \widehat{\boldsymbol v}_{t+1}
            -
            \widetilde{\boldsymbol v}_{t+1}
        }_\infty^2
    \right)^{1/2}.
    \label{eq:app-adam-eta-expectation-cs}
\end{align}

Since
\(\norm{\bar{\boldsymbol g}_{t+1}(\boldsymbol\theta_t)}\le LR\)
and
\(\norm{\boldsymbol\xi_{t+1}}_{\Psi_2\mid\mathcal F_t}\le\sigma\),
standard sub-Gaussian moment bounds imply
\(\mathbb E[\xi_{t+1,i}^4\mid\mathcal F_t]\le C\sigma^4\)
for every \(i\in[d]\).
Using
\(\norm{\boldsymbol\xi_{t+1}}^4
\le d\sum_{i=1}^{d}\xi_{t+1,i}^4\)
and \((a+b)^4\le8a^4+8b^4\), we obtain
\[
    \mathbb E
    \norm{\boldsymbol\xi_{t+1}}^4
    \le
    C d^2\sigma^4,
    \;\;
    \mathbb E
    \norm{
        \nabla_{\boldsymbol\theta}
        g(\boldsymbol\theta_t,X_{t+1})
    }^4
    \le
    C\left(
        L^2R^2+d\sigma^2
    \right)^2.
\]
Because \(\widehat{\boldsymbol m}_{t+1}\) is a convex combination of
the past stochastic gradients, convexity gives
\(\mathbb E\norm{\widehat{\boldsymbol m}_{t+1}}^4
\le C(L^2R^2+d\sigma^2)^2\).
It remains to control the second-moment discrepancy. Recall that
\[
    \widehat{\boldsymbol v}_{t+1}
    -
    \widetilde{\boldsymbol v}_{t+1}
    =
    \frac{\beta_2^{t+1}}
    {1-\beta_2^{t+1}}
    \boldsymbol v_{t+1}
    +
    (1-\beta_2)\boldsymbol\chi_{t+1}.
\]
Since
\(\boldsymbol v_{t+1}
=(1-\beta_2)\sum_{j=0}^{t}\beta_2^{t-j}
(\nabla_{\boldsymbol\theta}
g(\boldsymbol\theta_j,X_{j+1}))^{\odot2}\),
the triangle inequality in \(L^2\) gives
\[
\begin{aligned}
    \left(
        \mathbb E
        \norm{\boldsymbol v_{t+1}}_\infty^2
    \right)^{1/2}
    &\le
    (1-\beta_2)
    \sum_{j=0}^{t}
    \beta_2^{t-j}
    \left(
        \mathbb E
        \norm{
            \nabla_{\boldsymbol\theta}
            g(\boldsymbol\theta_j,X_{j+1})
        }^4
    \right)^{1/2}\\
    &\le
    C(1-\beta_2^{t+1})
    \left(
        L^2R^2+d\sigma^2
    \right).
\end{aligned}
\]
Moreover,
\[
    \norm{\boldsymbol\chi_{t+1}}_\infty
    \le
    \left\|
        \nabla_{\boldsymbol\theta}
        g(\boldsymbol\theta_t,X_{t+1})
    \right\|^2
    +
    \mathbb E\left[
        \left.
        \left\|
            \nabla_{\boldsymbol\theta}
            g(\boldsymbol\theta_t,X_{t+1})
        \right\|^2
        \right|
        \mathcal F_t
    \right],
\]
and conditional Jensen's inequality together with the fourth-moment
bound above yields
\[
    \left(
        \mathbb E
        \norm{\boldsymbol\chi_{t+1}}_\infty^2
    \right)^{1/2}
    \le
    C\left(
        L^2R^2+d\sigma^2
    \right).
\]
Hence, by the triangle inequality in \(L^2\),
\begin{align}
    \left(
        \mathbb E
        \norm{
            \widehat{\boldsymbol v}_{t+1}
            -
            \widetilde{\boldsymbol v}_{t+1}
        }_\infty^2
    \right)^{1/2}
    &\le
    C\left(
        L^2R^2+d\sigma^2
    \right)
    \left[
        \beta_2^{t+1}
        +
        (1-\beta_2)
    \right]
    \notag\\
    &=
    C\left(
        L^2R^2+d\sigma^2
    \right)
    \vartheta_{2,t+1}.
    \label{eq:app-adam-vhat-vtilde-second-moment}
\end{align}
Substituting this estimate and the fourth-moment bound for
\(\widehat{\boldsymbol m}_{t+1}\) into
\cref{eq:app-adam-eta-expectation-cs} proves
\cref{eq:app-adam-eta-second-moment}.
\end{proof}

\vspace{0.5em}

\begin{theorem}[Expected tracking bound for Adam]
\label{thm:app-adam-expected-tracking}
Suppose Assumptions~\ref{assump:cond-subgauss-noise-restate} and
\ref{assumption:adaptive-strong-monotonicity-restated} hold under the
standing setup,
\(\operatorname{diam}(\Theta)\le R\), and the deterministic stepsize
satisfies
$0<\alpha
    \le
    \min\left\{
        \frac{\mu\epsilon^2}{4L^2},
        \frac{1}{\mu}
    \right\}$.
If \(\mathbb E[\Delta_t^2]\le\Delta^2\) for all \(t\ge0\), then,
for every \(t\ge1\),
\begin{align}
    \mathbb E\norm{
        \boldsymbol\theta_t-\boldsymbol\theta_t^\star
    }^2
    \le\;&
    \rho_\alpha^t
    \mathbb E\norm{
        \boldsymbol\theta_0-\boldsymbol\theta_0^\star
    }^2
    +
    \frac{10\Delta^2}{\alpha^2\mu^2}
    +
    \frac{160\beta_1^2L^2R^2}
    {\mu^2\epsilon^2}
    +
    \frac{40C d\sigma^2}
    {\mu^2\epsilon^2(1+\beta_1)}
    \left[
        1-\beta_1
        +
        \alpha\mu\beta_1t\lambda_1^{t-1}
    \right]
    \notag\\
    &+
    \frac{10C\alpha}{\mu\epsilon^4}
    \left(
        L^2R^2+d\sigma^2
    \right)^2
    \left[
        \beta_2t\lambda_2^{t-1}
        +
        \frac{2(1-\beta_2)}{\alpha\mu}
    \right],
    \label{eq:app-adam-expected-tracking}
\end{align}
where
\(\rho_\alpha:=1-\alpha\mu/2\),
\(\lambda_i:=\rho_\alpha\vee\beta_i\) for \(i\in\{1,2\}\), and
\(C>0\) is a sufficiently large universal constant.
\end{theorem}

\begin{proof}[Proof of \cref{thm:app-adam-expected-tracking}]
Iterating \cref{lem:app-adam-recursive-relation} and taking expectations gives
\begin{align}
    \mathbb E
    \|\boldsymbol\theta_t-\boldsymbol\theta_t^\star\|^2
    \le\;&
    \rho_\alpha^t
    \mathbb E
    \|\boldsymbol\theta_0-\boldsymbol\theta_0^\star\|^2
    +
    \frac{5}{\alpha\mu}
    \sum_{\ell=0}^{t-1}
    \rho_\alpha^{t-\ell-1}
    \mathbb E[\Delta_\ell^2]
    \notag\\
    &+
    \frac{10\alpha q_+^2}{\mu}
    \sum_{\ell=0}^{t-1}
    \rho_\alpha^{t-\ell-1}
    \mathbb E\|\boldsymbol r_{\ell+1}\|^2
    +
    \frac{10\alpha}{\mu}
    \sum_{\ell=0}^{t-1}
    \rho_\alpha^{t-\ell-1}
    \mathbb E\|\boldsymbol\eta_{\ell+1}\|^2.
    \label{eq:app-adam-expected-pathwise}
\end{align}
Since $1-\rho_\alpha=\alpha\mu/2$,
$\sum_{\ell=0}^{t-1}\rho_\alpha^{t-\ell-1}
\le 2/(\alpha\mu)$.
Thus, using $\mathbb E[\Delta_\ell^2]\le\Delta^2$, the drift term is
bounded by $10\Delta^2/(\alpha^2\mu^2)$.

By \cref{eq:app-adam-r-second-moment},
\[
\mathbb E\|\boldsymbol r_{\ell+1}\|^2
\le
8L^2R^2
\left(
    \frac{\beta_1(1-\beta_1^\ell)}
    {1-\beta_1^{\ell+1}}
\right)^2
+
2Cd\sigma^2\kappa_{1,\ell+1}.
\]
Since
$\beta_1(1-\beta_1^\ell)/(1-\beta_1^{\ell+1})\le\beta_1$,
\[
\sum_{\ell=0}^{t-1}
\rho_\alpha^{t-\ell-1}
\left(
    \frac{\beta_1(1-\beta_1^\ell)}
    {1-\beta_1^{\ell+1}}
\right)^2
\le
\frac{2\beta_1^2}{\alpha\mu}.
\]
Moreover,
\[
\kappa_{1,\ell+1}
=
\frac{1-\beta_1}{1+\beta_1}
+
\frac{
    2(1-\beta_1)\beta_1^{\ell+1}
}{
    (1+\beta_1)(1-\beta_1^{\ell+1})
}
\le
\frac{1-\beta_1}{1+\beta_1}
+
\frac{2\beta_1^{\ell+1}}{1+\beta_1},
\]
where $1-\beta_1^{\ell+1}\ge1-\beta_1$.
Let $\lambda_1:=\rho_\alpha\vee\beta_1$. Since
$\rho_\alpha^{t-\ell-1}\beta_1^{\ell+1}
\le\beta_1\lambda_1^{t-1}$,
\[
\begin{aligned}
    \sum_{\ell=0}^{t-1}
    \rho_\alpha^{t-\ell-1}\kappa_{1,\ell+1}
    &\le
    \frac{1-\beta_1}{1+\beta_1}
    \sum_{\ell=0}^{t-1}
    \rho_\alpha^{t-\ell-1}
    +
    \frac{2}{1+\beta_1}
    \sum_{\ell=0}^{t-1}
    \rho_\alpha^{t-\ell-1}\beta_1^{\ell+1}\\
    &\le
    \frac{2}{1+\beta_1}
    \left[
        \frac{1-\beta_1}{\alpha\mu}
        +
        \beta_1t\lambda_1^{t-1}
    \right].
\end{aligned}
\]
Recalling $q_+=\epsilon^{-1}$, it follows that
\[
\begin{aligned}
    \frac{10\alpha q_+^2}{\mu}
    \sum_{\ell=0}^{t-1}
    \rho_\alpha^{t-\ell-1}
    \mathbb E\|\boldsymbol r_{\ell+1}\|^2
    \le\;&
    \frac{160\beta_1^2L^2R^2}
    {\mu^2\epsilon^2}
    +
    \frac{40Cd\sigma^2}
    {\mu^2\epsilon^2(1+\beta_1)}
    \left[
        1-\beta_1
        +
        \alpha\mu\beta_1t\lambda_1^{t-1}
    \right].
\end{aligned}
\]

Similarly, \cref{lem:app-adam-eta-second-moment} gives
$\mathbb E\|\boldsymbol\eta_{\ell+1}\|^2
\le
C\epsilon^{-4}(L^2R^2+d\sigma^2)^2
\vartheta_{2,\ell+1}$.
Since
$\vartheta_{2,\ell+1}=(1-\beta_2)+\beta_2^{\ell+1}$, let
$\lambda_2:=\rho_\alpha\vee\beta_2$. Then
$\sum_{\ell=0}^{t-1}
\rho_\alpha^{t-\ell-1}\beta_2^{\ell+1}
\le\beta_2t\lambda_2^{t-1}$, and hence
\[
\sum_{\ell=0}^{t-1}
\rho_\alpha^{t-\ell-1}
\vartheta_{2,\ell+1}
\le
\frac{2(1-\beta_2)}{\alpha\mu}
+
\beta_2t\lambda_2^{t-1}.
\]
Therefore,
\[
\begin{aligned}
    \frac{10\alpha}{\mu}
    \sum_{\ell=0}^{t-1}
    \rho_\alpha^{t-\ell-1}
    \mathbb E\|\boldsymbol\eta_{\ell+1}\|^2
    \le
    \frac{10C\alpha}{\mu\epsilon^4}
    \left(
        L^2R^2+d\sigma^2
    \right)^2
    \left[
        \frac{2(1-\beta_2)}{\alpha\mu}
        +
        \beta_2t\lambda_2^{t-1}
    \right].
\end{aligned}
\]

Substituting these estimates into
\cref{eq:app-adam-expected-pathwise} proves
\cref{eq:app-adam-expected-tracking}.
\end{proof}

\section{Proofs for metric-projected stationarity guarantees under general Adam preconditioning}
\label{app:D}
For this section, we will make the following standard assumption in nonconvex optimization:

\vspace{0.5em}

\begin{assumption}[Uniform lower boundedness]
\label{assumption-appendix:lower-bounded}
There exists a constant $G^\star > -\infty$ such that for all 
$t \in [T]$ and all $\boldsymbol{\theta} \in \mathbb{R}^d$, $G_t(\boldsymbol{\theta}) \geq G^\star$.
\end{assumption}

For this section, $\mathcal P_\Theta$ denotes projection onto $\Theta$ in the $\widetilde{\boldsymbol P}_{t+1}^{-1}$-metric. For the projected Adam recursion, we replace the final line of
\plaineqref{eq:adam-update-restate} by
\[
    \boldsymbol\theta_{t+1}
    :=
    \mathcal P_\Theta
    \left(
        \boldsymbol\theta_t
        -
        \alpha\widetilde{\boldsymbol P}_{t+1}
        \nabla_{\boldsymbol\theta}G_{t+1}(\boldsymbol\theta_t)
        -
        \alpha
        \left(
            \widetilde{\boldsymbol P}_{t+1}\boldsymbol r_{t+1}
            +
            \boldsymbol\eta_{t+1}
        \right)
    \right).
\]
We also define the noiseless projected point by
\(\bar{\boldsymbol\theta}_{t+1}
:=
\mathcal P_\Theta(
\boldsymbol\theta_t
-
\alpha\widetilde{\boldsymbol P}_{t+1}
\nabla_{\boldsymbol\theta}G_{t+1}(\boldsymbol\theta_t))\), and the
preconditioned projected-gradient mapping by
\[
    \mathcal G_{\alpha,t}(\boldsymbol\theta_t)
    :=
    \frac{1}{\alpha}
    \left(
        \boldsymbol\theta_t-\mathcal P_\Theta(
\boldsymbol\theta_t
-
\alpha\widetilde{\boldsymbol P}_{t+1}
\nabla_{\boldsymbol\theta}G_{t+1}(\boldsymbol\theta_t))
    \right).
\]

The projected-gradient mapping is the constrained analogue of the gradient:
\(\mathcal G_{\alpha}^{\boldsymbol P}(\boldsymbol\theta)=\boldsymbol 0\)
if and only if \(\boldsymbol\theta\) satisfies the first-order constrained
stationarity condition
\(-\nabla f(\boldsymbol\theta)\in N_\Theta(\boldsymbol\theta)\) (see \cref{lem:projected-gradient-stationarity}). Thus, small
\(\|\mathcal G_{\alpha}^{\boldsymbol P}(\boldsymbol\theta)\|\) measures
approximate stationarity for the constrained problem, while in the
unconstrained case \(\Theta=\mathbb R^d\), it reduces to the usual
preconditioned gradient \(\boldsymbol P\nabla f(\boldsymbol\theta)\).

Before we proceed, we will need to prove a few technical results that we will find useful. The projected stationarity analysis requires slightly different control of
the preconditioner perturbation than the tracking analysis above. In
particular, the descent argument naturally involves the weighted quantity
\(\|\widetilde{\boldsymbol P}_{t+1}^{-1}\boldsymbol\eta_{t+1}\|
_{\widetilde{\boldsymbol P}_{t+1}}^2\). Controlling this quantity directly
avoids introducing a deterministic lower spectral bound on \(\widetilde{\boldsymbol P}_{t+1}\).

\vspace{0.5em}

\begin{lemma}[Weighted control of the preconditioner perturbation]
\label{lem:app-adam-eta-weighted}
Under the standing setup, suppose
Assumption~\ref{assump:cond-subgauss-noise-restate} holds and
\(\operatorname{diam}(\Theta)\le R\).
Then there exists a universal constant \(C>0\) such that, for
every \(T\ge1\) and \(\delta\in(0,1)\), with probability at least
\(1-\delta\), simultaneously for every
\(t\in\{0,\ldots,T-1\}\),
\begin{equation}
    \left\|
        \widetilde{\boldsymbol P}_{t+1}^{-1}
        \boldsymbol\eta_{t+1}
    \right\|_{\widetilde{\boldsymbol P}_{t+1}}^2
    \le
    C\epsilon^{-3}
    \left[
        L^2R^2
        +
        \sigma^2
        \left(
            d+\log\frac{4T}{\delta}
        \right)
    \right]^2
    \vartheta_{2,t+1}.
    \label{eq:app-adam-eta-weighted-hp}
\end{equation}
Moreover, for every \(t\ge0\),
\begin{equation}
    \mathbb E
    \left[
        \left\|
            \widetilde{\boldsymbol P}_{t+1}^{-1}
            \boldsymbol\eta_{t+1}
        \right\|_{\widetilde{\boldsymbol P}_{t+1}}^2
    \right]
    \le
    C\epsilon^{-3}
    \left(
        L^2R^2+d\sigma^2
    \right)^2
    \vartheta_{2,t+1}.
    \label{eq:app-adam-eta-weighted-expectation}
\end{equation}
\end{lemma}

\begin{proof}[Proof of \cref{lem:app-adam-eta-weighted}]
For each coordinate, write
\(P_{t+1,i}=(\sqrt{\widehat v_{t+1,i}}+\epsilon)^{-1}\) and
\(\widetilde P_{t+1,i}
=(\sqrt{\widetilde v_{t+1,i}}+\epsilon)^{-1}\).
For any \(a,b\ge0\),
\[
    \frac{
        \left[
            (\sqrt a+\epsilon)^{-1}
            -
            (\sqrt b+\epsilon)^{-1}
        \right]^2
    }{
        (\sqrt b+\epsilon)^{-1}
    }
    =
    \frac{
        (\sqrt a-\sqrt b)^2
    }{
        (\sqrt a+\epsilon)^2(\sqrt b+\epsilon)
    }
    \le
    \epsilon^{-3}|a-b|.
\]
Since
\(\boldsymbol\eta_{t+1}
=(\boldsymbol P_{t+1}-\widetilde{\boldsymbol P}_{t+1})
\widehat{\boldsymbol m}_{t+1}\),
applying this inequality coordinatewise gives
\begin{equation}
    \left\|
        \widetilde{\boldsymbol P}_{t+1}^{-1}
        \boldsymbol\eta_{t+1}
    \right\|_{\widetilde{\boldsymbol P}_{t+1}}^2
    \le
    \epsilon^{-3}
    \norm{\widehat{\boldsymbol m}_{t+1}}^2
    \norm{
        \widehat{\boldsymbol v}_{t+1}
        -
        \widetilde{\boldsymbol v}_{t+1}
    }_\infty.
    \label{eq:app-adam-eta-weighted-reduction}
\end{equation}

On the event of \cref{lem:app-adam-gradient-hp}, which has
probability at least \(1-\delta\), the estimates in the proof of
\cref{lem:app-adam-eta-hp} give, simultaneously for every
\(t\in\{0,\ldots,T-1\}\),
\begin{align*}
    &\norm{\widehat{\boldsymbol m}_{t+1}}^2
    \le
    C\left[
        L^2R^2
        +
        \sigma^2
        \left(
            d+\log\frac{4T}{\delta}
        \right)
    \right], \\
    &\norm{
        \widehat{\boldsymbol v}_{t+1}
        -
        \widetilde{\boldsymbol v}_{t+1}
    }_\infty
    \le
    C\left[
        L^2R^2
        +
        \sigma^2
        \left(
            d+\log\frac{4T}{\delta}
        \right)
    \right]
    \vartheta_{2,t+1}.
\end{align*}
Both estimates hold on the same gradient-concentration event.
Substituting them into
\cref{eq:app-adam-eta-weighted-reduction} proves
\cref{eq:app-adam-eta-weighted-hp}.

For the expectation bound, Cauchy--Schwarz applied to
\cref{eq:app-adam-eta-weighted-reduction} gives
\begin{align*}
    \mathbb E
    \left[
        \norm{\widehat{\boldsymbol m}_{t+1}}^2
        \norm{
            \widehat{\boldsymbol v}_{t+1}
            -
            \widetilde{\boldsymbol v}_{t+1}
        }_\infty
    \right]
    &\le
    \left(
        \mathbb E
        \norm{\widehat{\boldsymbol m}_{t+1}}^4
    \right)^{1/2}
    \left(
        \mathbb E
        \norm{
            \widehat{\boldsymbol v}_{t+1}
            -
            \widetilde{\boldsymbol v}_{t+1}
        }_\infty^2
    \right)^{1/2}\\
    &\le
    C
    \left(
        L^2R^2+d\sigma^2
    \right)^2
    \vartheta_{2,t+1},
\end{align*}
where the last inequality follows from the sub-Gaussian moment
estimates used in
\cref{lem:app-adam-eta-second-moment}.
Combining this with
\cref{eq:app-adam-eta-weighted-reduction} proves
\cref{eq:app-adam-eta-weighted-expectation}.
\end{proof}

With this weighted perturbation estimate in hand, we can now analyze
stationarity directly in the geometry induced by the predictable
Adam preconditioner. The next lemma bounds the average
\(\widetilde{\boldsymbol P}_{t+1}^{-1}\)-metric
projected-gradient mapping.

\vspace{0.5em}

\begin{lemma}[High-probability metric-projected gradient bound under general Adam preconditioning]
\label{thm:app-nonconvex-adam-hp}
Under the standing setup and
Assumptions~\ref{assumption:lipschitz-restate},
\ref{assumption-appendix:lower-bounded}, and
\ref{assump:cond-subgauss-noise-restate}, suppose
\(\operatorname{diam}(\Theta)\le R\).
Then, for all integers \(T\ge1\) and all
\(\delta\in(0,1)\), if \(0<\alpha\le\epsilon/(4L)\), the
iterates generated by the metric-projected Adam update satisfy,
with probability at least \(1-\delta\),
\begin{align}
    \frac{1}{T}\sum_{t=0}^{T-1}
    \left\|
        \mathcal G_{\alpha,t}(\boldsymbol\theta_t)
    \right\|_{\widetilde{\boldsymbol P}_{t+1}^{-1}}^2
    &\le
    \frac{
        8\bigl(
            G_1(\boldsymbol\theta_0)
            -
            G^\star
            +
            \mathfrak D_T
        \bigr)
    }{\alpha T}
    \notag
    \;+
    \frac{12C}{\epsilon^3T}
    \left[
        L^2R^2
        +
        \sigma^2
        \left(
            d+\log\frac{8T}{\delta}
        \right)
    \right]^2
    \sum_{t=0}^{T-1}
    \vartheta_{2,t+1}
    \\
    &+
    \frac{12}{\epsilon T}
    \sum_{t=0}^{T-1}
    \left(
        2LR\,
        \frac{\beta_1(1-\beta_1^t)}
        {1-\beta_1^{t+1}}
        +
        C\sigma
        \sqrt{
            \kappa_{1,t+1}
            \left(
                d+\log\frac{8T}{\delta}
            \right)
        }
    \right)^2.
    \label{eq:app-nonconvex-adam-hp}
\end{align}
Here
\(
\mathfrak D_T
:=
\sum_{t=0}^{T-2}
\bigl(
G_{t+2}(\boldsymbol\theta_{t+1})
-
G_{t+1}(\boldsymbol\theta_{t+1})
\bigr)_+
\)
is the pathwise objective-variation budget, with the empty sum
interpreted as zero when \(T=1\).
The matrix \(\widetilde{\boldsymbol P}_{t+1}\) is the predictable
Adam preconditioner, \(\kappa_{1,t}\) and
\(\vartheta_{2,t}\) are defined in
\plaineqref{eq:adam-constants}, and \(C>0\) is a sufficiently
large universal constant.
\end{lemma}

\begin{proof}[Proof of \cref{thm:app-nonconvex-adam-hp}]
Let
\(\boldsymbol s_{t+1}
:=
\boldsymbol\theta_{t+1}-\boldsymbol\theta_t\).
By the variational inequality for the
\(\widetilde{\boldsymbol P}_{t+1}^{-1}\)-metric projection
(\cref{lem:metric-projection-facts}), taking
\(\boldsymbol y=\boldsymbol\theta_t\) and applying
Cauchy--Schwarz and Young's inequality in the
\(\widetilde{\boldsymbol P}_{t+1}/
\widetilde{\boldsymbol P}_{t+1}^{-1}\)
dual norm pair gives
\begin{align}
    \left\langle
        \nabla_{\boldsymbol\theta}
        G_{t+1}(\boldsymbol\theta_t),
        \boldsymbol s_{t+1}
    \right\rangle
    \le
    -
    \frac{1}{2\alpha}
    \norm{\boldsymbol s_{t+1}}_{\widetilde{\boldsymbol P}_{t+1}^{-1}}^2
    +
    \frac{\alpha}{2}
    \left\|
        \boldsymbol r_{t+1}
        +
        \widetilde{\boldsymbol P}_{t+1}^{-1}
        \boldsymbol\eta_{t+1}
    \right\|_{\widetilde{\boldsymbol P}_{t+1}}^2.
    \label{eq:projected-inner-product}
\end{align}

Since \(G_{t+1}\) is \(L\)-smooth,
\cref{lem:descent-lemma-l-smooth} implies
\begin{align}
    G_{t+1}(\boldsymbol\theta_{t+1})
    \le\;&
    G_{t+1}(\boldsymbol\theta_t)
    -
    \frac{1}{2\alpha}
    \norm{\boldsymbol s_{t+1}}_{\widetilde{\boldsymbol P}_{t+1}^{-1}}^2
    +
    \frac{\alpha}{2}
    \left\|
        \boldsymbol r_{t+1}
        +
        \widetilde{\boldsymbol P}_{t+1}^{-1}
        \boldsymbol\eta_{t+1}
    \right\|_{\widetilde{\boldsymbol P}_{t+1}}^2
    +
    \frac{L}{2}
    \norm{\boldsymbol s_{t+1}}^2.
    \label{eq:projected-descent-start}
\end{align}
Because
\(\widetilde{\boldsymbol P}_{t+1}\preceq q_+\boldsymbol I\),
where \(q_+=\epsilon^{-1}\),
\(\norm{\boldsymbol s_{t+1}}^2
\le
q_+\norm{\boldsymbol s_{t+1}}
_{\widetilde{\boldsymbol P}_{t+1}^{-1}}^2\).
The stepsize condition
\(\alpha\le(4Lq_+)^{-1}=\epsilon/(4L)\) therefore gives
\begin{equation}
    G_{t+1}(\boldsymbol\theta_{t+1})
    \le
    G_{t+1}(\boldsymbol\theta_t)
    -
    \frac{1}{4\alpha}
    \norm{\boldsymbol s_{t+1}}_{\widetilde{\boldsymbol P}_{t+1}^{-1}}^2
    +
    \frac{\alpha}{2}
    \left\|
        \boldsymbol r_{t+1}
        +
        \widetilde{\boldsymbol P}_{t+1}^{-1}
        \boldsymbol\eta_{t+1}
    \right\|_{\widetilde{\boldsymbol P}_{t+1}}^2.
    \label{eq:projected-descent-step}
\end{equation}
Equivalently,
\begin{align}
    \norm{\boldsymbol s_{t+1}}_{\widetilde{\boldsymbol P}_{t+1}^{-1}}^2
    \le\;&
    4\alpha
    \left(
        G_{t+1}(\boldsymbol\theta_t)
        -
        G_{t+1}(\boldsymbol\theta_{t+1})
    \right)
    +
    2\alpha^2
    \left\|
        \boldsymbol r_{t+1}
        +
        \widetilde{\boldsymbol P}_{t+1}^{-1}
        \boldsymbol\eta_{t+1}
    \right\|_{\widetilde{\boldsymbol P}_{t+1}}^2.
    \label{eq:projected-step-bound}
\end{align}

We next compare the realized projected step with the noiseless
projected-gradient step. By nonexpansiveness of the
\(\widetilde{\boldsymbol P}_{t+1}^{-1}\)-metric projection,
\[
    \norm{
        \boldsymbol\theta_{t+1}
        -
        \bar{\boldsymbol\theta}_{t+1}
    }_{\widetilde{\boldsymbol P}_{t+1}^{-1}}
    \le
    \alpha
    \left\|
        \boldsymbol r_{t+1}
        +
        \widetilde{\boldsymbol P}_{t+1}^{-1}
        \boldsymbol\eta_{t+1}
    \right\|_{\widetilde{\boldsymbol P}_{t+1}}.
\]
Hence,
\begin{align}
    \left\|
        \mathcal G_{\alpha,t}(\boldsymbol\theta_t)
    \right\|_{\widetilde{\boldsymbol P}_{t+1}^{-1}}^2
    &\le
    \frac{2}{\alpha^2}
    \norm{\boldsymbol s_{t+1}}_{\widetilde{\boldsymbol P}_{t+1}^{-1}}^2
    +
    2
    \left\|
        \boldsymbol r_{t+1}
        +
        \widetilde{\boldsymbol P}_{t+1}^{-1}
        \boldsymbol\eta_{t+1}
    \right\|_{\widetilde{\boldsymbol P}_{t+1}}^2.
    \label{eq:gradient-mapping-step}
\end{align}
Combining
\cref{eq:projected-step-bound,eq:gradient-mapping-step} gives
\begin{align}
    \left\|
        \mathcal G_{\alpha,t}(\boldsymbol\theta_t)
    \right\|_{\widetilde{\boldsymbol P}_{t+1}^{-1}}^2
    \le\;&
    \frac{8}{\alpha}
    \left(
        G_{t+1}(\boldsymbol\theta_t)
        -
        G_{t+1}(\boldsymbol\theta_{t+1})
    \right)
    +
    6
    \left\|
        \boldsymbol r_{t+1}
        +
        \widetilde{\boldsymbol P}_{t+1}^{-1}
        \boldsymbol\eta_{t+1}
    \right\|_{\widetilde{\boldsymbol P}_{t+1}}^2.
    \label{eq:projected-one-step}
\end{align}

We now sum the one-step descent inequality and account for the
temporal variation of the objective. By telescoping,
\[
\begin{aligned}
    \sum_{t=0}^{T-1}
    \left(
        G_{t+1}(\boldsymbol\theta_t)
        -
        G_{t+1}(\boldsymbol\theta_{t+1})
    \right)
    &=
    G_1(\boldsymbol\theta_0)
    -
    G_T(\boldsymbol\theta_T)
    +
    \sum_{t=0}^{T-2}
    \left(
        G_{t+2}(\boldsymbol\theta_{t+1})
        -
        G_{t+1}(\boldsymbol\theta_{t+1})
    \right)\\
    &\;\;\le
    G_1(\boldsymbol\theta_0)-G^\star+\mathfrak D_T,
\end{aligned}
\]
where the last inequality uses
\(G_T(\boldsymbol\theta_T)\ge G^\star\) and the definition of
\(\mathfrak D_T\).
Summing \cref{eq:projected-one-step} over
\(t=0,\ldots,T-1\) and dividing by \(T\) therefore gives
\begin{align}
    \frac{1}{T}\sum_{t=0}^{T-1}
    \left\|
        \mathcal G_{\alpha,t}(\boldsymbol\theta_t)
    \right\|_{\widetilde{\boldsymbol P}_{t+1}^{-1}}^2
    \le\;&
    \frac{
        8\bigl(
            G_1(\boldsymbol\theta_0)
            -
            G^\star
            +
            \mathfrak D_T
        \bigr)
    }{\alpha T}
    +
    \frac{6}{T}
    \sum_{t=0}^{T-1}
    \left\|
        \boldsymbol r_{t+1}
        +
        \widetilde{\boldsymbol P}_{t+1}^{-1}
        \boldsymbol\eta_{t+1}
    \right\|_{\widetilde{\boldsymbol P}_{t+1}}^2.
    \label{eq:projected-pathwise-bound}
\end{align}

Finally,
\[
    \left\|
        \boldsymbol r_{t+1}
        +
        \widetilde{\boldsymbol P}_{t+1}^{-1}
        \boldsymbol\eta_{t+1}
    \right\|_{\widetilde{\boldsymbol P}_{t+1}}^2
    \le
    2q_+\norm{\boldsymbol r_{t+1}}^2
    +
    2
    \left\|
        \widetilde{\boldsymbol P}_{t+1}^{-1}
        \boldsymbol\eta_{t+1}
    \right\|_{\widetilde{\boldsymbol P}_{t+1}}^2.
\]
Applying
\cref{lem:app-adam-r-hp,lem:app-adam-eta-weighted}
with failure probability \(\delta/2\) each and taking a union
bound therefore gives, with probability at least \(1-\delta\),
simultaneously for all \(t=0,\ldots,T-1\),
\[
    \norm{\boldsymbol r_{t+1}}
    \le
    2LR\,
    \frac{\beta_1(1-\beta_1^t)}
    {1-\beta_1^{t+1}}
    +
    C\sigma
    \sqrt{
        \kappa_{1,t+1}
        \left(
            d+\log\frac{8T}{\delta}
        \right)
    }
\]
and
\[
    \left\|
        \widetilde{\boldsymbol P}_{t+1}^{-1}
        \boldsymbol\eta_{t+1}
    \right\|_{\widetilde{\boldsymbol P}_{t+1}}^2
    \le
    C\epsilon^{-3}
    \left[
        L^2R^2
        +
        \sigma^2
        \left(
            d+\log\frac{8T}{\delta}
        \right)
    \right]^2
    \vartheta_{2,t+1}.
\]
Substituting these two estimates into
\cref{eq:projected-pathwise-bound} and recalling
\(q_+=\epsilon^{-1}\) proves
\cref{eq:app-nonconvex-adam-hp}.
\end{proof}

The average bound immediately gives a corresponding
randomized-iterate statement. This is useful when the optimization
algorithm returns an iterate sampled uniformly from its trajectory
rather than the entire sequence.

\vspace{0.5em}

\begin{remark}[Returned-iterate interpretation]
\label{rem:returned-iterate-gradient-mapping}
Let \(\tau\sim\mathrm{Unif}\{0,\ldots,T-1\}\) be sampled
independently of \(\mathcal F_T\).
Then, on the same high-probability event as
\cref{thm:app-nonconvex-adam-hp},
\[
    \mathbb E_\tau
    \left[
        \left\|
            \mathcal G_{\alpha,\tau}(\boldsymbol\theta_\tau)
        \right\|_{\widetilde{\boldsymbol P}_{\tau+1}^{-1}}^2
    \right]
    =
    \frac{1}{T}
    \sum_{t=0}^{T-1}
    \left\|
        \mathcal G_{\alpha,t}(\boldsymbol\theta_t)
    \right\|_{\widetilde{\boldsymbol P}_{t+1}^{-1}}^2.
\]
Here \(\mathbb E_\tau\) averages only over the independent index
\(\tau\), with \(\mathcal F_T\) fixed.
Thus \cref{thm:app-nonconvex-adam-hp} also provides a
returned-iterate guarantee in the predictable preconditioner
geometry. Moreover, if the right-hand side of
\cref{eq:app-nonconvex-adam-hp} is denoted by
\(B_T(\delta)\), then for every \(\eta\in(0,1)\), Markov's
inequality gives, on the same event,
\[
    \mathbb P_\tau\left(
        \left\|
            \mathcal G_{\alpha,\tau}(\boldsymbol\theta_\tau)
        \right\|_{\widetilde{\boldsymbol P}_{\tau+1}^{-1}}^2
        >
        \frac{B_T(\delta)}{\eta}
    \right)
    \le
    \eta,
\]
where \(\mathbb P_\tau\) also refers only to the independent
choice of \(\tau\), with \(\mathcal F_T\) fixed.
Consequently, the corresponding unconditional failure probability
is at most \(\delta+\eta\).
\end{remark}

\subsection{Proof of \cref{cor:nonconvex-adam-hp}}
\label{app:D2}

The bound in \cref{thm:app-nonconvex-adam-hp} retains the exact
bias-correction factors at each iteration. We now average these
quantities explicitly to separate the finite-horizon contributions
from the persistent terms in the metric-residual upper bound.

\vspace{0.5em}

\begin{theorem}[Explicit high-probability projected-gradient rate under projected Adam]
\label{cor:nonconvex-adam-hp-app}
Under the conditions of \cref{thm:app-nonconvex-adam-hp}, with
probability at least \(1-\delta\),
\begin{equation}
    \frac{1}{T}
    \sum_{t=0}^{T-1}
    \left\|
        \mathcal G_{\alpha,t}(\boldsymbol\theta_t)
    \right\|_{\widetilde{\boldsymbol P}_{t+1}^{-1}}^2
    \lesssim
    \frac{
        G_1(\boldsymbol\theta_0)-G^\star+\mathfrak D_T
    }{\alpha T}
    +
    \frac{\mathsf{Decay}_T(\delta)}{T}
    +
    \mathsf{Floor}_T(\delta),
\end{equation}
where
\begin{align*}
    \mathsf{Decay}_T(\delta)
    &\sim \mathcal{O} \left\{
    \frac{
        \beta_2
    }{
        \epsilon^3(1-\beta_2)
    }
    \left[
        L^2R^2
        +
        \sigma^2
        \left(
            d+\log\frac{T}{\delta}
        \right)
    \right]^2
    +
    \frac{
        \sigma^2
        \left(
            d+\log\frac{T}{\delta}
        \right)
        (1+\log T)
    }{
        \epsilon(1+\beta_1)
    } \right\}, \\
    \mathsf{Floor}_T(\delta)
    &\sim \mathcal{O} \left \{ 
    \frac{
        L^2R^2\beta_1^2
    }{
        \epsilon
    }
    +
    \frac{
        \sigma^2(1-\beta_1)
    }{
        \epsilon(1+\beta_1)
    }
    \left(
        d+\log\frac{T}{\delta}
    \right)
    +
    \frac{
        1-\beta_2
    }{
        \epsilon^3
    }
    \left[
        L^2R^2
        +
        \sigma^2
        \left(
            d+\log\frac{T}{\delta}
        \right)
    \right]^2 \right \}.
\end{align*}
\end{theorem}

\begin{proof}[Proof of \cref{cor:nonconvex-adam-hp-app}]
By \cref{thm:app-nonconvex-adam-hp} and
$(a+b)^2\le2a^2+2b^2$, the first-moment term is bounded by
\[
\frac{96L^2R^2}{\epsilon T}
\sum_{t=1}^{T}
\left(
    \frac{\beta_1(1-\beta_1^{t-1})}
    {1-\beta_1^t}
\right)^2
+
\frac{24C^2\sigma^2}{\epsilon T}
\left(
    d+\log\frac{8T}{\delta}
\right)
\sum_{t=1}^{T}\kappa_{1,t}.
\]
Since $1-\beta_1^{t-1}\le1-\beta_1^t$,
$\beta_1(1-\beta_1^{t-1})/(1-\beta_1^t)\le\beta_1$, and hence
\[
\frac{1}{T}\sum_{t=1}^{T}
\left(
    \frac{\beta_1(1-\beta_1^{t-1})}
    {1-\beta_1^t}
\right)^2
\le\beta_1^2.
\]

We next bound the average of $\kappa_{1,t}$. For $0<\beta<1$,
\[
1-\beta^t
=
(1-\beta)\sum_{j=0}^{t-1}\beta^j
\ge
t(1-\beta)\beta^{t-1},
\]
so $\beta^t/(1-\beta^t)\le \beta/[t(1-\beta)]
\le1/[t(1-\beta)]$; the same bound is trivial for $\beta=0$.
Since
\[
\kappa_{1,t}
=
\frac{1-\beta_1}{1+\beta_1}
\left(
    1+\frac{2\beta_1^t}{1-\beta_1^t}
\right),
\]
we obtain
$\kappa_{1,t}\le(1-\beta_1)/(1+\beta_1)
+2/[(1+\beta_1)t]$. Therefore,
\[
\frac{1}{T}\sum_{t=1}^{T}\kappa_{1,t}
\le
\frac{1-\beta_1}{1+\beta_1}
+
\frac{2(1+\log T)}{(1+\beta_1)T},
\]
where $\sum_{t=1}^{T}t^{-1}\le1+\log T$. Similarly, using
$\vartheta_{2,t}=(1-\beta_2)+\beta_2^t$ and the geometric-series
identity,
\[
\frac{1}{T}\sum_{t=1}^{T}\vartheta_{2,t}
=
(1-\beta_2)
+
\frac{\beta_2(1-\beta_2^T)}
{T(1-\beta_2)}
\le
(1-\beta_2)
+
\frac{\beta_2}{(1-\beta_2)T}.
\]

Substituting these estimates into
\cref{thm:app-nonconvex-adam-hp} yields
\begin{align*}
    \frac{1}{T}\sum_{t=0}^{T-1}
    \left\|
        \mathcal G_{\alpha,t}(\boldsymbol\theta_t)
    \right\|_{\widetilde{\boldsymbol P}_{t+1}^{-1}}^2
    \lesssim\;&
    \frac{
        G_1(\boldsymbol\theta_0)-G^\star+\mathfrak D_T
    }{\alpha T}
    +
    \frac{L^2R^2\beta_1^2}{\epsilon}
    +
    \frac{\sigma^2}{\epsilon}
    \left(
        d+\log\frac{8T}{\delta}
    \right)
    \left[
        \frac{1-\beta_1}{1+\beta_1}
        +
        \frac{1+\log T}{(1+\beta_1)T}
    \right]
    \\
    &+
    \frac{1}{\epsilon^3}
    \left[
        L^2R^2
        +
        \sigma^2
        \left(
            d+\log\frac{8T}{\delta}
        \right)
    \right]^2
    \left[
        (1-\beta_2)
        +
        \frac{\beta_2}{(1-\beta_2)T}
    \right].
\end{align*}

Thus we find that \begin{align*} \mathsf{Decay}_T(\delta) &:= \mathcal{O} \left\{ \frac{ \beta_2 }{ \epsilon^3(1-\beta_2) } \left[ L^2R^2 + \sigma^2 \left( d+\log\frac{T}{\delta} \right) \right]^2 + \frac{ \sigma^2 \left( d+\log\frac{T}{\delta} \right) (1+\log T) }{ \epsilon(1+\beta_1) } \right\},\\ \mathsf{Floor}_T(\delta) &:= \mathcal{O} \left\{ \frac{ L^2R^2\beta_1^2 }{ \epsilon } + \frac{ \sigma^2(1-\beta_1) }{ \epsilon(1+\beta_1) } \left( d+\log\frac{T}{\delta} \right) + \frac{ 1-\beta_2 }{ \epsilon^3 } \left[ L^2R^2 + \sigma^2 \left( d+\log\frac{T}{\delta} \right) \right]^2 \right\}. \end{align*} This proves the claim.
\end{proof}

\subsection{Proof of expected projected stationarity bounds}
\label{app:D4}

For completeness, we also give the expectation analogue of
\cref{thm:app-nonconvex-adam-hp}. The proof follows the same pathwise
descent argument, replacing the uniform high-probability controls of the
first-moment and preconditioner errors by their corresponding second-moment
bounds. As in \cref{cor:nonconvex-adam-hp-app}, the sums over
\(c_{1,t}\), \(\kappa_{1,t}\), and \(\vartheta_{2,t}\) can be evaluated
explicitly to obtain a decay--floor decomposition; we omit this additional
step for brevity.

\vspace{0.5em}

\begin{theorem}[Expected projected-gradient stationarity gap under projected Adam]
\label{thm:nonconvex-adam-expectation}
Under Assumptions~\ref{assumption:lipschitz-restate},
\ref{assumption-appendix:lower-bounded}, and
\ref{assump:cond-subgauss-noise-restate}, suppose
\(\operatorname{diam}(\Theta)\le R\). Then, for all integers \(T\ge1\),
if \(\alpha\le\epsilon/(4L)\), the iterates generated by the projected Adam
update satisfy
\begin{align}
    \frac{1}{T}\sum_{t=0}^{T-1}
    \mathbb E\left[
        \left\|
            \mathcal G_{\alpha,t}(\boldsymbol\theta_t)
        \right\|_{\widetilde{\boldsymbol P}_{t+1}^{-1}}^2
    \right]
    &\le
    \frac{
        8\bigl(
            G_1(\boldsymbol\theta_0)
            -
            G^\star
            +
            \mathbb E[\mathfrak D_T]
        \bigr)
    }{\alpha T}
    \notag
    +
    \frac{12C}{\epsilon^3T}
    \left(
        L^2R^2+d\sigma^2
    \right)^2
    \sum_{t=0}^{T-1}
    \vartheta_{2,t+1}
    \notag\\
    &\;+
    \frac{12}{\epsilon T}
    \sum_{t=0}^{T-1}
    \left(
        18L^2R^2\left( \frac{\beta_1(1-\beta_1^t)}
    {1-\beta_1^{t+1}}\right)^2
        +
        2Cd\sigma^2\kappa_{1,t+1}
    \right).
    \label{eq:nonconvex-adam-expectation}
\end{align}
Here
\(
\mathfrak D_T
:=
\sum_{t=0}^{T-2}
\bigl(
G_{t+2}(\boldsymbol\theta_{t+1})
-
G_{t+1}(\boldsymbol\theta_{t+1})
\bigr)_+
\)
is the pathwise objective-variation budget,
\(\widetilde{\boldsymbol P}_{t+1}\) is the predictable Adam
preconditioner, and \(\kappa_{1,t}\), and \(\vartheta_{2,t}\) are defined in \plaineqref{eq:adam-constants}.
\end{theorem}

\begin{proof}[Proof of \cref{thm:nonconvex-adam-expectation}]
From the proof of \cref{thm:app-nonconvex-adam-hp}, before applying the
high-probability error bounds, we have the pathwise inequality
\begin{align}
    \frac{1}{T}\sum_{t=0}^{T-1}
    \left\|
        \mathcal G_{\alpha,t}(\boldsymbol\theta_t)
    \right\|_{\widetilde{\boldsymbol P}_{t+1}^{-1}}^2
    \le\;&
    \frac{
        8\bigl(
            G_1(\boldsymbol\theta_0)
            -
            G^\star
            +
            \mathfrak D_T
        \bigr)
    }{\alpha T}
    +
    \frac{12q_+}{T}
    \sum_{t=0}^{T-1}
    \norm{\boldsymbol r_{t+1}}^2
    +
    \frac{12}{T}
    \sum_{t=0}^{T-1}
    \left\|
        \widetilde{\boldsymbol P}_{t+1}^{-1}
        \boldsymbol\eta_{t+1}
    \right\|_{\widetilde{\boldsymbol P}_{t+1}}^2.
    \label{eq:app-nonconvex-expectation-pathwise}
\end{align}

Taking expectations and applying
\cref{lem:app-adam-r-second-moment} gives
\(
\mathbb E\norm{\boldsymbol r_{t+1}}^2
\le
18L^2R^2\left( \frac{\beta_1(1-\beta_1^t)}
    {1-\beta_1^{t+1}}\right)^2
+
2Cd\sigma^2\kappa_{1,t+1}
\).
Similarly, the expectation part of
\cref{lem:app-adam-eta-weighted} gives
\[
    \mathbb E\left[
        \left\|
            \widetilde{\boldsymbol P}_{t+1}^{-1}
            \boldsymbol\eta_{t+1}
        \right\|_{\widetilde{\boldsymbol P}_{t+1}}^2
    \right]
    \le
    C\epsilon^{-3}
    \left(
        L^2R^2+d\sigma^2
    \right)^2
    \vartheta_{2,t+1}.
\]
Substituting these two estimates into
\cref{eq:app-nonconvex-expectation-pathwise} and using
\(q_+=\epsilon^{-1}\) yields
\cref{eq:nonconvex-adam-expectation}.
\end{proof}

\section{Technical lemmas}
\label{app:E}

\subsection{Metric projection and projected-gradient mappings}
\label{app:E1}
We collect a few standard facts about projected-gradient mappings used in
\cref{thm:app-nonconvex-adam-hp}. Let \(\Theta\subset\mathbb R^d\) be
nonempty, closed, and convex. For a positive definite matrix
\(\boldsymbol A\), define the \(\boldsymbol A\)-metric projection by
\[
\mathcal P_\Theta^{\boldsymbol A}(\boldsymbol z)
:=
\argmin_{\boldsymbol\theta\in\Theta}
\frac12
\|\boldsymbol\theta-\boldsymbol z\|_{\boldsymbol A}^2.
\]
When the metric is clear from context, we suppress the superscript and write
\(\mathcal P_\Theta\). For \(\boldsymbol\theta\in\Theta\), we write
\[
    N_\Theta(\boldsymbol\theta)
    :=
    \left\{
        \boldsymbol v\in\mathbb R^d:
        \langle \boldsymbol v,\boldsymbol y-\boldsymbol\theta\rangle
        \le 0
        \;\;\text{for all }\boldsymbol y\in\Theta
    \right\}
\]
for the normal cone of \(\Theta\) at \(\boldsymbol\theta\). Thus
\(-\nabla f(\boldsymbol\theta)\in N_\Theta(\boldsymbol\theta)\) is equivalent
to $\left\langle
        \nabla f(\boldsymbol\theta),
        \boldsymbol y-\boldsymbol\theta
    \right\rangle
    \ge 0$ 
for all $\boldsymbol y\in\Theta$, which says that no feasible first-order direction decreases \(f\). When
\(\boldsymbol\theta\) lies in the interior of \(\Theta\), this reduces to the usual unconstrained condition \(\nabla f(\boldsymbol\theta)=\boldsymbol 0\).

\vspace{0.5em}

\begin{lemma}[Metric projection facts]
\label{lem:metric-projection-facts}
Let \(\Theta\subset\mathbb R^d\) be nonempty, closed, and convex, and let
\(\boldsymbol A\succ 0\). Then, for every \(\boldsymbol z\in\mathbb R^d\),
\(\mathcal P_\Theta^{\boldsymbol A}(\boldsymbol z)\) exists and is unique.
Moreover, it satisfies the variational inequality
\[
    \left\langle
        \boldsymbol A
        \left(
            \mathcal P_\Theta^{\boldsymbol A}(\boldsymbol z)-\boldsymbol z
        \right),
        \boldsymbol y-\mathcal P_\Theta^{\boldsymbol A}(\boldsymbol z)
    \right\rangle
    \ge 0,
    \;\;
    \forall \boldsymbol y\in\Theta.
\]
In addition, the projection is nonexpansive in the
\(\boldsymbol A\)-metric:
\[
    \left\|
        \mathcal P_\Theta^{\boldsymbol A}(\boldsymbol z)
        -
        \mathcal P_\Theta^{\boldsymbol A}(\boldsymbol z')
    \right\|_{\boldsymbol A}
    \le
    \|\boldsymbol z-\boldsymbol z'\|_{\boldsymbol A},
    \;\;
    \forall \boldsymbol z,\boldsymbol z'\in\mathbb R^d .
\]
\end{lemma}

\begin{proof}
The claim follows by viewing \(\mathbb R^d\) as a Hilbert space with inner
product \(\langle \boldsymbol x,\boldsymbol y\rangle_{\boldsymbol A}
:=\boldsymbol x^\top\boldsymbol A\boldsymbol y\). Since \(\Theta\) is closed
and convex, the Hilbert-space projection theorem gives existence,
uniqueness, and nonexpansiveness of the projection. The variational
inequality is the first-order optimality condition for the strongly convex
problem defining \(\mathcal P_\Theta^{\boldsymbol A}(\boldsymbol z)\).
\end{proof}

For the Adam analysis, we take
\(\boldsymbol A=\widetilde{\boldsymbol P}_{t+1}^{-1}\). Thus
\(\mathcal P_\Theta\) denotes projection onto \(\Theta\) in the
\(\widetilde{\boldsymbol P}_{t+1}^{-1}\)-metric. Given a differentiable
function \(f\), a positive definite preconditioner \(\boldsymbol P\), and a
stepsize \(\alpha>0\), define the preconditioned projected-gradient mapping
\[
    \mathcal G_{\alpha}^{\boldsymbol P}(\boldsymbol\theta)
    :=
    \frac{1}{\alpha}
    \left[
        \boldsymbol\theta
        -
        \mathcal P_\Theta^{\boldsymbol P^{-1}}
        \left(
            \boldsymbol\theta
            -
            \alpha\boldsymbol P\nabla f(\boldsymbol\theta)
        \right)
    \right].
\]
In \cref{thm:app-nonconvex-adam-hp}, this becomes
\(\mathcal G_{\alpha,t}(\boldsymbol\theta_t)\) with
\(\boldsymbol P=\widetilde{\boldsymbol P}_{t+1}\) and
\(f=G_{t+1}\).

\vspace{0.5em}

\begin{lemma}[Projected-gradient mapping and constrained stationarity]
\label{lem:projected-gradient-stationarity}
Let \(\Theta\subset\mathbb R^d\) be nonempty, closed, and convex, let
\(f:\mathbb R^d\to\mathbb R\) be differentiable, let
\(\boldsymbol P\succ 0\), and let \(\alpha>0\). Then
\[
    \mathcal G_{\alpha}^{\boldsymbol P}(\boldsymbol\theta)=\boldsymbol 0
    \Longleftrightarrow
    -\nabla f(\boldsymbol\theta)\in N_\Theta(\boldsymbol\theta).
\]
In particular, \(\mathcal G_{\alpha}^{\boldsymbol P}(\boldsymbol\theta)=0\)
is exactly the first-order stationarity condition for the constrained
problem \(\min_{\boldsymbol\theta\in\Theta} f(\boldsymbol\theta)\). Moreover,
when \(\Theta=\mathbb R^d\),
\[
    \mathcal G_{\alpha}^{\boldsymbol P}(\boldsymbol\theta)
    =
    \boldsymbol P\nabla f(\boldsymbol\theta),
    \;\;
    \left\|
        \mathcal G_{\alpha}^{\boldsymbol P}(\boldsymbol\theta)
    \right\|_{\boldsymbol P^{-1}}^2
    =
    \|\nabla f(\boldsymbol\theta)\|_{\boldsymbol P}^2 .
\]
\end{lemma}

\begin{proof}
Let
\[
    \bar{\boldsymbol\theta}
    :=
    \mathcal P_\Theta^{\boldsymbol P^{-1}}
    \left(
        \boldsymbol\theta-\alpha\boldsymbol P\nabla f(\boldsymbol\theta)
    \right).
\]
If \(\mathcal G_{\alpha}^{\boldsymbol P}(\boldsymbol\theta)=\boldsymbol 0\),
then \(\bar{\boldsymbol\theta}=\boldsymbol\theta\). Applying
\cref{lem:metric-projection-facts} with
\(\boldsymbol A=\boldsymbol P^{-1}\) and
\(\boldsymbol z=\boldsymbol\theta-\alpha\boldsymbol P\nabla f(\boldsymbol\theta)\)
gives, for every \(\boldsymbol y\in\Theta\),
\[
    \left\langle
        \boldsymbol P^{-1}
        \left(
            \boldsymbol\theta
            -
            \boldsymbol z
        \right),
        \boldsymbol y-\boldsymbol\theta
    \right\rangle
    =
    \alpha
    \left\langle
        \nabla f(\boldsymbol\theta),
        \boldsymbol y-\boldsymbol\theta
    \right\rangle
    \ge 0.
\]
Equivalently,
\(\langle-\nabla f(\boldsymbol\theta),\boldsymbol y-\boldsymbol\theta\rangle
\le 0\) for all \(\boldsymbol y\in\Theta\), which is precisely
\(-\nabla f(\boldsymbol\theta)\in N_\Theta(\boldsymbol\theta)\).

Conversely, if
\(-\nabla f(\boldsymbol\theta)\in N_\Theta(\boldsymbol\theta)\), then
\(\langle\nabla f(\boldsymbol\theta),
\boldsymbol y-\boldsymbol\theta\rangle\ge 0\) for every
\(\boldsymbol y\in\Theta\). Hence
\[
    \left\langle
        \boldsymbol P^{-1}
        \left(
            \boldsymbol\theta
            -
            \left[
                \boldsymbol\theta
                -
                \alpha\boldsymbol P\nabla f(\boldsymbol\theta)
            \right]
        \right),
        \boldsymbol y-\boldsymbol\theta
    \right\rangle
    \ge 0,
    \;\;
    \forall \boldsymbol y\in\Theta.
\]
By the variational characterization of the metric projection,
\(\boldsymbol\theta=
\mathcal P_\Theta^{\boldsymbol P^{-1}}(
\boldsymbol\theta-\alpha\boldsymbol P\nabla f(\boldsymbol\theta))\), and
therefore
\(\mathcal G_{\alpha}^{\boldsymbol P}(\boldsymbol\theta)=\boldsymbol 0\).
The unconstrained identities follow immediately from
\(\mathcal P_{\mathbb R^d}^{\boldsymbol P^{-1}}\) being the identity map.
\end{proof}

\subsection{Martingale concentration inequalities}
\label{app:E2}

\begin{theorem}[Martingale subgaussian Hoeffding inequality; cf.\
{\cite[Theorem~2.7.3]{Vershynin_2018}}]
    \label{thm:martingale-subgaussian-hoeffding}
    Let \((\mathcal{F}_t)_{t=0}^T\) be a filtration and let
    \(Z_1,\ldots,Z_T\) be an adapted martingale difference sequence, so that
    \(\mathbb{E}[Z_t\mid\mathcal{F}_{t-1}]=0\) almost surely. Suppose that
    \(\norm{Z_t\mid\mathcal{F}_{t-1}}_{\psi_2}\le K_t\) almost surely for
    deterministic constants \(K_t\ge 0\). Then, for every \(s\ge 0\),
    \[
        \mathbb{P}\left\{
            \sum_{t=1}^T Z_t \ge s
        \right\}
        \le
        \exp\left(
            -\frac{c s^2}{\sum_{t=1}^T K_t^2}
        \right),
    \]
    where \(c>0\) is an absolute constant.
\end{theorem}

\begin{proof}
    Let \(S_T:=\sum_{t=1}^T Z_t\). We use the exponential moment method.
    For any \(\lambda>0\), Markov's inequality gives
    \[
        \mathbb{P}\{S_T\ge s\}
        \le
        e^{-\lambda s}\mathbb{E}e^{\lambda S_T}.
    \]
    Since \(S_{T-1}\) is \(\mathcal{F}_{T-1}\)-measurable, the tower property
    gives
    \[
        \mathbb{E}e^{\lambda S_T}
        =
        \mathbb{E}\left[
            e^{\lambda S_{T-1}}
            \mathbb{E}\left[
                e^{\lambda Z_T}\mid\mathcal{F}_{T-1}
            \right]
        \right].
    \]
    By the standard subgaussian MGF bound in conditional form,
    \(\mathbb{E}[e^{\lambda Z_t}\mid\mathcal{F}_{t-1}]
    \le \exp(C\lambda^2K_t^2)\) almost surely for every
    \(\lambda\in\mathbb{R}\). Iterating the tower-property argument therefore
    yields
    \[
        \mathbb{E}e^{\lambda S_T}
        \le
        \exp\left(
            C\lambda^2\sum_{t=1}^T K_t^2
        \right).
    \]
    Hence,
    \[
        \mathbb{P}\{S_T\ge s\}
        \le
        \exp\left(
            -\lambda s
            +C\lambda^2\sum_{t=1}^T K_t^2
        \right).
    \]
    Optimizing in \(\lambda\) by taking
    \(\lambda=s/(2C\sum_{t=1}^T K_t^2)\) gives the claimed bound after
    adjusting the absolute constant.
\end{proof}

\vspace{0.5em}

\begin{theorem}[Martingale subexponential Bernstein inequality; cf.\
{\cite[Theorem~2.9.1]{Vershynin_2018}}]
    \label{thm:martingale-subexponential-bernstein}
    Let \((\mathcal{F}_t)_{t=0}^T\) be a filtration and let
    \(Z_1,\ldots,Z_T\) be an adapted martingale difference sequence, so that
    \(\mathbb{E}[Z_t\mid\mathcal{F}_{t-1}]=0\) almost surely. Suppose that
    \(\norm{Z_t\mid\mathcal{F}_{t-1}}_{\psi_1}\le K_t\) almost surely for
    deterministic constants \(K_t\ge 0\). Then, for every \(s\ge 0\),
    \[
        \mathbb{P}\left\{
            \sum_{t=1}^T Z_t \ge s
        \right\}
        \le
        \exp\left[
            -c\min\left\{
                \frac{s^2}{\sum_{t=1}^T K_t^2},
                \frac{s}{\max_{1\le t\le T}K_t}
            \right\}
        \right],
    \]
    where \(c>0\) is an absolute constant.
\end{theorem}

\begin{proof}
    Again let \(S_T:=\sum_{t=1}^T Z_t\). For any \(\lambda>0\), the
    exponential moment method gives
    \[
        \mathbb{P}\{S_T\ge s\}
        \le
        e^{-\lambda s}\mathbb{E}e^{\lambda S_T}.
    \]
    The conditional subexponential MGF bound implies that there exist
    absolute constants \(c,C>0\) such that
    \[
        \mathbb{E}\left[
            e^{\lambda Z_t}\mid\mathcal{F}_{t-1}
        \right]
        \le
        \exp(C\lambda^2K_t^2)
    \]
    almost surely whenever \(\lambda\le c/K_t\). Therefore, using the tower
    property recursively, for
    \[
        0\le\lambda
        \le
        \frac{c}{\max_{1\le t\le T}K_t},
    \]
    we obtain
    \[
        \mathbb{E}e^{\lambda S_T}
        \le
        \exp\left(
            C\lambda^2\sum_{t=1}^T K_t^2
        \right).
    \]
    Consequently,
    \[
        \mathbb{P}\{S_T\ge s\}
        \le
        \exp\left(
            -\lambda s
            +C\lambda^2\sum_{t=1}^T K_t^2
        \right).
    \]
    Writing \(\sigma^2:=\sum_{t=1}^T K_t^2\), we minimize the exponent
    subject to the preceding constraint. Taking
    \[
        \lambda
        :=
        \min\left\{
            \frac{s}{2C\sigma^2},
            \frac{c}{\max_{1\le t\le T}K_t}
        \right\}
    \]
    yields
    \[
        \mathbb{P}\{S_T\ge s\}
        \le
        \exp\left[
            -c'
            \min\left\{
                \frac{s^2}{\sum_{t=1}^T K_t^2},
                \frac{s}{\max_{1\le t\le T}K_t}
            \right\}
        \right]
    \]
    for an absolute constant \(c'>0\), which proves the claim.
\end{proof}

\subsection{\(L\)-smoothness and descent inequalities}
\label{app:E3}

\begin{lemma}[Descent lemma for $L$-smooth functions; \cite{nesterov2018lectures}]
\label{lem:descent-lemma-l-smooth}
Let $f:\mathbb{R}^n \to \mathbb{R}$ be differentiable and $L$-smooth, meaning that $$\|\nabla f(y)-\nabla f(x)\| \le L\|y-x\| \;\; \text{for all } x,y \in \mathbb{R}^n.$$  Then, for any $x,y \in \mathbb{R}^n$, we have
\[
\left|
f(y) - f(x) - \left\langle \nabla f(x), y - x \right\rangle
\right|
\leq \frac{L}{2}\|y-x\|^2 .
\]
\end{lemma}

\section{Additional experimental details and results}
\label{app:F}
We provide additional details for the numerical experiments in Appendix~\ref{app:F3}. For each problem and regime, we tune the learning rate over a small fixed grid and choose the run with the best learning rate predictive performance.

\subsection{General setup}
\label{app:F1}

\paragraph{Optimizers and tuning.}
We compare vanilla SGD and Adam. For SGD, we tune the learning rate over
\(\{10^{-4},3\cdot 10^{-4}, 10^{-3},3\cdot 10^{-3},10^{-2},3\cdot 10^{-2}\}\). For Adam, we tune the learning rate over the same grid as SGD with
\(\beta_1=0.9\), \(\beta_2=0.999\), and \(\epsilon=10^{-8}\). For each
optimizer, problem, and regime, we select the learning rate that minimizes
the average metric over the second half of the trajectory. Curves are then
reported for the selected learning rate, averaged over seeds
\(\{0,1,2, 3, 4, 5, 6, 7, 8, 9\}\), with shaded regions denoting one standard error of the mean.

\paragraph{Drift and noise schedules.}
For vector-valued targets, the drift is generated through normalized
Gaussian directions:
\[
    \boldsymbol{\theta}_{t+1}^\star
    =
    \boldsymbol{\theta}_t^\star
    +
    \Delta_t
    \frac{\boldsymbol u_t}{\|\boldsymbol u_t\|},
    \;\;
    \boldsymbol u_t\sim\mathcal N(\boldsymbol 0,\boldsymbol I_d).
\]
Thus \(\Delta_t\) sets the drift magnitude, while \(\sigma_t\) sets the
stochastic observation noise. In the two-layer MLP experiment, the same idea
is implemented in function space: the teacher parameters are perturbed in a
random normalized parameter direction and rescaled so that the induced change
in teacher predictions has magnitude controlled by \(\Delta_t\).

\begin{table}[H]

\caption{Drift and noise schedules used in the main comparison figures.}

\label{tab:experiment-schedules}
\centering
\begin{tabular}{llcc}
\toprule
Problem & Regime & \(\Delta_t\) & \(\sigma_t\) \\
\midrule
Least squares
& High drift / low noise
& \(\log(t+2)\)
& \(5\cdot 10^{-4}\log(t+2)\) \\
& Low drift / high noise
& \(5\cdot 10^{-4}\log(t+2)\)
& \(10\log(t+2)\) \\
\midrule
Logistic regression
& High drift / low noise
& \(100\)
& \(1\) \\
& Low drift / high noise
& \(10^{-4}\)
& \(10\) \\
\midrule
Teacher--student MLP
& High drift / low noise
& \(5\log(t+2)\)
& \(10^{-3}\log(t+2)\) \\
& Low drift / high noise
& \(5\cdot10^{-4}\log(t+2)\)
& \(100\log(t+2)\) \\
\midrule
Phase retrieval
& High drift / low noise
& \(10\log(t+2)\)
& \(10^{-2}\log(t+2)\) \\
& Low drift / high noise
& \(10^{-4}\log(t+2)\)
& \(10\log(t+2)\) \\
\midrule
Matrix factorization
& High drift / low noise
& \(0.1\)
& \(0.01\) \\
& Low drift / high noise
& \(5\cdot10^{-4}\)
& \(1.0\) \\
\midrule
Lasso
& Higher drift / low noise
& \(2\cdot10^{-3}\log(t+2)\)
& \(0.1\log(t+2)\) \\
& Lower drift / high noise
& \(10^{-3}\log(t+2)\)
& \(10\log(t+2)\) \\
\bottomrule
\end{tabular}
\end{table}

\paragraph{Evaluation.}
For least squares, logistic regression, and Lasso, we run \(T=800\) iterations and evaluate every \(20\) iterations. For teacher--student MLP
and phase retrieval, we run \(T=500\) iterations and evaluate every \(20\) iterations, while for matrix factorization we run \(T=500\) iterations and evaluate every \(5\) iterations. We use mini-batches of size \(1\) for least squares, \(32\) for Lasso, and \(256\) for teacher--student MLP and phase retrieval.

For least squares and Lasso, we report the parameter tracking error
\(\|\boldsymbol{\theta}_t-\boldsymbol{\theta}_t^\star\|^2\).
For logistic regression, teacher--student MLP, phase retrieval, and
matrix factorization, we report the stationarity gap
\(\|\nabla F_t(\boldsymbol{\theta}_t)\|^2\), where for matrix
factorization the gradient is taken jointly with respect to both factor matrices.

\subsection{Task-specific details}
\label{app:F2}

\paragraph{Strongly convex least squares.}
We use dimension \(d=50\) and \(n=100\) observations. The population
objective is
\[
    F_t(\boldsymbol{\theta})
    =
    \frac12
    \|\boldsymbol A(\boldsymbol{\theta}-\boldsymbol{\theta}_t^\star)\|^2,
\]
where \(\boldsymbol A^\top \boldsymbol A\) has eigenvalues logarithmically
spaced between \(\mu=1\) and \(L=10\). Stochastic observations are generated as
\[
    \boldsymbol y_t
    =
    \boldsymbol A\boldsymbol{\theta}_t^\star
    +
    \boldsymbol\varepsilon_t,
\]
where the observation noise is scaled by \(\sigma_t/\sqrt{nL}\). We report
the tracking error
\(\|\boldsymbol{\theta}_t-\boldsymbol{\theta}_t^\star\|^2\), since the
moving optimum is identifiable in parameter space.

\paragraph{Teacher--student MLP regression.}
We use a two-layer ReLU network
\[
    f_{\boldsymbol{\theta}}(\boldsymbol x)
    =
    \boldsymbol W_2
    \mathrm{ReLU}(\boldsymbol W_1\boldsymbol x+\boldsymbol b_1)
    +
    b_2,
\]
with input dimension \(100\) and hidden width \(128\). Teacher and student
networks are initialized with weights drawn at scale \(0.04\), and the
student is warm-started from the teacher at \(t=0\). At each step,
\(\boldsymbol x\sim\mathcal N(\boldsymbol 0,\boldsymbol I)\) and labels are generated as
\[
    y_t
    =
    f_{\boldsymbol{\theta}_t^\star}(\boldsymbol x)
    +
    \sigma_t\varepsilon_t,
    \;\;
    \varepsilon_t\sim\mathcal N(0,1).
\]
We evaluate on a fixed validation set of size \(1024\) and report stationarity gap \(\|\nabla F_t(\boldsymbol{\theta}_t)\|^2\).

\paragraph{Phase retrieval.}
We use dimension \(d=50\). Given
\(\boldsymbol x\sim\mathcal N(\boldsymbol 0,\boldsymbol I_d)\), the
time-varying teacher produces
\[
    y_t
    =
    (\boldsymbol x^\top\boldsymbol w_t^\star)^2
    +
    \sigma_t\varepsilon_t,
    \;\;
    \varepsilon_t\sim\mathcal N(0,1).
\]
The corresponding population objective is
\[
    F_t(\boldsymbol{\theta})
    =
    \frac12
    \mathbb E_{\boldsymbol x}
    \left[
        \left(
            (\boldsymbol x^\top\boldsymbol{\theta})^2
            -
            (\boldsymbol x^\top\boldsymbol w_t^\star)^2
        \right)^2
    \right].
\]
The target \(\boldsymbol w_t^\star\) evolves by normalized Gaussian drift,
and the model is warm-started at the initial target. We use batch size
\(256\), a validation set of size \(1024\), and report stationarity gap \(\|\nabla F_t(\boldsymbol{\theta}_t)\|^2\).

\paragraph{Matrix factorization.}
We use a rank-\(5\) matrix factorization problem with \(m=n=60\). We write
the time-varying target matrix as
\[
    \boldsymbol M_t^\star
    =
    \boldsymbol P_t^\star
    \boldsymbol\Sigma
    (\boldsymbol Q_t^\star)^\top,
    \;\;
    \boldsymbol U_t^\star
    =
    \boldsymbol P_t^\star\boldsymbol\Sigma^{1/2},
    \;\;
    \boldsymbol V_t^\star
    =
    \boldsymbol Q_t^\star\boldsymbol\Sigma^{1/2},
\]
where \(\boldsymbol P_t^\star,\boldsymbol Q_t^\star\in\mathbb R^{60\times 5}\)
have orthonormal columns and \(\boldsymbol\Sigma\) is fixed. The model predicts
\(\boldsymbol U\boldsymbol V^\top\), with population objective
\(F_t(\boldsymbol U,\boldsymbol V)
=
(2mn)^{-1}
\|\boldsymbol U\boldsymbol V^\top-\boldsymbol M_t^\star\|_F^2\).
We normalize the initial target so that
\(\|\boldsymbol M_0^\star\|_F/\sqrt{mn}=1\).

To generate nonstationarity while preserving the rank and singular values of
the target, we draw fixed skew-symmetric matrices satisfying
\(\boldsymbol S_U^\top=-\boldsymbol S_U\) and
\(\boldsymbol S_V^\top=-\boldsymbol S_V\), and define the corresponding
orthogonal rotations
\[
    \boldsymbol R_U
    =
    \exp(\omega\boldsymbol S_U),
    \;\;
    \boldsymbol R_V
    =
    \exp(\omega\boldsymbol S_V).
\]
The target subspaces then evolve according to
\[
    \boldsymbol P_{t+1}^\star
    =
    \boldsymbol R_U\boldsymbol P_t^\star,
    \;\;
    \boldsymbol Q_{t+1}^\star
    =
    \boldsymbol R_V\boldsymbol Q_t^\star.
\]
Since matrix exponentials of skew-symmetric matrices are orthogonal,
\(\boldsymbol P_t^\star\) and \(\boldsymbol Q_t^\star\) retain orthonormal
columns, so the rank and singular values of \(\boldsymbol M_t^\star\) remain
fixed. Moreover,
\[
    \boldsymbol M_t^\star
    =
    \boldsymbol R_U^t
    \boldsymbol M_0^\star
    (\boldsymbol R_V^t)^\top,
\]
and hence, by orthogonal invariance of the Frobenius norm,
\[
    \|\boldsymbol M_{t+1}^\star-\boldsymbol M_t^\star\|_F
    =
    \|
        \boldsymbol R_U\boldsymbol M_0^\star\boldsymbol R_V^\top
        -
        \boldsymbol M_0^\star
    \|_F.
\]
Thus the per-step drift
\(\Delta_t
=
\|\boldsymbol M_{t+1}^\star-\boldsymbol M_t^\star\|_F/\sqrt{mn}\)
is constant over time, and we choose \(\omega\) to attain the desired
drift level.

We inject noise directly into the gradient oracle. Writing
\(\boldsymbol E_t
=
\boldsymbol U_t\boldsymbol V_t^\top-\boldsymbol M_t^\star\), we use
\[
    \widehat{\nabla}_{\boldsymbol U}F_t
    =
    \frac{1}{mn}
    (\boldsymbol E_t+\sigma_t\boldsymbol Z_t)\boldsymbol V_t,
    \;\;
    \widehat{\nabla}_{\boldsymbol V}F_t
    =
    \frac{1}{mn}
    (\boldsymbol E_t+\sigma_t\boldsymbol Z_t)^\top\boldsymbol U_t,
\]
where \(\boldsymbol Z_t\) is independent and mean zero with
\(\|\boldsymbol Z_t\|_F/\sqrt{mn}=1\). Thus \(\Delta_t\) and \(\sigma_t\)
measure target drift and gradient noise on the same per-entry RMS scale.
We report the stationarity gap
\(\|\nabla_{\boldsymbol U}F_t\|_F^2
+\|\nabla_{\boldsymbol V}F_t\|_F^2\).

\subsection{Additional experimental examples}
\label{app:F3}

We include logistic regression and lasso regression as additional examples to
test whether the same qualitative noise--drift tradeoff appears beyond the
main comparison tasks.

\paragraph{Logistic regression.}
We consider convex rank-deficient logistic regression with dimension
\(d=100\), rank \(r=20\), and \(n=1000\) samples. Covariates are generated as
\(\boldsymbol A=\boldsymbol Z\boldsymbol U^\top\), where
\(\boldsymbol U\in\mathbb R^{d\times r}\) has orthonormal columns and
\(\boldsymbol Z\in\mathbb R^{n\times r}\) has i.i.d.\ Gaussian entries scaled
by \(1/\sqrt r\). Let \(\varsigma(z)=(1+e^{-z})^{-1}\) denote the sigmoid
function and define
\(\boldsymbol p_t=\varsigma(\boldsymbol A\boldsymbol\theta_t^\star)\).
The population objective is
\[
    F_t(\boldsymbol\theta)
    =
    \frac{1}{n}
    \sum_{i=1}^{n}
    \left[
        \log\!\left(1+\exp(\boldsymbol a_i^\top\boldsymbol\theta)\right)
        -
        p_{t,i}\boldsymbol a_i^\top\boldsymbol\theta
    \right],
\]
so that \(\nabla F_t(\boldsymbol\theta_t^\star)=\boldsymbol 0\).

Since
\(\boldsymbol A\boldsymbol\theta
=\boldsymbol Z\boldsymbol U^\top\boldsymbol\theta\), the objective depends on
\(\boldsymbol\theta\) only through its active coordinates
\(\boldsymbol U^\top\boldsymbol\theta\). In particular,
\[
    \nabla F_t(\boldsymbol\theta)
    =
    \boldsymbol U
    \left[
        \frac{1}{n}\boldsymbol Z^\top
        \left(
            \varsigma(\boldsymbol Z\boldsymbol U^\top\boldsymbol\theta)
            -
            \boldsymbol p_t
        \right)
    \right]
    \in \operatorname{range}(\boldsymbol U),
\]
and the objective is flat on
\(\operatorname{range}(\boldsymbol U)^\perp\), with
\(\operatorname{rank}(\nabla^2F_t)\le r\). We therefore restrict stochastic-gradient noise to the same identifiable
subspace and use
\[
    \boldsymbol g_{t+1}(\boldsymbol\theta)
    =
    \nabla F_t(\boldsymbol\theta)
    +
    \frac{\sigma_t}{\sqrt r}
    \boldsymbol U\boldsymbol\varepsilon_{t+1},
    \;\;
    \boldsymbol\varepsilon_{t+1}
    \sim\mathcal N(\boldsymbol 0,\boldsymbol I_r).
\]
This oracle is conditionally unbiased, with covariance
\((\sigma_t^2/r)\boldsymbol U\boldsymbol U^\top\) and
\(\mathbb E\|\boldsymbol g_{t+1}(\boldsymbol\theta)
-\nabla F_t(\boldsymbol\theta)\|^2=\sigma_t^2\).
Restricting the perturbation to \(\operatorname{range}(\boldsymbol U)\)
avoids injecting noise into the \(d-r\) null-space directions, where the loss
provides no restoring gradient and isotropic noise would induce an unidentifiable random walk.

The target likewise drifts only within the active subspace according to
\(\boldsymbol\theta_{t+1}^\star
=\boldsymbol\theta_t^\star+\Delta_t\boldsymbol U\boldsymbol u_t\), where
\(\boldsymbol u_t=\boldsymbol h_t/\|\boldsymbol h_t\|_2\) and
\(\boldsymbol h_t\sim\mathcal N(\boldsymbol 0,\boldsymbol I_r)\).
Thus \(\Delta_t\) controls target drift while \(\sigma_t\) independently
controls stochastic-gradient noise. We report the stationarity gap
\(\|\nabla F_t(\boldsymbol\theta_t)\|^2\).

\paragraph{Lasso regression.}
We consider Lasso regression with dimension \(d=100\), sample size
\(n=200\), sparsity level \(s=10\), and regularization parameter
\(\lambda=0.02\). The design matrix
\(\boldsymbol X\in\mathbb R^{n\times d}\) has independent Gaussian entries
scaled by \(1/\sqrt d\). The initial target \(\boldsymbol\beta_0^\star\) is
supported on a randomly selected set \(S\subset[d]\) with \(|S|=s\) and is
normalized so that \(\|\boldsymbol\beta_0^\star\|_2=3\). The support \(S\)
remains fixed throughout the experiment.

The population objective at time \(t\) is
\[
    F_t(\boldsymbol\theta)
    =
    \frac{1}{2n}
    \|\boldsymbol X\boldsymbol\theta
    -
    \boldsymbol X\boldsymbol\beta_t^\star\|_2^2
    +
    \lambda\|\boldsymbol\theta\|_1 .
\]
At each stochastic step, we sample a mini-batch of \(32\) rows from
\(\boldsymbol X\) and observe
\(y_{i,t}
=
\boldsymbol x_i^\top\boldsymbol\beta_t^\star
+
\sigma_t\varepsilon_{i,t}\), where
\(\varepsilon_{i,t}\) are independent standard Gaussian random variables. Thus the observation noise is mean zero and sub-Gaussian.

The target evolves only on its fixed support \(S\). Writing
\(\boldsymbol\beta_{t,S}^\star\) for its active coordinates, we take a
normalized random step of magnitude \(\Delta_t\) along the sphere
\(\|\boldsymbol\beta_{t,S}^\star\|_2=3\), so that
\[
    \operatorname{supp}(\boldsymbol\beta_t^\star)=S,
    \;\;
    \|\boldsymbol\beta_t^\star\|_2=3
\]
for all \(t\). Hence the target remains \(s\)-sparse and in a fixed compact
set throughout the experiment. We warm-start the model at
\(\boldsymbol\beta_0^\star\) and report the parameter tracking error
\(\|\boldsymbol\theta_t-\boldsymbol\beta_t^\star\|_2^2\).

\begin{figure}[!hbtp]
    \centering
    \includegraphics[width=\textwidth]{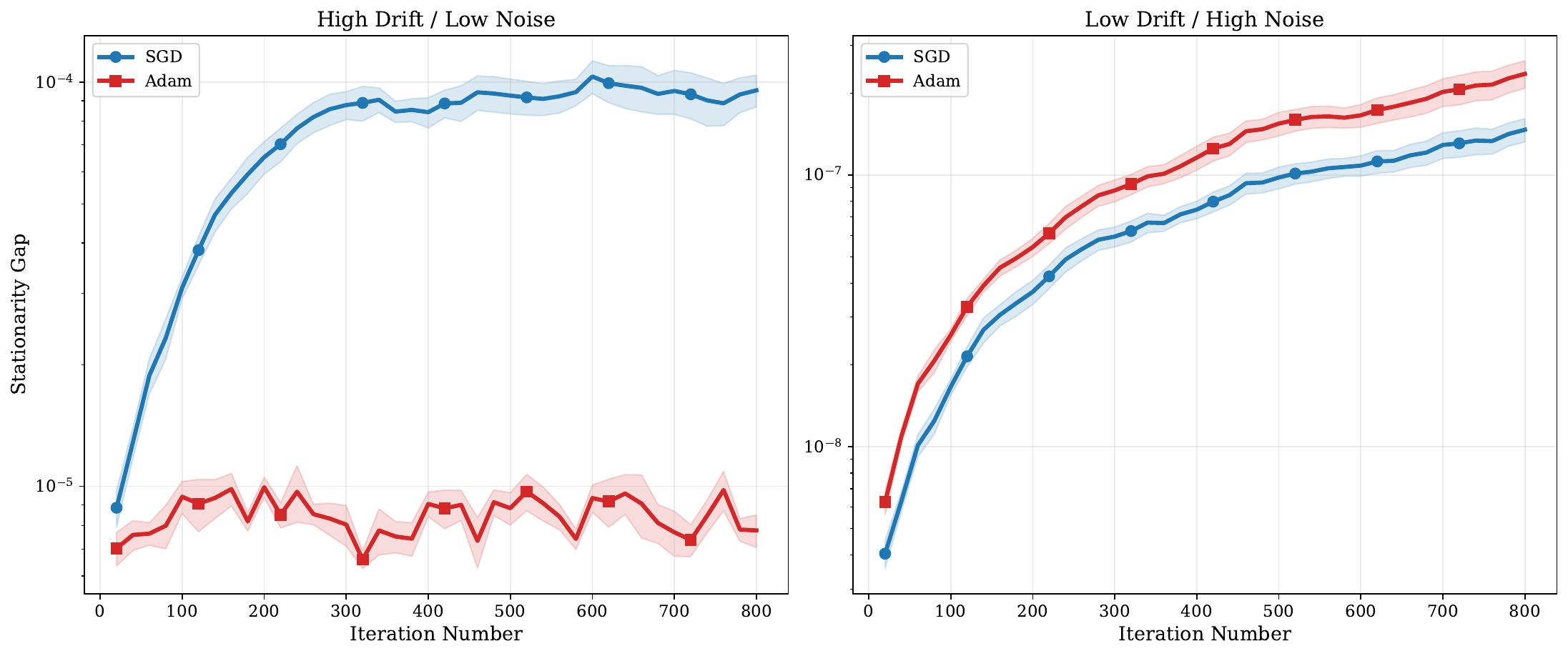}
    \caption{Logistic regression, where we report the stationarity gap.
\textbf{Left}: high-drift, low-noise regime with
\(\Delta_t=c_2\) and \(\sigma_t=c_1\).
\textbf{Right}: low-drift, high-noise regime with
\(\Delta_t=c_3\) and \(\sigma_t=c_4\),
where \(0<c_1<c_2\) and \(0<c_3<c_4\).
    The same qualitative behavior appears: SGD performs better in the lower-noise regime, whereas Adam improves relative performance in the substantially noisier regime on the right.}
    \label{fig:diagram_logistic}
\end{figure}

\begin{figure}[H]
    \centering
    \includegraphics[width=\textwidth]{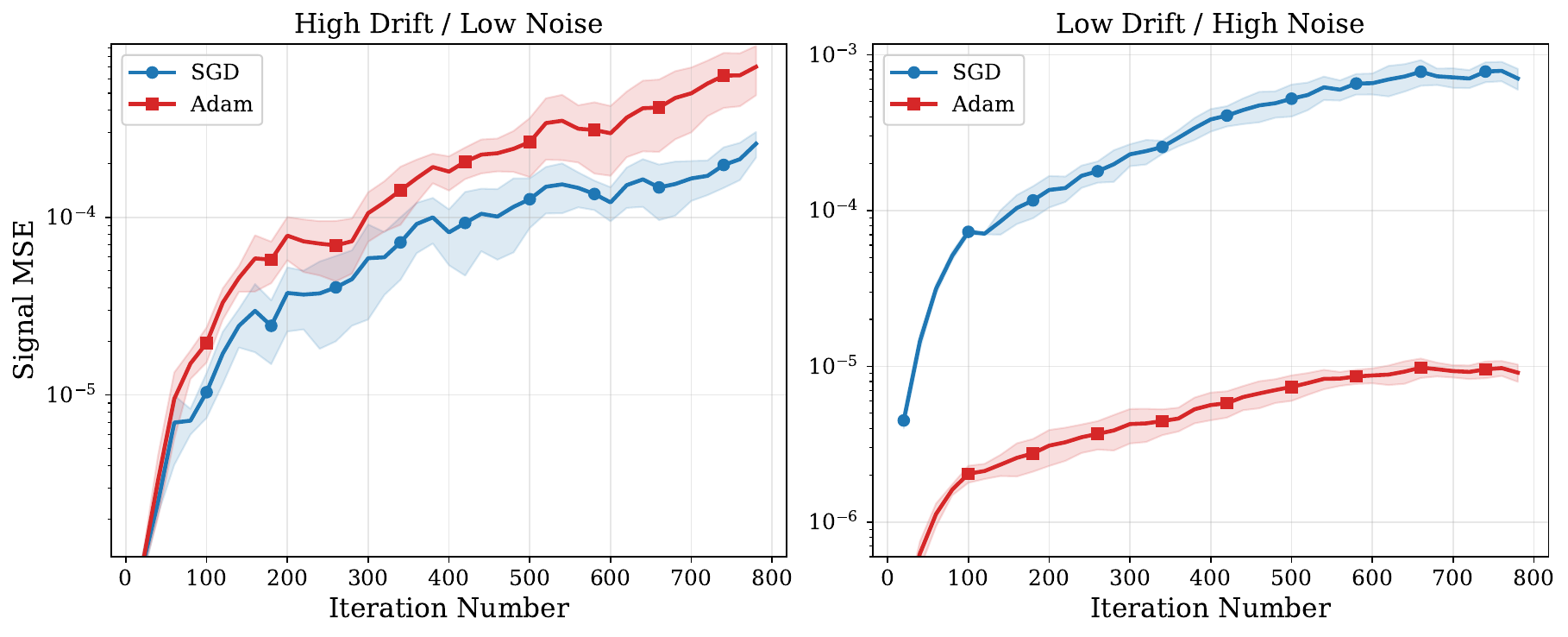}
    \caption{Lasso regression, where we report prediction MSE.
\textbf{Left}: higher-drift, lower-noise regime with
\(\Delta_t=c_2\log(t+2)\) and \(\sigma_t=c_3\log(t+2)\).
\textbf{Right}: lower-drift, higher-noise regime with
\(\Delta_t=c_1\log(t+2)\) and \(\sigma_t=c_4\log(t+2)\),
where \(0<c_1<c_2\) and \(0<c_3<c_4\).
A similar qualitative pattern is observed here as well: the relatively drift-dominated regime favors SGD, while the noise-dominated regime favors Adam.}
    \label{fig:diagram_lasso}
\end{figure}

\subsection{Adam hyperparameter dependence}
\label{app:F4}
We next isolate how Adam's hyperparameters affect the noise--drift tradeoff.
The first set of experiments uses the online quadratic tracking problem $F_t(\boldsymbol{\theta})
    =
    \frac{\mu}{2}
    \|\boldsymbol{\theta}-\boldsymbol{\theta}_t^\star\|^2$.
The second uses phase retrieval and reports
\(\|\nabla F_t(\boldsymbol{\theta}_t)\|^2\).

\begin{figure}[H]
    \centering
    \includegraphics[width=\textwidth]{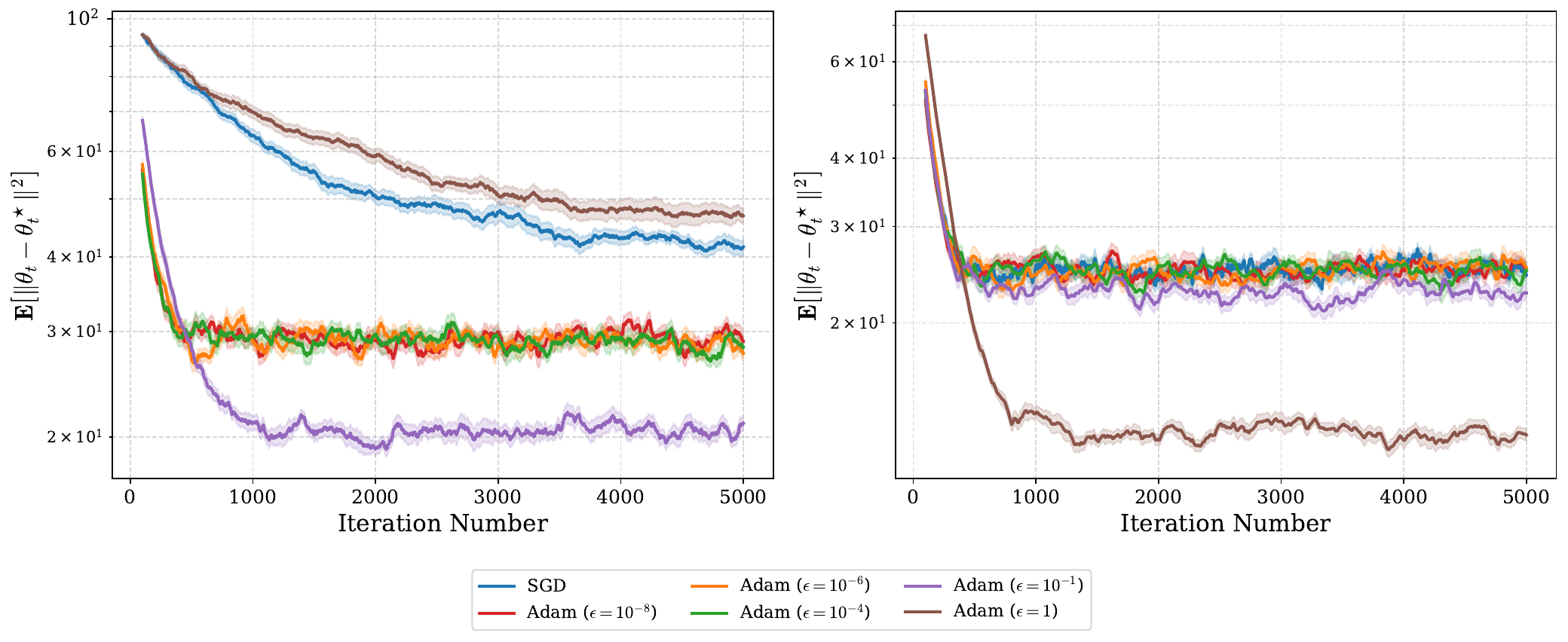}
    \caption{
Dependence on the Adam stabilization parameter $\epsilon$ for the online quadratic tracking problem.
The \textbf{left panel} corresponds to high drift and low noise ($\Delta=0.2$, $\sigma^2=0.01$), while the \textbf{right panel} corresponds to low drift and high noise ($\Delta=0.01$, $\sigma^2=1$).
Across both panels, we fix $\beta_1=0.9$, $\beta_2=0.999$, $\mu=0.01$, and $d=100$.
Consistent with our theory, moderate stabilization performs best under high drift, while larger $\epsilon$ improves tracking in the noise-dominated regime by suppressing stochastic adaptive-preconditioning effects.
}
    \label{fig:diagram_eps_comparison}
\end{figure}

\begin{figure}[H]
    \centering
    \includegraphics[width=\textwidth]{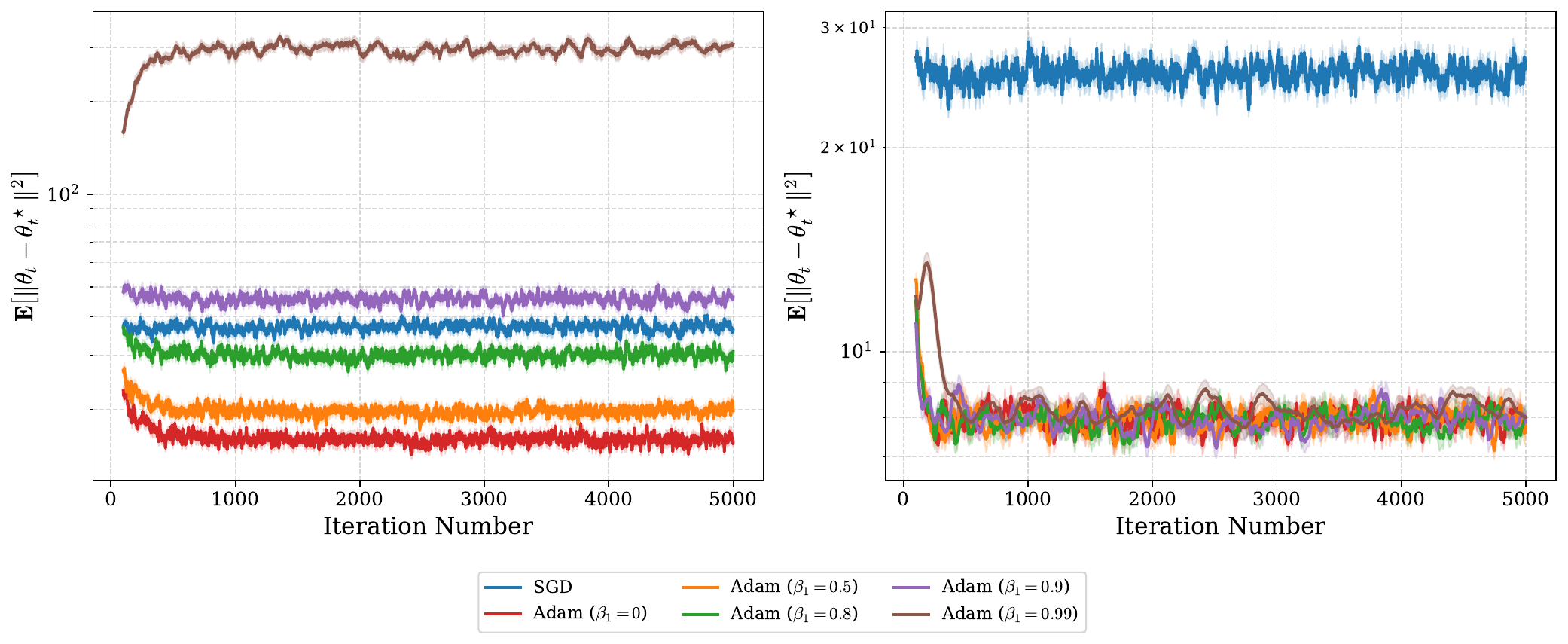}
    \caption{
Dependence on the Adam first-moment parameter $\beta_1$ for the online quadratic tracking problem.
The \textbf{left panel} corresponds to high drift and low noise ($\Delta=2$, $\sigma^2=0.01$), while the \textbf{right panel} corresponds to low drift and high noise ($\Delta=0.05$, $\sigma^2=10$).
Across both panels, we fix $\beta_2=0.99$, $\epsilon=10^{-8}$, $d=100$, and $\gamma=0.05$.
Consistent with our theory, increasing $\beta_1$ degrades tracking in the drift-dominated regime by slowing adaptation, whereas under high noise the dependence on $\beta_1$ is much weaker, reflecting the tradeoff between noise averaging and stale-gradient bias.
}
    \label{fig:diagram_beta1_comparison}
\end{figure}

\begin{figure}[H]
    \centering
    \includegraphics[width=\textwidth]{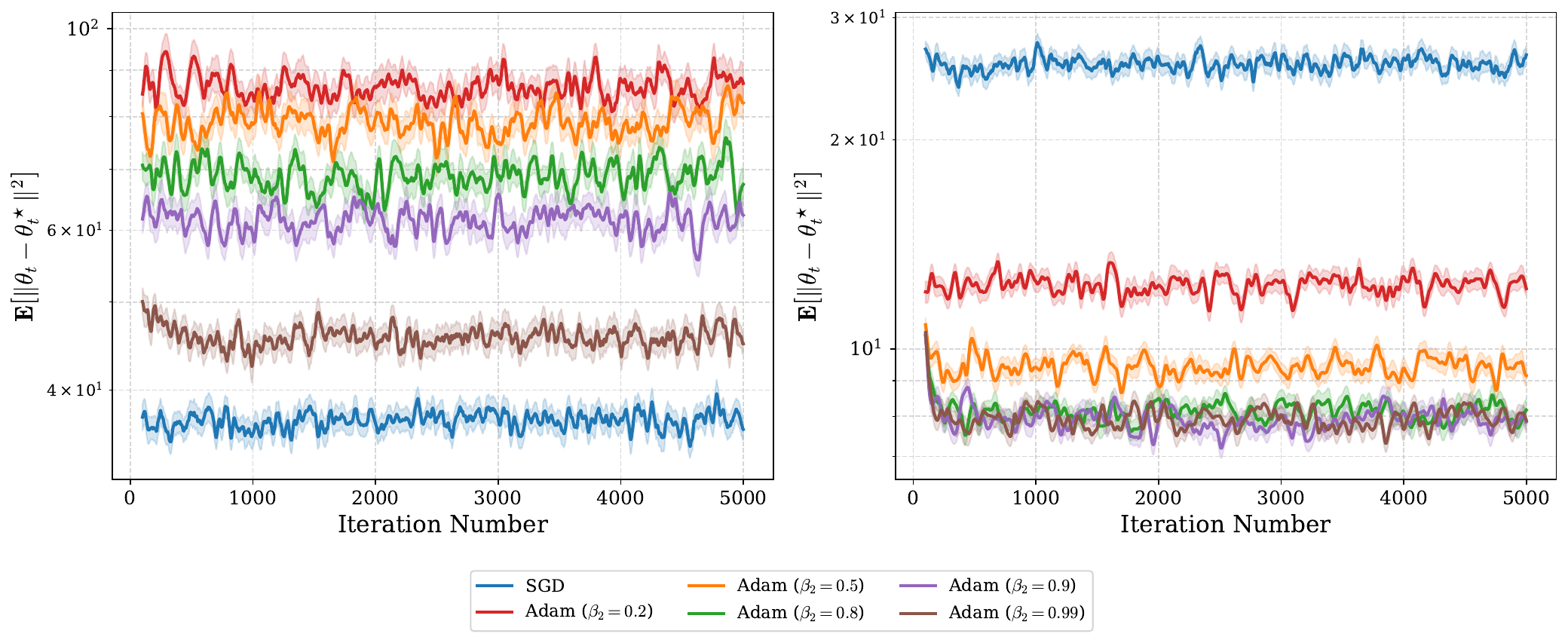}
  \caption{
Dependence on the Adam second-moment parameter $\beta_2$ for the online quadratic tracking problem.
The \textbf{left panel} corresponds to high drift and low noise ($\Delta=2$, $\sigma^2=0.01$), while the \textbf{right panel} corresponds to low drift and high noise ($\Delta=0.02$, $\sigma^2=10$).
Across both panels, we fix $\beta_1=0.9$, $\epsilon=10^{-8}$, $d=100$, and $\gamma=0.05$.
Consistent with our theory, increasing $\beta_2$ improves tracking in both regimes by reducing persistent perturbations from the stochastic second-moment estimate. Over this horizon, this stability benefit dominates the slower adaptation associated with longer second-moment memory.
}
    \label{fig:diagram_beta2_comparison}
\end{figure}

\begin{figure}[H]
    \centering
    \includegraphics[width=\textwidth]{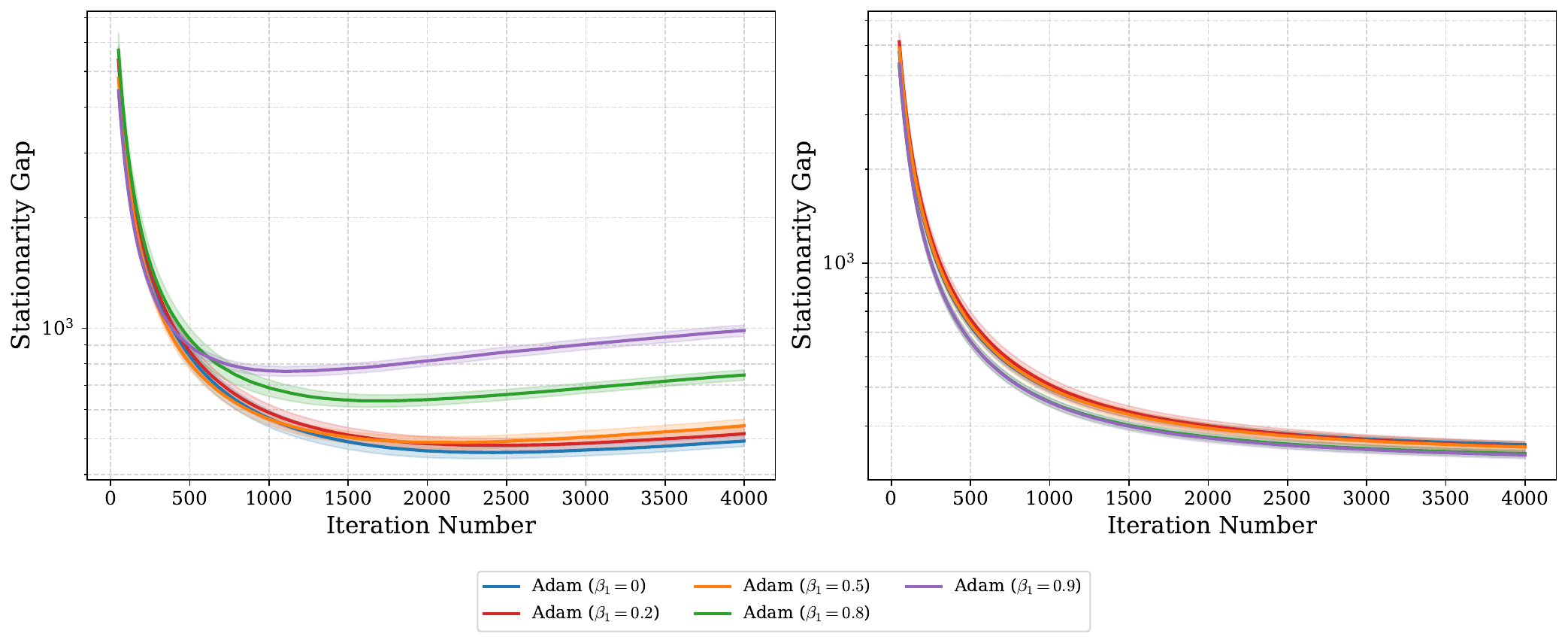}
 \caption{
Dependence on the Adam first-moment parameter $\beta_1$ for the phase-retrieval problem, measured by the stationarity gap.
The \textbf{left panel} corresponds to high drift and low noise ($\Delta=0.10$, $\sigma^2=0.01$), while the \textbf{right panel} corresponds to no drift and high noise ($\Delta=0$, $\sigma^2=1000$).
Across both panels, we fix $\beta_2=0.999$, $\epsilon=10^{-4}$, $d=100$, and $\gamma=0.01$.
Consistent with our theory, increasing $\beta_1$ worsens stationarity under high drift due to stale first-moment information, while under high noise larger $\beta_1$ improves stationarity through stronger gradient averaging.
}
    \label{fig:diagram_stationary_gap_comparison}
\end{figure}

\begin{figure}[H]
    \centering
    \includegraphics[width=\textwidth]{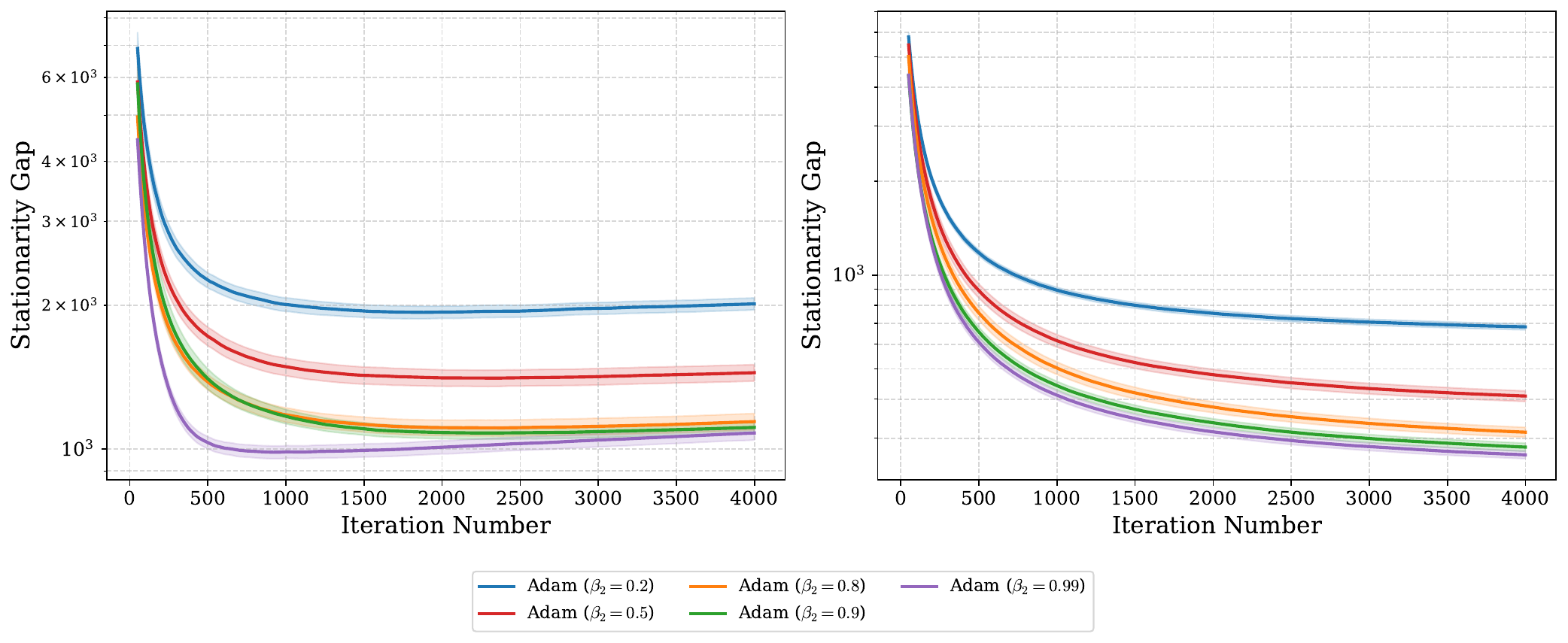}
 \caption{
Dependence on the Adam second-moment parameter $\beta_2$ for the phase-retrieval problem, measured by the stationarity gap.
The \textbf{left panel} corresponds to high drift and low noise ($\Delta=0.10$, $\sigma^2=0.01$), while the \textbf{right panel} corresponds to no drift and high noise ($\Delta=0$, $\sigma^2=1000$).
Across both panels, we fix $\beta_1=0.9$, $\epsilon=10^{-4}$, $d=100$, and $\gamma=0.01$.
Consistent with our theory, increasing $\beta_2$ improves long-run stationarity in both regimes by reducing persistent perturbations from the stochastic second-moment estimate. Over this horizon, this stability benefit dominates the slower adaptation induced by longer second-moment memory.
}
    \label{fig:diagram_stationary_gap_comparison_beta2}
\end{figure}

\end{document}